\documentclass[pdflatex,sn-mathphys-num]{sn-jnl}% Math and Physical Sciences Numbered Reference Style
\usepackage{graphicx}%
\usepackage{subcaption}
\usepackage{multirow}%
\usepackage{amsmath,amssymb,amsfonts}%
\usepackage{amsthm}%
\usepackage{mathrsfs}%
\usepackage[title]{appendix}%
\usepackage{xcolor}%
\usepackage{textcomp}%
\usepackage{manyfoot}%
\usepackage{booktabs}%
\usepackage{algorithmicx}%
\usepackage{algpseudocode}%
\usepackage{listings}%
\usepackage[ruled,vlined]{algorithm2e}

\theoremstyle{thmstyleone}%
\newtheorem{theorem}{Theorem}[section]%  meant for continuous numbers
\newtheorem{proposition}[theorem]{Proposition}% 

\theoremstyle{thmstyletwo}%
\newtheorem{remark}{Remark}[section]%

\theoremstyle{thmstylethree}%
\newtheorem{assumption}{Assumption}[section]
\newtheorem{corollary}{Corollary}[section]

\begin{document}

\title[Cell-induced densification and tether formation in fibrous extracellular matrices]{Cell-induced densification and tether formation in fibrous extracellular matrices with biomimetic physics-informed neural networks}

% Authors (extracted from your .tex)
\author[1]{\fnm{Anci} \sur{Lin}}
% \email{<not provided in the tex>}

\author*[2,3]{\fnm{Zhiwen} \sur{Zhang}}\email{zhangzw@hku.hk}

\author*[1]{\fnm{Wenju} \sur{Zhao}}\email{zhaowj@sdu.edu.cn}

% Affiliations (parsed from your \address[...] blocks)
\affil[1]{\orgdiv{School of Mathematics},
          \orgname{Shandong University},
          \orgaddress{\city{Jinan},
                      \state{Shandong},
                      \postcode{250100},
                      \country{P.R. China}}}

\affil[2]{\orgdiv{Department of Mathematics},
          \orgname{The University of Hong Kong},
          \orgaddress{\street{Pokfulam Road},
                      \city{Hong Kong},
                      \state{Hong Kong SAR},
                      \country{P.R. China}}}

\affil[3]{\orgdiv{Materials Innovation Institute for Life Sciences and Energy (MILES)},
          \orgname{HKU-SIRI},
          \orgaddress{\city{Shenzhen},
                      \country{P.R. China}}}

%Existing methods often suffer from over-smoothing in near-field patterns.
\abstract{Nonconvex multi-well energies in cell-induced phase transitions give rise to fine-scale microstructures, low-regularity transition layers and sharp interfaces, all of which pose numerical challenges for physics-informed learning. Here we introduce biomimetic physics-informed neural networks (Bio-PINNs), which implement a near-to-far curriculum by progressively revealing the computational domain away from the cell boundary and combining this schedule with a deformation-uncertainty proxy that concentrates collocation points near evolving transition layers and tether-forming regions. Across single-cell and multicellular benchmarks, Bio-PINNs recover the densified phase more reliably near cell boundaries and in intercellular gaps, while capturing tether morphology more faithfully than representative ungated and residual-driven adaptive baselines.}




\maketitle

\section{Introduction}
Contractile cells mechanically remodel fibrous extracellular matrices (ECMs) by transmitting traction through focal adhesions, thereby driving compression-induced fibre microbuckling and densification phase transitions. Around single cells, this remodelling can produce pericellular densification and needle-like microstructures; in multicellular settings, it can generate intercellular tethers, that is, relatively straight densified bands that mechanically couple distant cells \cite{grekas2021cells}. A recent variational description captures these phenomena through a non-rank-one-convex strain-energy density regularized by a higher-gradient term that introduces an intrinsic length scale for phase-transition layers \cite{grekas2021cells,grekas2022approx}. The resulting solutions may contain sharp interfaces, competing morphologies, and strong sensitivity to cell spacing and arrangement \cite{ball1976convexity,kinderlehrer1991young,kohn1986optimal}.

To visually anchor the biological setting, Fig.~\ref{fig:intro-exp-num} juxtaposes experimentally observed mechanically interacting mammary acini with a representative Bio-PINN solution for a long-distance two-cell configuration. To parallel the visual logic of Fig.~2B in Shi et al.~\cite{shi2014rapid}, the numerical panel uses an in-panel Zoom~1 at the bottom to highlight the intercellular corridor and an external Zoom~2 on the right to highlight the pericellular band adjacent to the cell boundary. The experiment and simulation both exhibit these two motifs: a mechanically organized intercellular corridor between neighbouring structures and a localized pericellular band on the outer flank. This qualitative agreement helps motivate the mechanically interacting patterns that the numerical benchmarks studied below are designed to recover, and we return to this experimental panel in Sec.~\ref{sec:two-cell} as a qualitative biological anchor when interpreting the two-cell numerical morphologies.

\begin{figure}[t]
    \centering
    \includegraphics[width=0.9\textwidth]{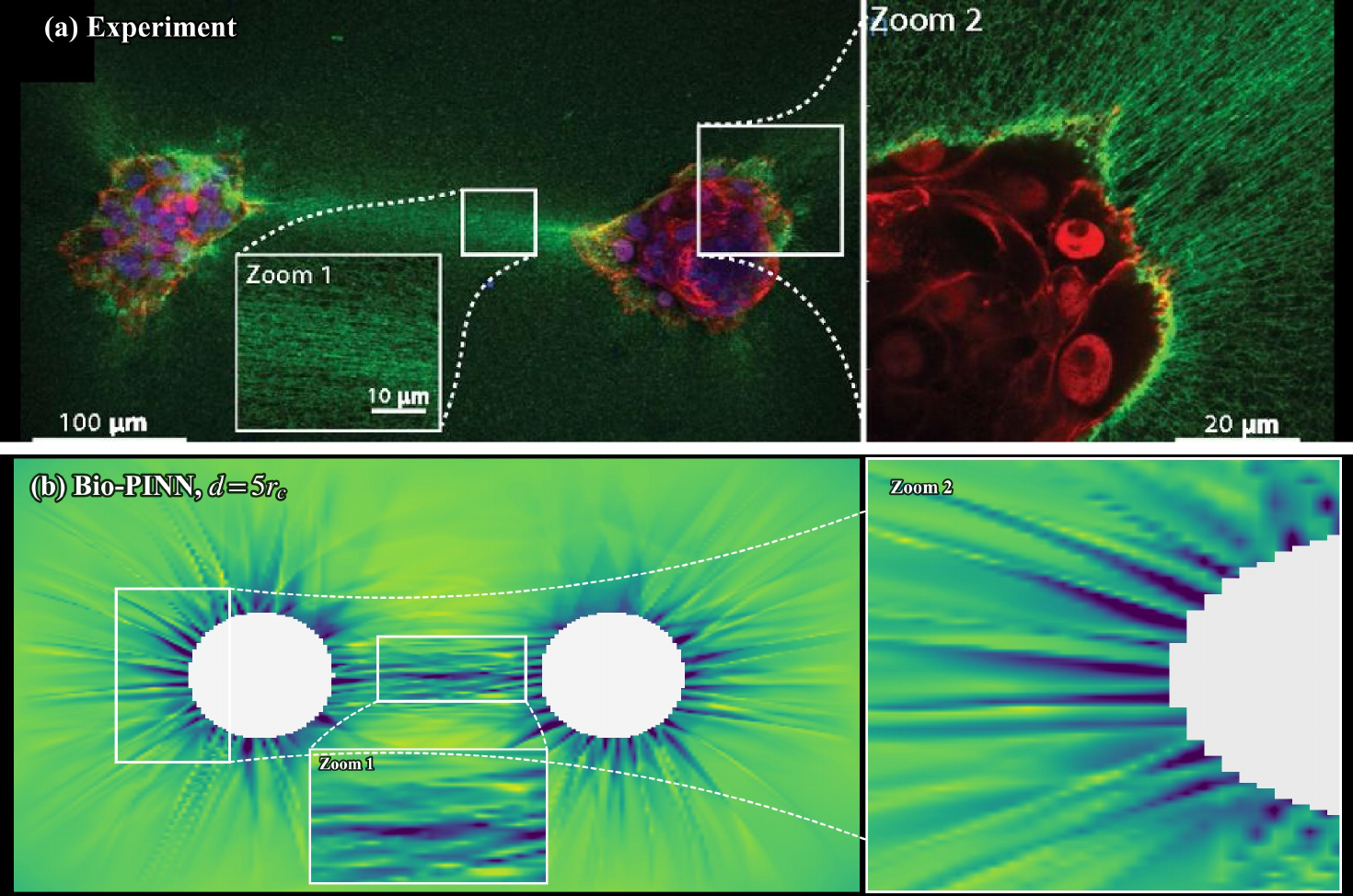}
    \caption{Experiment--simulation comparison for mechanically interacting acini. (a) Fluorescence image of two mammary acini embedded in collagen, adapted from Fig.~2B of Shi et al.~\cite{shi2014rapid} (included here with permission from the authors). (b) Bio-PINN solution in the long-distance weak-coupling regime, $d=5r_c$. The numerical panel mirrors the visual structure of the experimental panel: Zoom~1 is placed inside the main field at the bottom and highlights the intercellular corridor, whereas Zoom~2 is placed externally on the right and highlights the pericellular band on the outer cell flank. Both panels display the coexistence of an intercellular mechanically organized corridor and an outer pericellular band. In the numerical panel, colour denotes $J=\det F$, with smaller $J$ indicating stronger densification.}
    \label{fig:intro-exp-num}
\end{figure}

These features make cell-induced phase transitions challenging to resolve numerically. Nonconvex multi-well energies are commonly treated using higher-order or mixed discretizations, phase-field, gradient-enhanced or nonlocal models, continuation schemes, and relaxation-based approximations \cite{arnold2002dg,brenner2005c0ip,argyris1968tuba,ciarlet1978fem,hughes2005iga,brezzi1991mixed,glowinski1989augmented,allgower1990numerical,cahn1958free,allencahn1979,mindlin1964micro,eringen2002nonlocal,kinderlehrer1991young,kohn1986optimal}. For cell-induced remodelling, however, ultrathin bands and tethers demand extreme spatial resolution, while nonconvexity amplifies sensitivity to initialization, cell geometry, and regularization choices. Minimizing sequences may refine oscillations as they approach the infimum, undermining stability and complicating interpretation \cite{grekas2022approx}.

Physics-informed neural networks (PINNs) and related scientific machine-learning methods are attractive because they provide mesh-free approximations and can be trained directly from governing equations or variational principles \cite{raissi2020hfm,jin2021nsfnets,he2021ade,yang2021bpinn}. Yet their training is usually most reliable when the target solution is sufficiently regular \cite{raissi2019pinn,karniadakis2021piml}. In systems with sharp interfaces, discontinuities, or fine-scale structure, optimization often stagnates or converges to oversmoothed states even under adaptive sampling \cite{krishnapriyan2021failuremodes,wang2022pinnntk,daw2023r3}. For cell-induced phase transitions, this weakness is particularly pronounced. Residual-based indicators can saturate in relatively flat regions and miss mechanically decisive transition zones, global loss reweighting does not reflect the near-to-far progression of cell-mediated remodelling, and the premature introduction of high-frequency content can destabilize optimization near evolving interfaces \cite{daw2023r3,rahaman2019spectralbias,wang2022pinnntk,tancik2020fourier}.

Here we translate two physical cues from cell-mediated remodelling into the training procedure. First, remodelling begins in the pericellular region and propagates outward from the cell boundary, suggesting that learning should reveal the domain from near to far rather than all at once. Second, the relevant microstructures are organized around an intrinsic interfacial length scale, suggesting that collocation points should be concentrated near regions likely to develop transition layers and tethers. Guided by these cues, we develop biomimetic physics-informed neural networks (Bio-PINNs), an energy-based framework that couples a distance-gated near-to-far curriculum with an uncertainty-guided retain, resample, and release (R3) update using low-discrepancy proposals \cite{daw2023r3,e2018deepritz,kharazmi2021vpinn}. The gate prioritizes mechanically decisive near-field interactions before exposing the far field, while the adaptive sampler reallocates collocation effort towards evolving interfaces and tether-forming regions.

We also analyse the resulting adaptive dynamics and establish structural properties for Retain, Resample, and Release, including persistent coverage under gate expansion and quantitative near-to-far accumulation. Figure~\ref{fig:overview-bio-pinns} summarizes the workflow.

\begin{figure}[t]
    \centering
    \includegraphics[width=\textwidth]{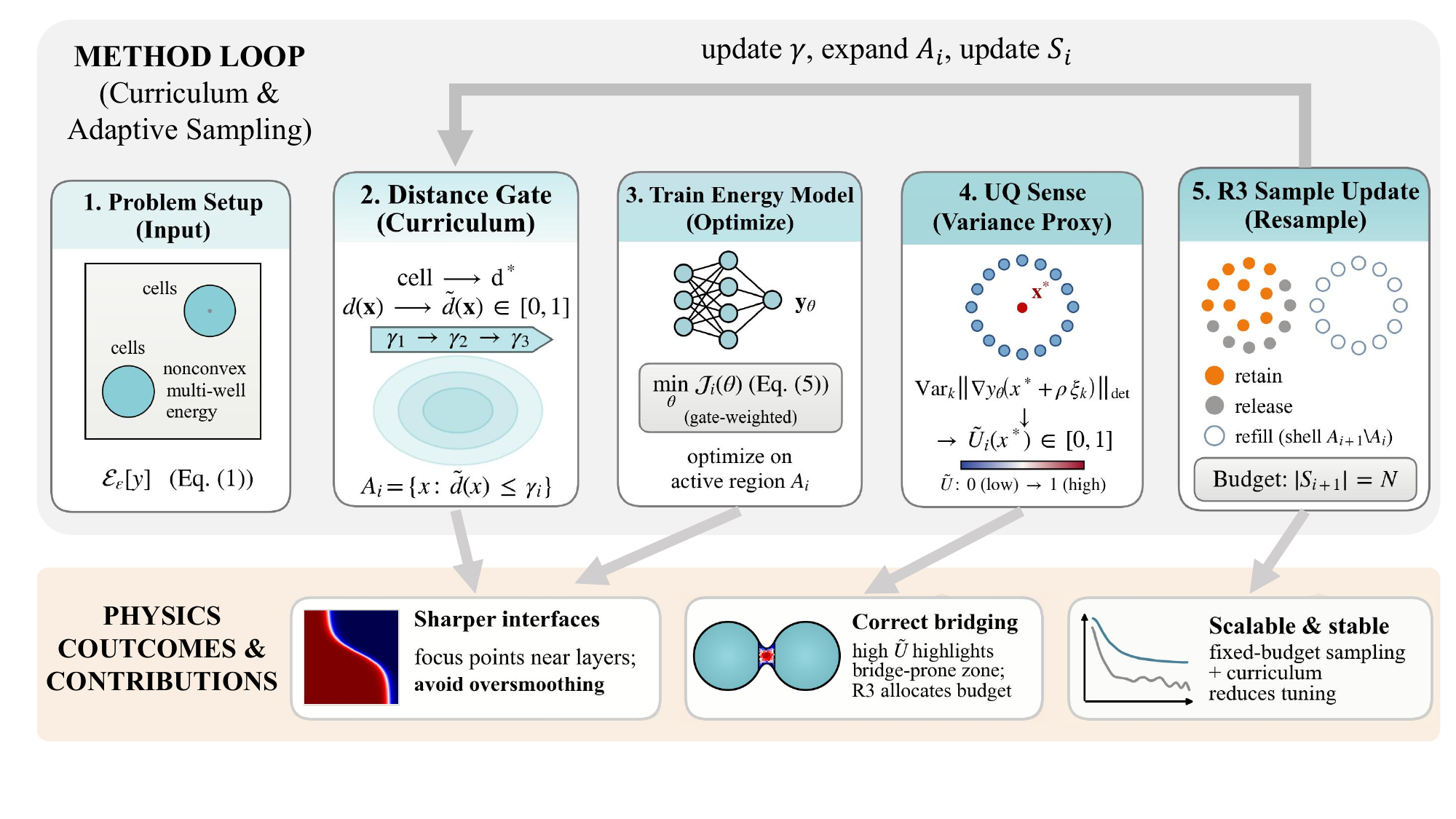}
    \caption{Cell-induced remodelling motivates the Bio-PINN workflow. A logistic distance gate $g_\gamma(x)$ progressively exposes the matrix from the near-cell region to the far field, while a deformation-uncertainty proxy drives an R3 retain, resample, and release update under a fixed collocation budget. Together, these steps focus training on pericellular transition layers and tether-forming regions as the active domain expands.}
    \label{fig:overview-bio-pinns}
\end{figure}

In single-cell benchmarks, Bio-PINNs yield lower near-cell $J$ values and more coherent pericellular densification than ungated or purely residual-driven baselines. In multicellular settings, Bio-PINNs better preserve separation-dependent tether morphologies across regularization regimes. Taken together, these results suggest that Bio-PINNs and related physics-aligned curricula can strengthen scientific machine learning for cell-induced phase transitions and, more broadly, for nonconvex variational problems with sharp interfacial structure.

\section{Results}\label{sec:results}

\subsection{Physical setting and learning framework}\label{sec:biopinn}
We begin with the physical patterns that the model must recover. Contractile cells in a fibrous extracellular matrix can drive strong pericellular densification and, when several cells interact, generate narrow intercellular tethers \cite{ball1976convexity,balljames1987mixtures,grekas2021cells,grekas2022approx}. In the continuum description studied here, these patterns arise as minimizers that form microstructure rather than as smooth single-phase states. The numerical challenge is therefore not simply to reduce a residual or an energy surrogate, but to recover sharp transition layers and geometry-dependent phase selection without introducing spurious deformation in the far field.

\subsubsection{Nonconvex variational target}
Let $\Omega\subset\mathbb R^2$ be a bounded Lipschitz domain, and let
$B_i=\{x\in\Omega:\ |x-c_i|<r_i\}$, $i=1,\dots,N_c$, denote the cell regions.
We work on the perforated domain
$\Omega_c:=\Omega\setminus\bigcup_{i=1}^{N_c} B_i$
and seek a deformation $y:\Omega_c\to\mathbb R^2$ by minimizing the energy
\begin{equation}\label{eq:total-energy}
	\mathcal{E}_\varepsilon[y]
	:=\int_{\Omega_c}\big(W(\nabla y)+\Phi(\det\nabla y)\big)\,dx
	+\frac{\varepsilon^2}{2}\int_{\Omega_c}\|\nabla^2 y\|^2\,dx .
\end{equation}
We denote the displacement by $u:=y-\mathrm{id}$ and the deformation gradient by
$F:=\nabla y=I+\nabla u$.

The strain energy $W$ is multi-well and non-rank-one-convex, encoding the coexistence of low-density and densified phases and promoting microstructure formation \cite{ball1976convexity,balljames1987mixtures}. The variational problem is therefore genuinely nonconvex. Minimizing sequences may develop fine-scale oscillations, the relaxed quasiconvex envelope can differ substantially from $W$, and standard lower-semicontinuity and stability properties may fail. As a result, minimizers can be non-unique and numerical approximations can become highly sensitive to initialization and regularization choices. The penalty $\Phi$ enforces orientation preservation and discourages interpenetration, for example through steep growth as $\det\nabla y\downarrow0$ \cite{ball1981invertibility}. The second-gradient term introduces an intrinsic interfacial length scale and leads to a fourth-order Euler--Lagrange system \cite{toupin1962couple,mindlin1964microstructure,grekas2022approx}. Cell contraction is imposed on each $\partial B_i$ through boundary data, and the resulting minimizers can exhibit sharp interfaces and competing tether morphologies.

\subsubsection{Two biomimetic priors}
Two physical cues guide the learning strategy. First, remodelling progresses from the pericellular region into the bulk ECM, so the computational domain should be revealed from near to far rather than all at once. Second, the relevant microstructures are controlled by an intrinsic interfacial length scale that is often represented through second derivatives of the displacement. We encode these cues by (i) a progressive distance gate that activates the domain from near to far and (ii) a deformation-uncertainty indicator that highlights regions likely to develop transition layers and tethers. The latter is used as a sampling signal for interfacial localization, so that collocation effort is concentrated where explicit second-gradient penalties would matter most.

\subsubsection{Energy-form PINN ansatz}
We represent the deformation by a neural approximation $y_\theta:\Omega_c\to\mathbb R^2$ based on a sufficiently regular feed-forward network $N_\theta$ and a standard lifting construction for outer Dirichlet data \cite{lagaris1998annpde}.
\begin{equation}\label{eq:ytheta}
	y_\theta(x)=\ell(x)+\varphi(x)N_\theta(x),
\end{equation}
where $\ell=g_D$ and $\varphi=0$ on $\Gamma_{\mathrm{out}}=\partial\Omega$. Inner-boundary conditions on
$\Gamma_{\mathrm{in}}:=\partial\Omega_c\setminus\Gamma_{\mathrm{out}}$  are
enforced weakly through a quadratic penalty $\mathcal P_{\mathrm{in}}[y_\theta]$. Training minimizes a
variational objective aligned with microstructure selection in nonconvex phase-transition models
\cite{e2018deepritz}, optionally augmented with weak residual terms \cite{kharazmi2021vpinn}.

\subsubsection{Near-to-far distance gating}
To impose a near-to-far curriculum, we introduce a normalized distance $\tilde d(x)\in[0,1]$ to the union of cell regions $\mathcal C=\cup_i B_i$ and define a smooth gate
\begin{equation}\label{eq:soft-gate-plain}
	g_{\gamma_i}(x)=\sigma\!\big(\alpha(\gamma_i-\tilde d(x))\big)\in(0,1)
\end{equation}
with steepness $\alpha>0$ and gate level $\gamma_i\in(0,1]$. Points with small $\tilde d(x)$ are activated early ($g_{\gamma_i}\approx1$), while the far field is progressively revealed as $\gamma_i$ increases. Figure~\ref{fig:gate_time_to_space} illustrates how this outward progression in remodelling can be translated into a distance-based gate, together with spatial snapshots of $g_{\gamma}(x)$ for different gate levels $\gamma$. The gate weights the training objective so that optimization prioritizes the currently active region, and the gate level is advanced automatically using a progress signal computed from the gate-weighted objective. In this way, an observed physical progression is converted into a numerical schedule that mitigates propagation failure and reduces premature smoothing across interfaces.

\begin{figure}[t]
  \centering
  \includegraphics[width=\linewidth]{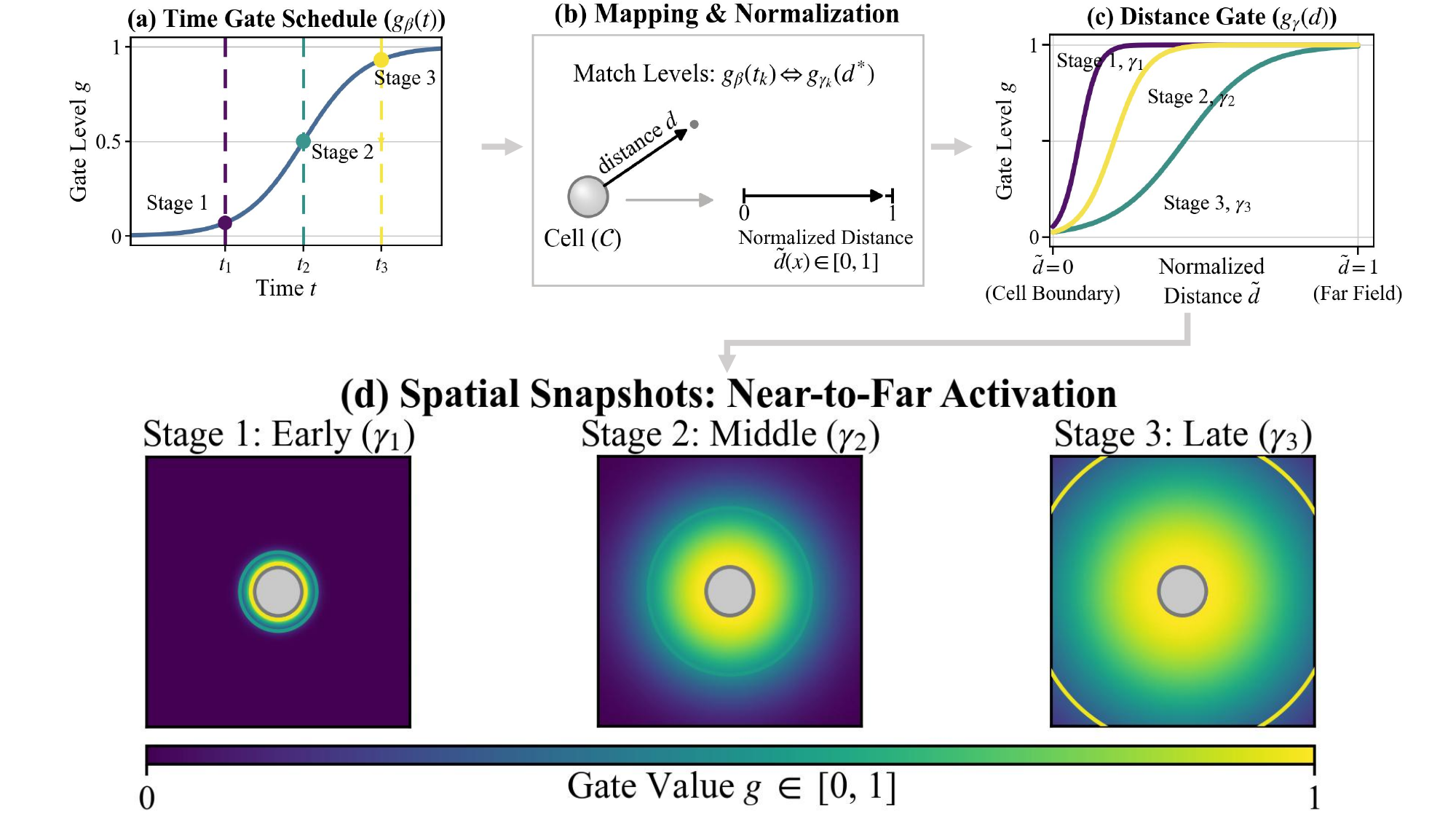}
\caption{Near-to-far activation derived from the physical progression of remodelling. The left panel shows a logistic time gate $g_{\beta}(t)=\sigma\!\left(\alpha(\beta-t)\right)$. The right panel shows its distance-based analogue $g_{\gamma}(\tilde d)=\sigma\!\left(\alpha(\gamma-\tilde d)\right)$, where $\tilde d$ is the normalized distance to the cell. The bottom panels show spatial snapshots of $g_{\gamma}(x)=\sigma\!\left(\alpha(\gamma-\tilde d(x))\right)$ for increasing $\gamma$, showing how the active region expands from the pericellular neighbourhood into the surrounding matrix.}
  \label{fig:gate_time_to_space}
\end{figure}

\subsubsection{Uncertainty as an interfacial-scale proxy}
Interfaces, sharp layers, and emerging tethers are precisely the regions where the learned deformation is most sensitive to local perturbations during training. Small changes in $y_\theta$ can switch phases or alter tether topology. We quantify this sensitivity through an uncertainty-quantification (UQ) proxy based on local deformation variability from an independent collocation.
\begin{equation}\label{eq:Ugrad}
	U_i(x)=\operatorname{Var}_k\!\Big(\|\nabla y_{\theta_i^{(k)}}(x)\|_F\Big),\qquad
	\widetilde U_i(x)=\mathcal N\!\big(U_i^{(\nabla y)}(x)\big)\in[0,1],
\end{equation}
where $\mathcal N(\cdot)$ denotes a stability normalization \cite{efron1979bootstrap,gal2016dropout,lakshminarayanan2017deepensembles}. Other features, such as $|\det\nabla y|$, can be used in the same manner. In this work, $\widetilde U_i$ serves as a local sampling indicator rather than as a calibrated uncertainty estimate. It highlights regions likely to host transition layers or tethers even when residual signals plateau, and it biases sampling towards locations where explicit second-gradient penalties would be most informative without repeatedly evaluating $\nabla^2 y_\theta$ across the full collocation set. Figure~\ref{fig:uq_proxy_local} schematizes the local probing procedure used to evaluate the UQ proxy.

\begin{figure}[t]
  \centering
  \includegraphics[width=0.88\linewidth]{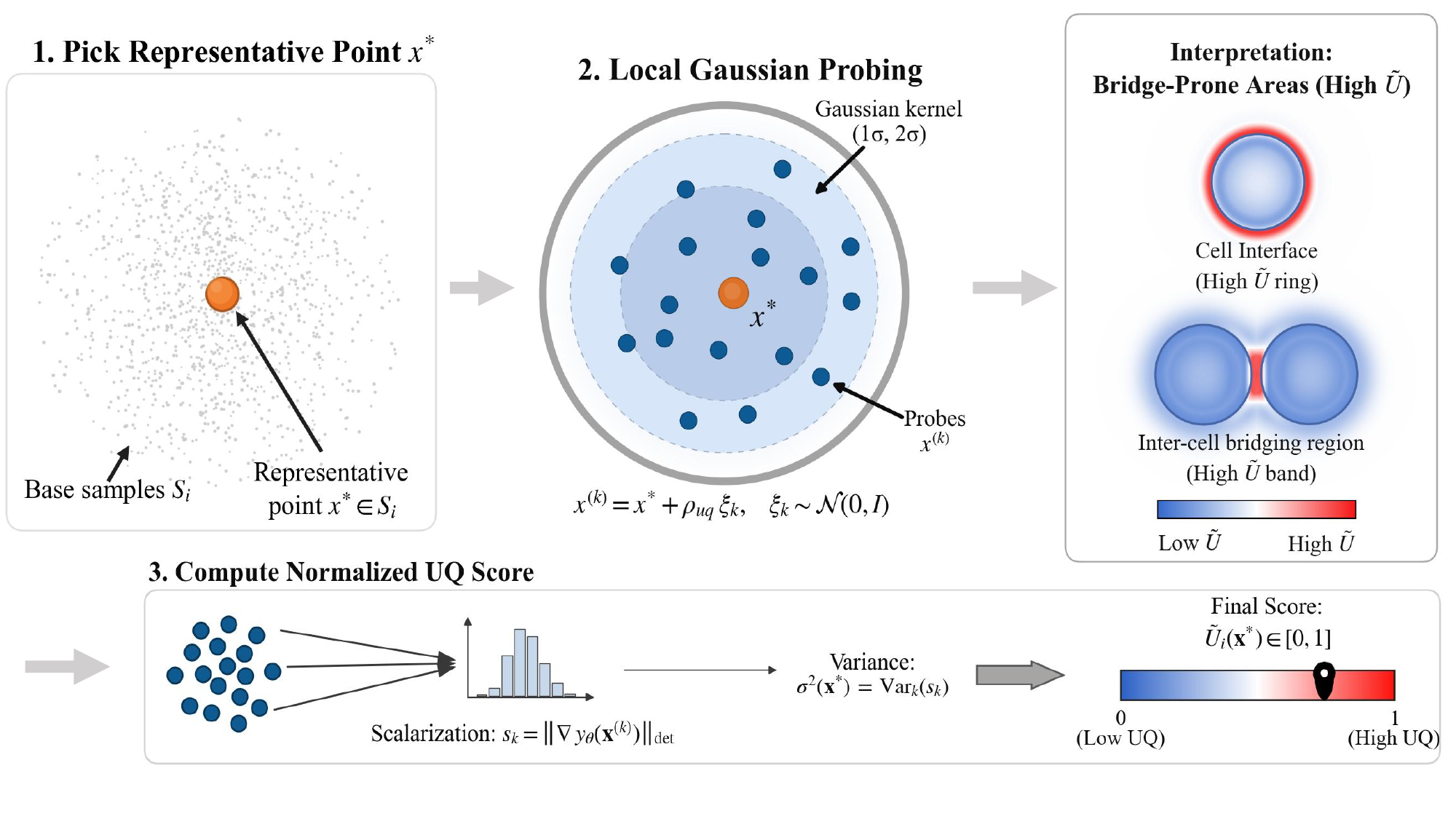}
\caption{Local probing used to detect regions likely to develop transition layers or tethers. Around each representative collocation point $x^\star$, Gaussian probes sample nearby deformations and estimate local variability, for example $\operatorname{Var}_k\|\nabla y_{\theta_i^{(k)}}(x^\star)\|_F$. The resulting normalized score $\widetilde U_i(x^\star)\in[0,1]$ is largest near evolving interfaces and intercellular interaction regions.}
  \label{fig:uq_proxy_local}
\end{figure}

\subsubsection{UQ-R3 adaptive collocation with low-discrepancy proposals}
Bio-PINNs maintain a fixed budget of bulk collocation points $\mathcal S_i$ inside the active region and update them through an uncertainty-driven retain, resample, and release (R3) step \cite{daw2023r3}. At each stage, points with high $\widetilde U_i$ are retained, low-uncertainty points are released, and new points are resampled to refill the budget. When the gate expands, Bio-PINNs inject samples into the newly revealed shell to ensure exploration of regions that have not yet been trained on. New points are generated using low-discrepancy designs, for example Hammersley points, to preserve coverage while focusing on uncertain structures. Figure~\ref{fig:uq_r3_loop} summarizes the resulting R3 update loop and shows how retained points and new proposals interact across stages.

\begin{figure}[t]
  \centering
  \includegraphics[width=\linewidth]{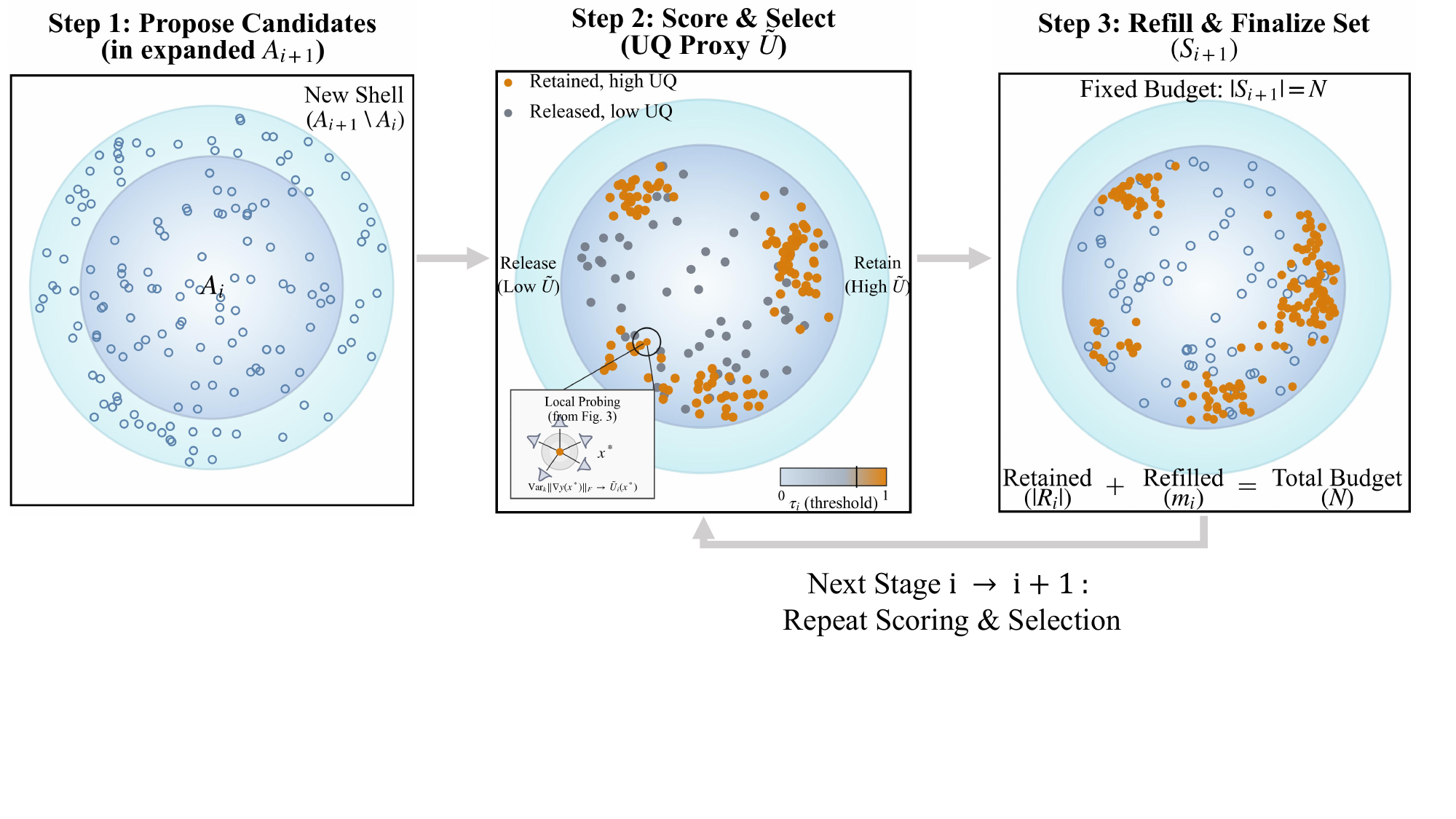}
  \vspace{-5em}
\caption{Adaptive collocation loop used in Bio-PINNs. In Step 1, candidate collocation points are proposed in the currently active region and, after gate expansion, in the newly opened shell using low-discrepancy designs. In Step 2, the UQ proxy $\widetilde U_i$ is evaluated and high-uncertainty points are retained while low-uncertainty points are released. In Step 3, the collocation set is refilled to preserve a fixed budget and the procedure is repeated, so that sampling progressively concentrates on transition layers and emerging tethers while maintaining coverage of the active domain.}
  \label{fig:uq_r3_loop}
\end{figure}

\subsubsection{Gate-weighted variational objective and staged training}
At stage $i$, Bio-PINN optimizes a gate-weighted empirical energy together with boundary penalties.
\begin{equation}\label{eq:gated-objective}
	\mathcal{J}_i(\theta)=\mathcal{E}^{(g)}_{\varepsilon,i}[y_\theta]+\mathcal{P}_{\mathrm{in}}[y_\theta],
\end{equation}
where $\mathcal{E}^{(g)}_{\varepsilon,i}$ is a gated estimator of \eqref{eq:total-energy} evaluated on $\mathcal S_i$. Training alternates between (i) several optimizer steps on $\mathcal{J}_i$ at fixed $\gamma_i$ and $\mathcal S_i$, (ii) advancing the gate level $\gamma_i$ based on objective progress, and (iii) updating $\mathcal S_i$ via UQ-R3. This staged loop couples where the model learns, through distance gating, with what it learns next, through uncertainty-driven sample reallocation, thereby providing a practical way to concentrate collocation effort on thin transition layers and tether microstructures as they emerge.

\subsection{Benchmark overview}\label{sec:experiments}
We next evaluate Bio-PINNs on benchmarks designed to probe two main signatures of cell-induced mechanical remodelling in this model, namely pericellular densification and separation-dependent tether formation. Unless stated otherwise, we visualize the learned deformation through the Jacobian $J=\det\nabla y_\theta$. Because the ratio of deformed to undeformed density equals $1/J$, smaller values of $J$ correspond to stronger densification, so $J$ distinguishes the undensified low-density matrix ($J\approx 1$) from the densified phase ($J\approx J_\star\approx 0.21$) \cite{grekas2022approx}. Throughout the Results, the central questions are whether training reaches the densified well in mechanically decisive regions and whether it does so without introducing spurious angular structure or spurious far-field deformation. For the two-cell benchmarks, the experimental panel of Fig.~\ref{fig:intro-exp-num} provides a recurring qualitative visual anchor: the features to track are the outer pericellular band and the mechanically organized intercellular corridor, which in the stronger-coupling regime sharpens into a densified tether.

The continuum model and the Bio-PINN training objective are described in Secs.~\ref{sec:methods:model} to \ref{sec:methods:objective}. Implementation details are given in Secs.~\ref{sec:methods:ritz} to \ref{sec:methods:objective}, with default hyperparameters listed in Supplementary Table~\ref{tab:ed-defaults}. We further verified that the main qualitative conclusions reported below are stable under representative sweeps of the resampling cadence, network capacity, and UQ-proxy settings (Supplementary Figs.~\ref{fig:ed-single-eps0-period} to \ref{fig:ed-three-eps0-uqrho}). Among these controls, the UQ probe variance $\rho_{uq}$ is especially interpretable. Smaller $\rho_{uq}$ yields more localized probing but can introduce noisier, more transient updates, whereas larger $\rho_{uq}$ increases spatial averaging and produces smoother but more diffuse updates while preserving the same interaction morphology.

% ============================================================
\subsection{Single-cell experiments}\label{sec:single-cell}

We begin with the simplest physical setting, namely a single contracting circular cell
$C=\{x:\lvert x-z\rvert=r_c\}\subset\Omega_c$ embedded in a fibrous matrix. In this case, the expected response is a densified pericellular band adjacent to the cell boundary, while the surrounding matrix remains close to the undensified state \cite{grekas2022approx}. The benchmark therefore isolates the basic difficulty of the problem. The method must resolve a thin pericellular transition layer without generating sampling-induced anisotropy in the surrounding ECM. Representative sensitivity studies for this setting are reported in Supplementary Figs.~\ref{fig:ed-single-eps0-period} to \ref{fig:ed-single-eps0-uqrho}, including sweeps over the R3 cadence $P$, network capacity, and UQ-proxy hyperparameters.

\subsubsection{Non-regularized limit with $\varepsilon_0=0$}\label{subsubsec:single-eps0}
We start from the non-regularized limit with $\varepsilon_0=0$, where sharp interfaces are not smoothed by an explicit second-gradient penalty. This is the most demanding single-cell regime because the solution must approach the densified well next to the cell while keeping the surrounding matrix near the undensified state. Figure~\ref{fig:detF} compares the Jacobian determinant $J=\det F$ across vanilla PINN~\cite{raissi2019pinn}, RAR-D~\cite{wu2023comprehensive}, residual-driven R3~\cite{daw2023r3}, and Bio-PINN. Here, the presence of a pericellular low-$J$ band alone is not sufficient evidence of success. The decisive diagnostics are the minimum attained value of $J$ and the azimuthal coherence of that band. The qualitative conclusions reported below remain stable under representative sweeps of the R3 cadence and model capacity (Supplementary Figs.~\ref{fig:ed-single-eps0-period} to \ref{fig:ed-single-eps0-width}).

\begin{figure}[t]
	\centering
	\begin{minipage}[t]{0.24\linewidth}
		\centering
		\includegraphics[width=\linewidth]{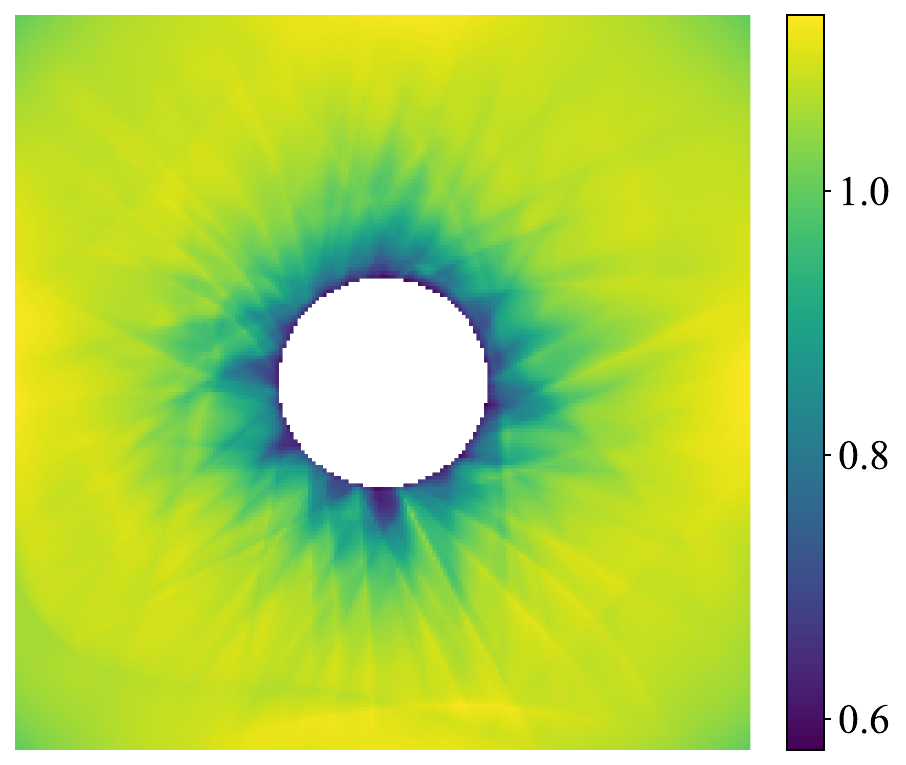}\\[-0.6em]
		{\scriptsize \textbf{(a)} PINN}
	\end{minipage}\hfill
	\begin{minipage}[t]{0.24\linewidth}
		\centering
		\includegraphics[width=\linewidth]{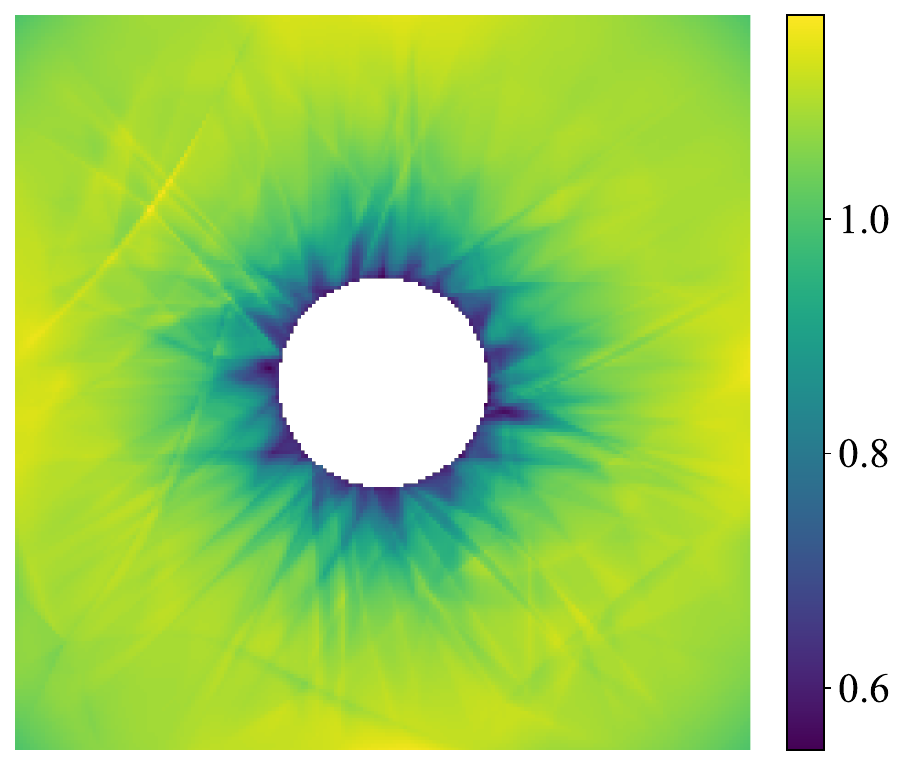}\\[-0.6em]
		{\scriptsize \textbf{(b)} RAR-D}
	\end{minipage}\hfill
	\begin{minipage}[t]{0.24\linewidth}
		\centering
		\includegraphics[width=\linewidth]{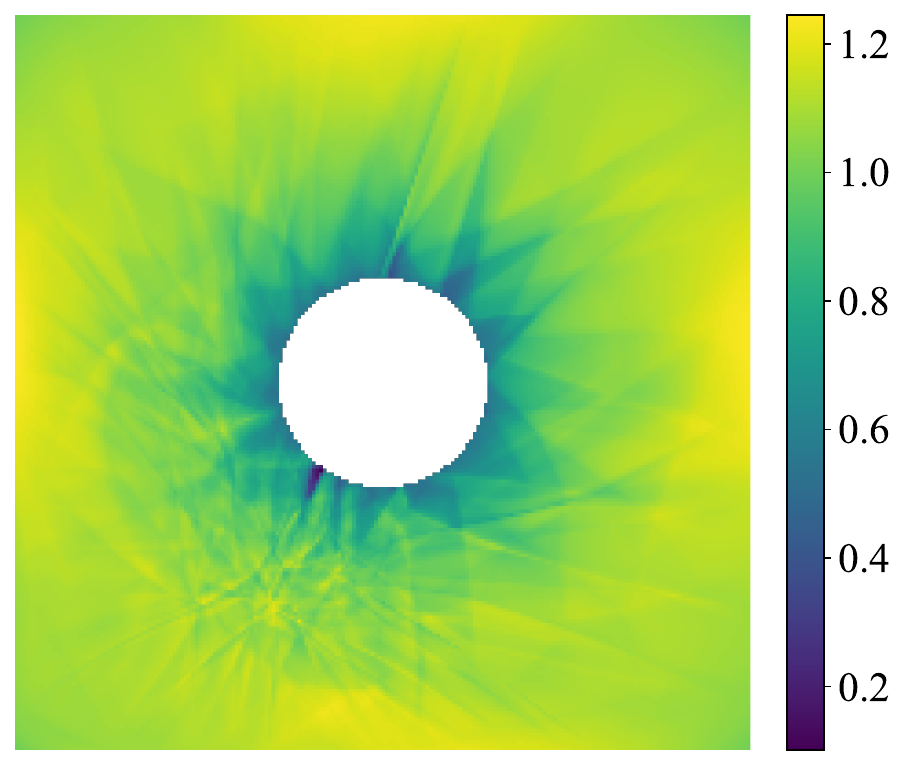}\\[-0.6em]
		{\scriptsize \textbf{(c)} R3}
	\end{minipage}\hfill
	\begin{minipage}[t]{0.24\linewidth}
		\centering
		\includegraphics[width=\linewidth]{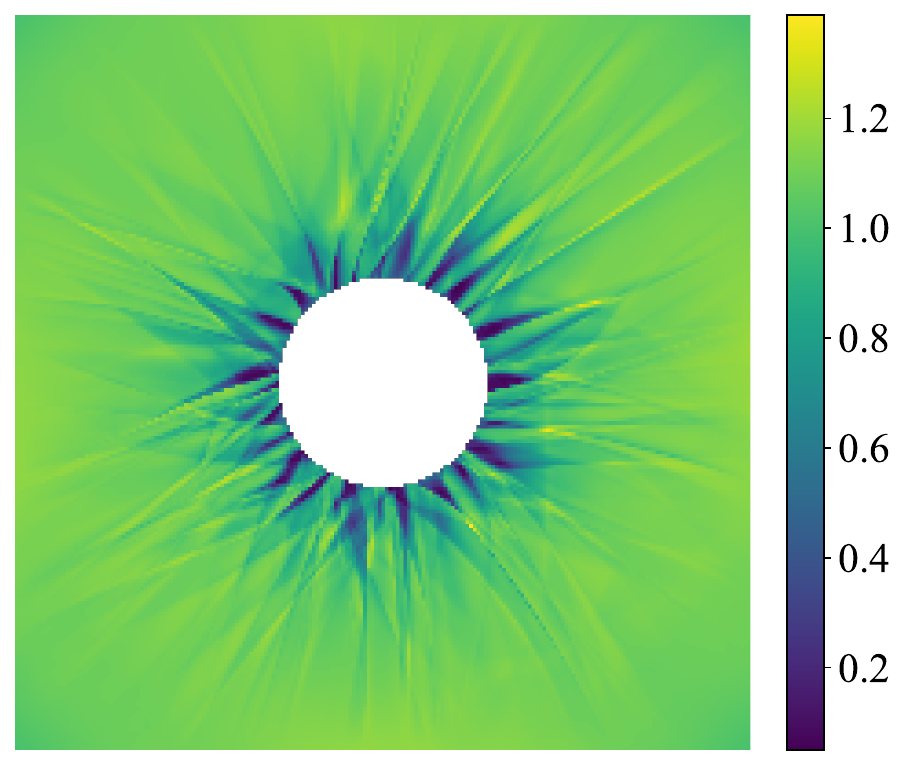}\\[-0.6em]
		{\scriptsize \textbf{(d)} Bio-PINN}
	\end{minipage}
\caption{Single-cell remodelling without second-gradient regularization ($\varepsilon_0 = 0$). The Jacobian determinant $J=\det F$ (with $F=\nabla y_\theta$) is shown for (a) PINN, (b) RAR-D, (c) R3, and (d) Bio-PINN. The benchmark asks whether the method can recover a densified pericellular band while leaving the surrounding matrix close to the undensified state. Among the methods shown here, Bio-PINN yields the lowest near-cell $J$ values together with a more azimuthally uniform pericellular band.}
	\label{fig:detF}
\end{figure}

In Fig.~\ref{fig:detF}, vanilla PINN in panel (a) captures the presence and approximate location of a pericellular low-$J$ band, but the transition layer remains diffuse and exhibits spurious radial streaks that extend into the surrounding matrix. These features do not form a coherent pericellular densification pattern. The minimum value indicated by the colour scale remains substantially above the densified-well level, staying around $J\gtrsim 0.6$. Optimization therefore does not reach the neighbourhood of the densified well near the cell boundary.

RAR-D in panel (b) sharpens the near-boundary structure by enriching regions of large residual. However, a residual-driven strategy alone does not enforce a near-to-far training progression. As a result, the pericellular band remains azimuthally non-uniform and weak spurious structure persists in the surrounding matrix. The minimum value of $J$ again stays elevated at about $J\gtrsim 0.6$, indicating that the solution is still not driven into the densified-well neighbourhood.

Residual-driven R3 in panel (c) concentrates samples more aggressively in difficult regions. In this energy-only setting, it can lower $J$ and sharpen parts of the transition layer. At the same time, repeated resampling onto thin high-error sets can amplify localized modes. Low-$J$ values then appear as discontinuous, spoke-like radial streaks radiating from the pericellular band rather than as a coherent and azimuthally uniform densification front. This behaviour reflects an unstable trade-off between fitting a sharp transition layer and avoiding sampling-induced anisotropy.

By contrast, Bio-PINN in panel (d) combines near-to-far gating with UQ-R3. The gate delays far-field participation and encourages a staged outward propagation from the cell, while the uncertainty proxy directs resampling towards the evolving transition layer rather than repeatedly over-refining narrow and potentially noisy sets. The resulting solution forms a contiguous densified pericellular band that reaches the densified well, with $J$ decreasing towards $J_\star$ as indicated by the colour scale. This band is more azimuthally uniform and exhibits markedly fewer spoke-like radial streaks than panels (a) to (c). Mild oscillations can still appear in localized high-curvature regions, which is expected in the absence of explicit second-gradient regularization.

\subsubsection{Sensitivity to the R3 cadence and the UQ proxy}
In the single-cell, non-regularized regime, the recovered solution is sensitive to the R3 resampling cadence. With very frequent resampling ($P=100$), the training dynamics do not localize the phase transition stably, and strong sampling-induced anisotropy contaminates the far field (Supplementary Fig.~\ref{fig:ed-single-eps0-period}). A moderate cadence ($P=400$) partially recovers the pericellular transition band but still under-resolves the expected microstructure. By contrast, longer periods with $P$ between 1600 and 3200 yield the anticipated fine-scale alternation between densified and non-densified states concentrated in the interfacial layer, while the far field remains approximately near-incompressible (Supplementary Fig.~\ref{fig:ed-single-eps0-period}). Once the network has sufficient capacity to represent a thin transition layer, further increases in depth and width can instead introduce an optimization bottleneck. The strongly nonconvex energy landscape, together with the enlarged parameter space, makes training more susceptible to poor local minima and slower convergence (Supplementary Figs.~\ref{fig:ed-single-eps0-depth} to \ref{fig:ed-single-eps0-width}).

For the UQ proxy, varying the Monte Carlo probe count $m_{uq}$ yields nearly indistinguishable fields, suggesting that the proxy remains stable even with modest probing effort (Supplementary Fig.~\ref{fig:ed-single-eps0-uqm}). The probe variance $\rho_{uq}$ acts as an effective length-scale control in local probing. Small-to-moderate $\rho_{uq}$ preserves fine angular structure and sharper focusing near the evolving transition layer, whereas very large $\rho_{uq}$ increases spatial averaging and progressively suppresses the finest patterns, producing smoother but more diffuse updates (Supplementary Fig.~\ref{fig:ed-single-eps0-uqrho}).

\subsubsection{Weak second-gradient regularization ($\varepsilon_0=0.01\,r_c$)}
We next add weak second-gradient regularization ($\varepsilon_0=0.01\,r_c$) to test whether the near-to-far curriculum remains useful once the interface is given a small intrinsic width. This regime is less singular than $\varepsilon_0=0$, but it still requires the method to localize the pericellular transition layer while avoiding diffuse spurious deformation in the surrounding matrix. Figure~\ref{fig:detF-eps001rc} shows that the $H^2$ term damps the finest oscillations across methods, while Bio-PINN still drives the near-cell region closer to the densified well and yields the most azimuthally coherent pericellular band among the methods compared here. The same qualitative trends persist under mild regularization. For both $\varepsilon_0=0$ and $\varepsilon_0=0.01r_c$, overly frequent R3 resampling degrades localization of the transition layer and amplifies sampling-induced artefacts, whereas longer periods yield a more stable interfacial layer concentrated near the cell boundary. Likewise, once the network can represent a thin transition layer, additional capacity provides limited qualitative benefit, and overly deep models can encounter an optimization bottleneck because of the strongly nonconvex energy landscape and the enlarged parameter space (Supplementary Figs.~\ref{fig:ed-single-eps0-period} to \ref{fig:ed-single-eps001rc-width}).

\begin{figure}[t]
	\centering
	\begin{minipage}[t]{0.24\linewidth}
		\centering
		\includegraphics[width=\linewidth]{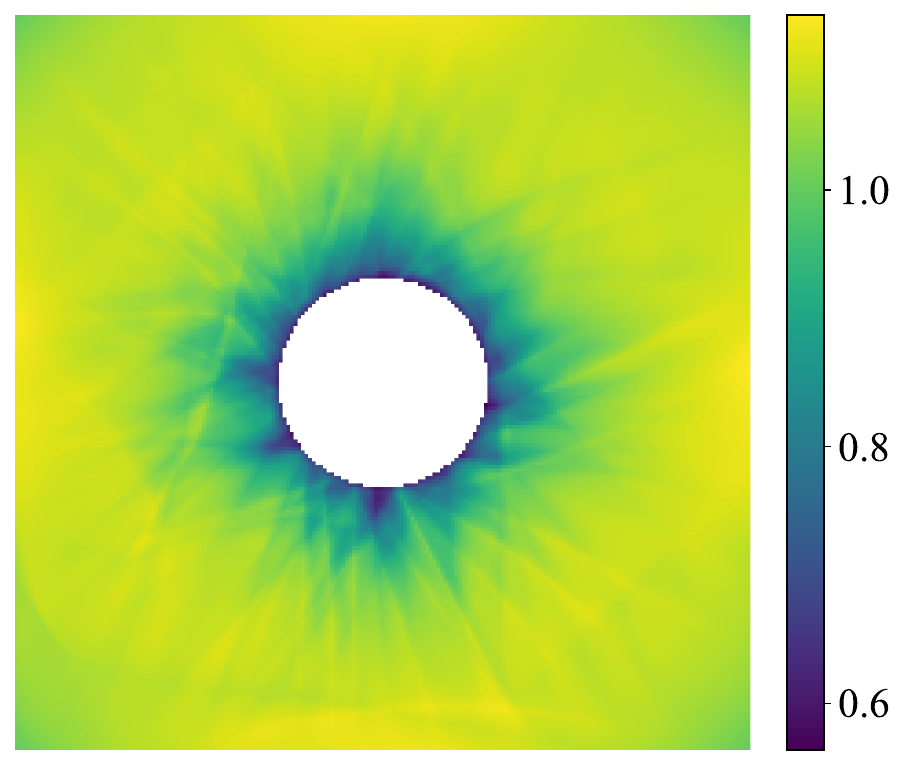}\\[-0.6em]
		{\scriptsize \textbf{(a)} PINN}
	\end{minipage}\hfill
	\begin{minipage}[t]{0.24\linewidth}
		\centering
		\includegraphics[width=\linewidth]{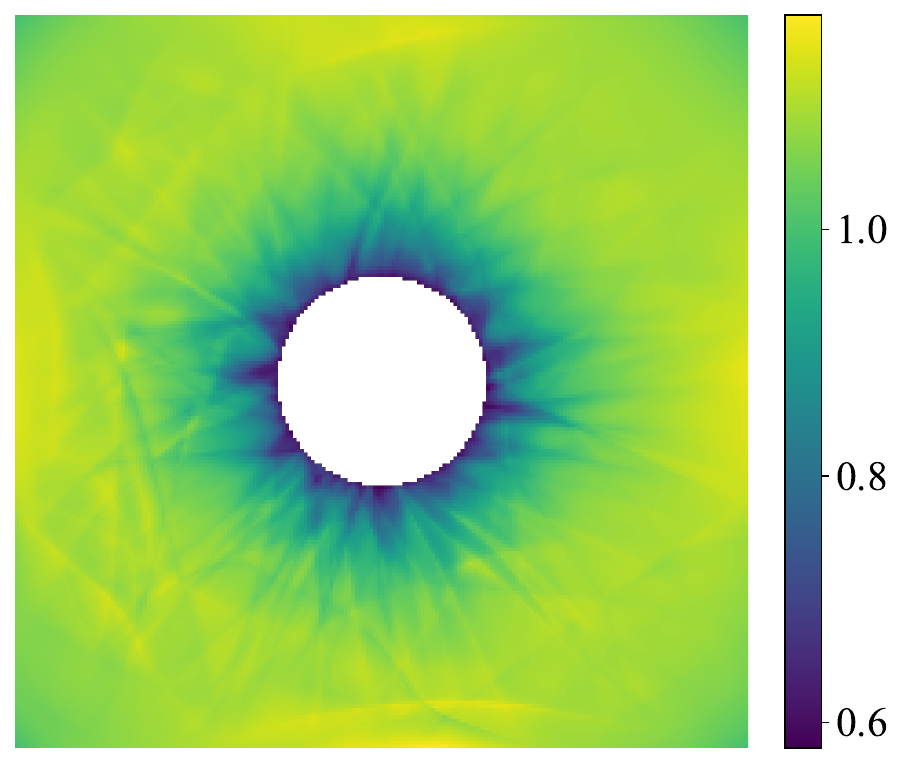}\\[-0.6em]
		{\scriptsize \textbf{(b)} RAR-D}
	\end{minipage}\hfill
	\begin{minipage}[t]{0.24\linewidth}
		\centering
		\includegraphics[width=\linewidth]{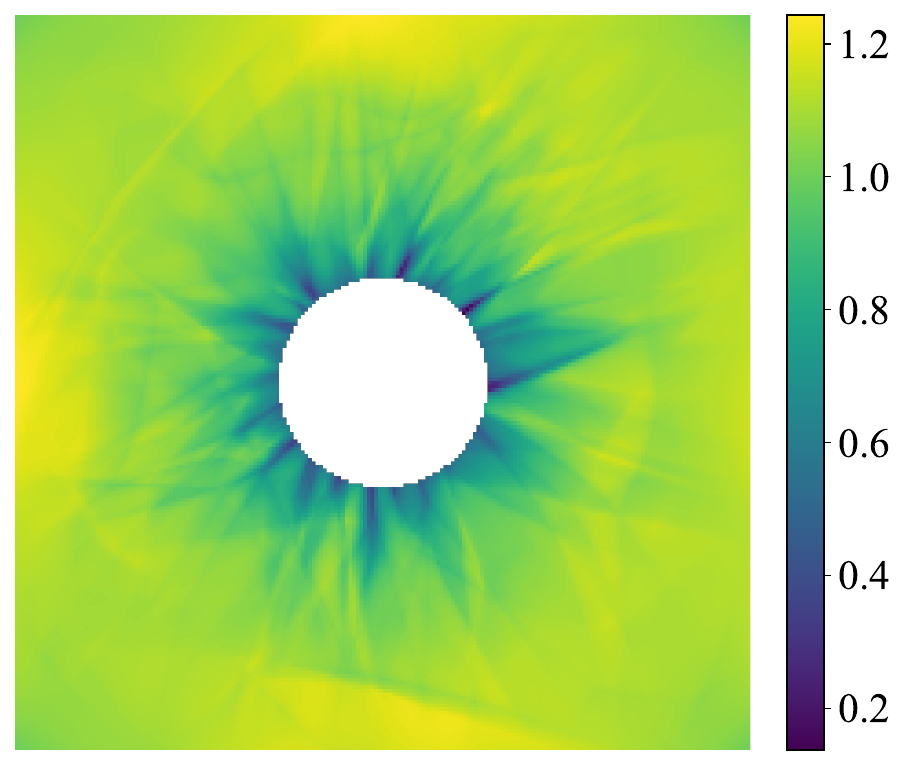}\\[-0.6em]
		{\scriptsize \textbf{(c)} R3}
	\end{minipage}\hfill
	\begin{minipage}[t]{0.24\linewidth}
		\centering
		\includegraphics[width=\linewidth]{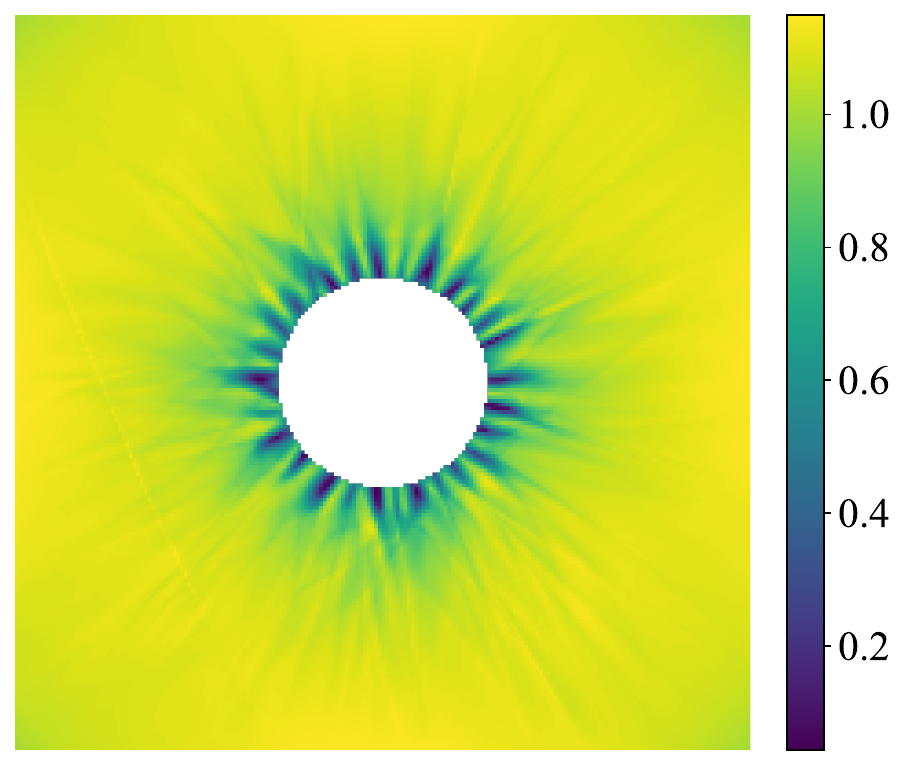}\\[-0.6em]
		{\scriptsize \textbf{(d)} Bio-PINN}
	\end{minipage}
\caption{Single-cell remodelling with weak second-gradient regularization ($\varepsilon_0=0.01\,r_c$). The Jacobian determinant $J=\det F$ (with $F=\nabla y_\theta$) is shown for (a) PINN, (b) RAR-D, (c) residual-driven R3, and (d) Bio-PINN. The weak $H^2$ penalty smooths the finest oscillations across methods, but Bio-PINN still yields a more coherent pericellular band and lower near-cell $J$ values than the baselines shown here.}
	\label{fig:detF-eps001rc}
\end{figure}

% ============================================================
\subsection{Two-cell experiments}\label{sec:two-cell}

We next turn to two interacting cells, where the physics is governed not only by pericellular densification but also by mechanical communication through the intervening matrix. We consider two circular cells
$C_1=\{x:\lvert x-z_1\rvert=r_c\}$ and $C_2=\{x:\lvert x-z_2\rvert=r_c\}$ embedded in
$\Omega_c$ with the same bulk model and boundary conditions (Sec.~\ref{sec:methods:model}). Unless stated otherwise, we use a symmetric configuration with $z_1=(-d/2,0)$, $z_2=(d/2,0)$ and contraction ratio $u_0=0.5$. Throughout this subsection, we focus on the non-regularized limit $\varepsilon_0=0$ so that separation-dependent tether formation is not pre-smoothed by explicit interfacial regularization. This is also the setting most directly connected to the mechanically interacting acini image in Fig.~\ref{fig:intro-exp-num}, which provides an experimental visual reference for the corridor- and band-like motifs discussed below.

\subsubsection{Interaction and tether formation without second-gradient regularization with $\varepsilon_0=0$}
\label{subsubsec:double-eps0}
We examine two representative cell separations, $d\in\{5r_c,\,2.5r_c\}$, and compare the Jacobian determinant $J=\det F$ across methods in Fig.~\ref{fig:detF-2cells-eps0}. The larger separation corresponds to weak interaction, whereas the smaller separation promotes a localized densification band in the intercellular gap. As in the single-cell case, the relevant question is not merely whether a low-$J$ structure appears, but whether the method reaches the densified-well neighbourhood in the pericellular and intercellular regions without introducing fragmented or spoke-like artefacts. Representative sensitivity studies for this two-cell non-regularized regime are reported in Supplementary Figs.~\ref{fig:ed-two-eps0-period} to \ref{fig:ed-two-eps0-uqrho}, including sweeps over the R3 cadence $P$, network capacity, and the UQ probe variance $\rho_{uq}$.

\begin{figure}[t]
	\centering
	{\small \textbf{Top row, $d=5r_c$, weak coupling}}
	\vspace{0.2em}

	\begin{minipage}[t]{0.24\linewidth}
		\centering
		\includegraphics[width=\linewidth]{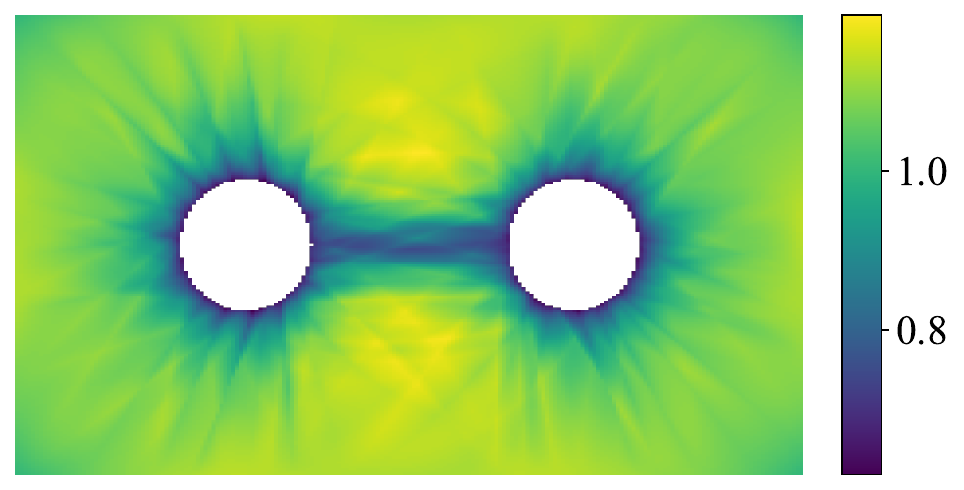}\\[-0.6em]
		{\scriptsize \textbf{(a)} PINN}
	\end{minipage}\hfill
	\begin{minipage}[t]{0.24\linewidth}
		\centering
		\includegraphics[width=\linewidth]{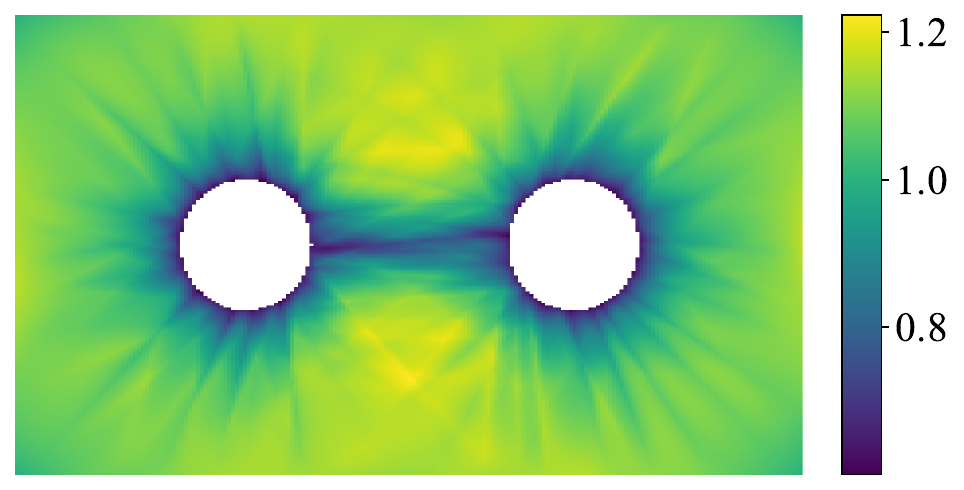}\\[-0.6em]
		{\scriptsize \textbf{(b)} RAR-D}
	\end{minipage}\hfill
	\begin{minipage}[t]{0.24\linewidth}
		\centering
		\includegraphics[width=\linewidth]{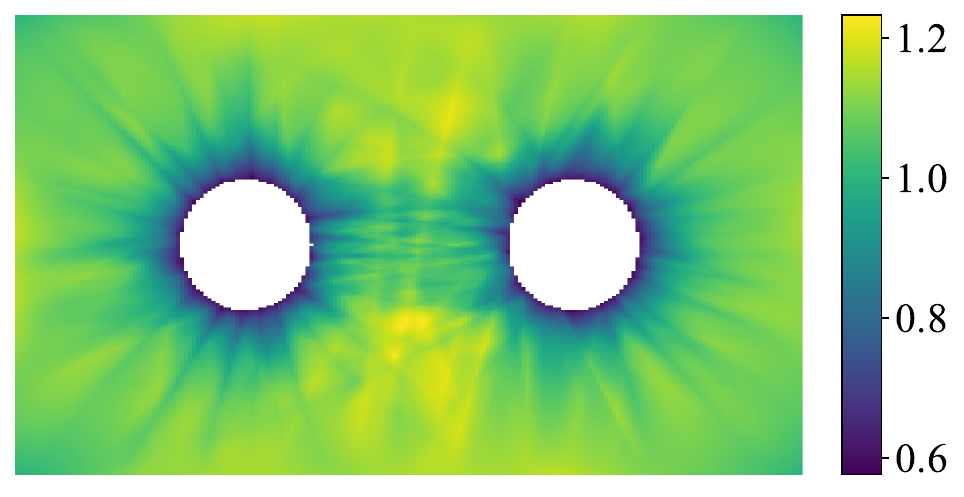}\\[-0.6em]
		{\scriptsize \textbf{(c)} R3}
	\end{minipage}\hfill
	\begin{minipage}[t]{0.24\linewidth}
		\centering
		\includegraphics[width=\linewidth]{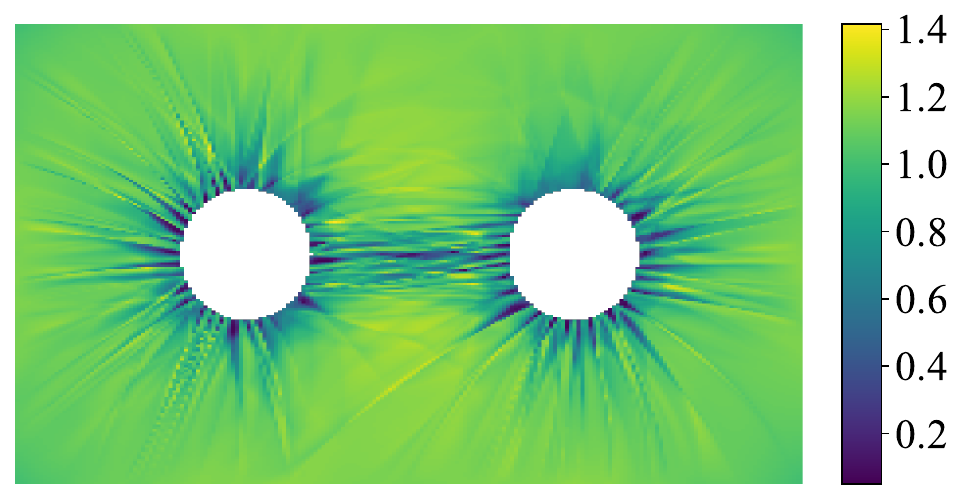}\\[-0.6em]
		{\scriptsize \textbf{(d)} Bio-PINN}
	\end{minipage}
	
	\vspace{0.6em}
	
	{\small \textbf{Bottom row, $d=2.5r_c$, strong coupling}}
	\vspace{0.2em}
	
	\begin{minipage}[t]{0.24\linewidth}
		\centering
		\includegraphics[width=\linewidth]{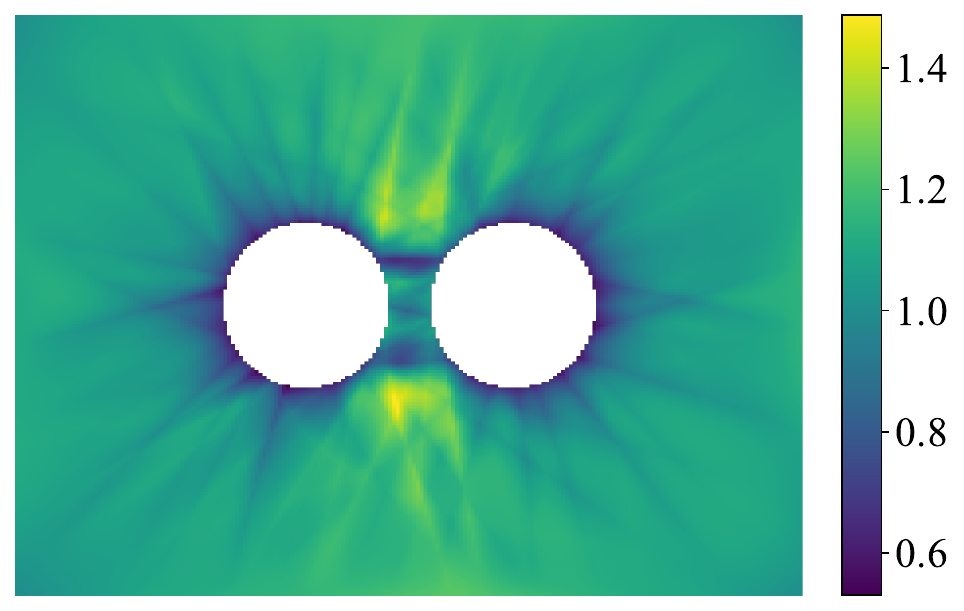}\\[-0.6em]
		{\scriptsize \textbf{(e)} PINN}
	\end{minipage}\hfill
	\begin{minipage}[t]{0.24\linewidth}
		\centering
		\includegraphics[width=\linewidth]{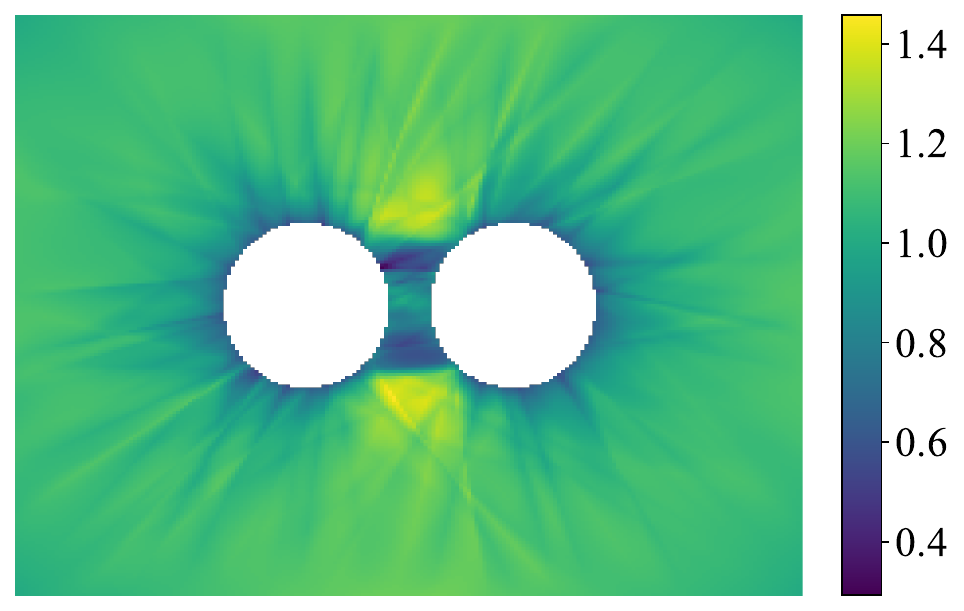}\\[-0.6em]
		{\scriptsize \textbf{(f)} RAR-D}
	\end{minipage}\hfill
	\begin{minipage}[t]{0.24\linewidth}
		\centering
		\includegraphics[width=\linewidth]{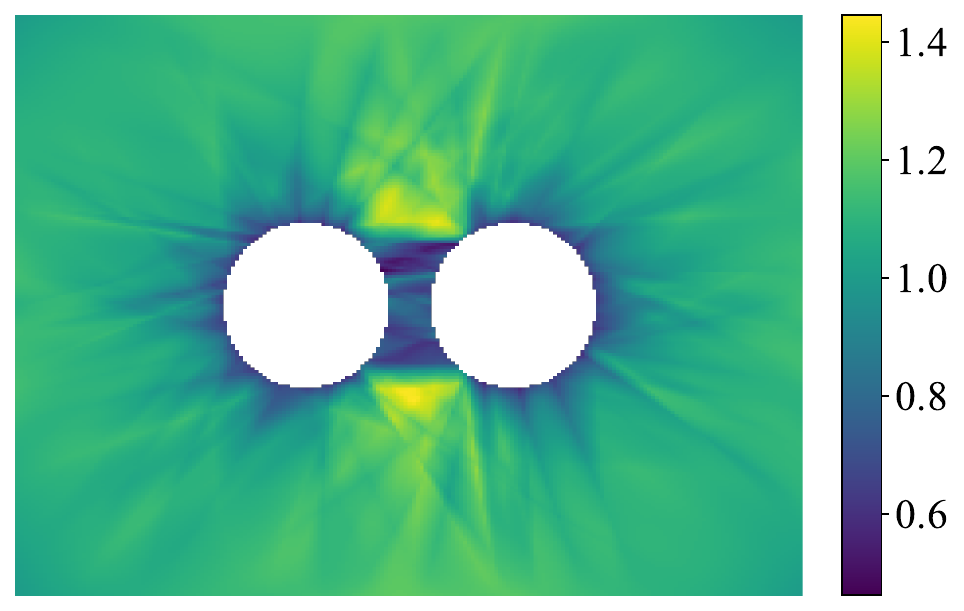}\\[-0.6em]
		{\scriptsize \textbf{(g)} R3}
	\end{minipage}\hfill
	\begin{minipage}[t]{0.24\linewidth}
		\centering
		\includegraphics[width=\linewidth]{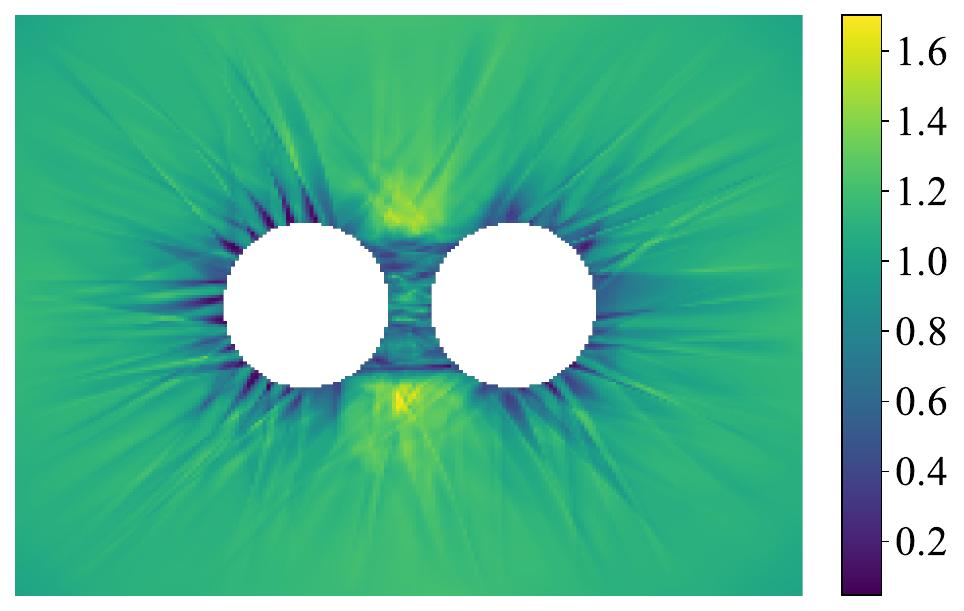}\\[-0.6em]
		{\scriptsize \textbf{(h)} Bio-PINN}
	\end{minipage}
	
\caption{Two-cell remodelling without second-gradient regularization ($\varepsilon_0=0$). The Jacobian determinant $J=\det F$ (with $F=\nabla y_\theta$) is shown for a long-distance regime (top row; $d=5\,r_c$, weak coupling) and a short-distance regime (bottom row; $d=2.5\,r_c$, strong coupling). The long-distance case probes weak coupling between cells, whereas the short-distance case probes tether formation in the intercellular gap. Panels (a,e) correspond to PINN, (b,f) to RAR-D, (c,g) to residual-driven R3, and (d,h) to Bio-PINN. Among the methods shown here, Bio-PINN produces a more contiguous densified band next to each cell and in the interaction region, with less spurious structure in the surrounding matrix. In particular, panel (d) recovers the same pair of motifs emphasized in Fig.~\ref{fig:intro-exp-num}: an intercellular corridor and an outer pericellular band, whereas panel (h) sharpens the intercellular low-$J$ region into a tether.}
	\label{fig:detF-2cells-eps0}

\end{figure}

\subsubsection{Long-distance regime with $d=5r_c$}
When $d = 5r_c$, each cell induces a pericellular densification band similar to the single-cell pattern, while the intercellular ECM remains predominantly in the undensified state with only minor localized densification. The resulting tethers are therefore fewer and weaker. In this setting, morphology alone is not decisive. Whether the densified phase is reached is reflected by the minimum attained value of $J$ and by the near-boundary levels indicated by the colour scale. In Fig.~\ref{fig:detF-2cells-eps0}, panels (a) to (c) show that the baselines reproduce two pericellular low-$J$ bands, yet they plateau at substantially higher minima, with the colour scale bottoming out around $J\gtrsim 0.6$ and, in some cases, closer to $0.7$. Optimization therefore fails to reach the neighbourhood of the densified well ($J\approx J_\star$) near the cell boundaries, even though a band is visible. By contrast, Bio-PINN in panel (d) drives $J$ towards the densified well, with values approaching $J\approx J_\star$ adjacent to each cell. The resulting bands are contiguous and more azimuthally uniform, and spurious structure in the surrounding matrix is reduced. Qualitatively, this weak-coupling morphology matches the two motifs emphasized in the experimental panel of Fig.~\ref{fig:intro-exp-num}: a mechanically organized corridor in the intercellular region and a localized pericellular band on the outer flank. We regard this as qualitative agreement rather than one-to-one validation, since the acinus experiment is not a parameter-matched calibration case, but it supports the physical plausibility of the weak-coupling Bio-PINN pattern. This behaviour is consistent with delaying far-field participation and concentrating sampling near the advancing pericellular transition layers, and it persists under representative variations of $P$ and network capacity (Supplementary Figs.~\ref{fig:ed-two-eps0-period} to \ref{fig:ed-two-eps0-width}).

\subsubsection{Short-distance regime with $d=2.5r_c$ and tether formation}
When $d=2.5r_c$, the pericellular remodelling zones overlap and an intercellular tether forms in the gap. This tether corresponds to the smallest values of $J$ in panels (e) to (h) of Fig.~\ref{fig:detF-2cells-eps0} and indicates a localized transition towards the densified-well neighbourhood ($J\approx J_\star$) within the intercellular region.

All methods indicate the qualitative presence of a tether, but the degree of densification differs markedly. PINN and RAR-D remain at elevated minima, with the colour scale bottoming out around $J\gtrsim 0.4$, which indicates an incomplete transition from the undensified state towards the densified state in the gap and in the pericellular regions near the cell boundaries. R3 can further reduce $J$ by aggressively concentrating samples, yet the smallest values often appear in fragmented bands accompanied by spurious spoke-shaped streaks around the intercellular interaction region. This points to sensitivity when resampling is repeatedly concentrated on thin high-error sets. Bio-PINN yields a more localized and contiguous densified band in the gap and reaches values approaching the densified-well level $J\approx J_\star$, while exhibiting fewer spurious spokes and less disturbance of the surrounding matrix. Relative to the corridor-like organization highlighted in Fig.~\ref{fig:intro-exp-num}, panel (h) can be read as a stronger-coupling regime in which the same mechanically communicating gap region sharpens into a denser tether-like bridge. Again, this is a qualitative interpretation rather than a direct geometric fit, but it helps connect the experimentally observed intercellular organization to the separation-dependent tether selected by the model. This behaviour is consistent with the design of the method. The near-to-far gate advances the active region outward without premature leakage, and the UQ-R3 update concentrates collocation effort on the two interacting transition layers and the intercellular gap rather than chasing thin and transient high-error sets.

\subsubsection{Role of the UQ probe variance}
The role of $\rho_{uq}$ is consistent with the single-cell case. Smaller values emphasize sharper, more localized hotspots, whereas larger values produce smoother but more diffuse updates. The qualitative ordering of methods remains stable under representative sweeps of $P$ and network capacity (Supplementary Figs.~\ref{fig:ed-two-eps0-period} to \ref{fig:ed-two-eps0-uqrho}).

\subsection{Three-cell experiments}\label{sec:exp:3cells}

We finally consider a three-cell configuration as a geometry stress test for multicellular remodelling in the non-regularized regime with $\varepsilon_0=0$. The setup follows the two-cell case described in Sec.~\ref{sec:methods:model}. Three circular cells are embedded in $\Omega_c$, share identical boundary conditions, and are trained with the same protocol. In this setting, the model must resolve several sharp pericellular transition layers together with a possible multicell interaction core, making the phase pattern more sensitive to sample allocation and training order than in the two-cell case. As before, we consider two representative separation regimes corresponding to weak coupling at long distances and strong coupling at short distances. We further assess the robustness of this stress test to the resampling cadence and UQ-proxy settings (Supplementary Figs.~\ref{fig:ed-three-eps0-period} to \ref{fig:ed-three-eps0-uqrho}).

\subsubsection{Representative fields and comparison using $\det F$}
Figure~\ref{fig:eps0-3cells-cmpr} compares three-cell solutions in terms of the Jacobian determinant $J=\det F$ for $\varepsilon_0=0$ against three baselines in both separation regimes. In this geometry, visible low-$J$ structure alone is not enough. A successful solution must reach the densified-well neighbourhood next to each cell while keeping the surrounding matrix close to the undensified state. The colour scale is therefore informative because it distinguishes methods that merely sketch the interaction pattern from those that recover the densified phase in the mechanically decisive multicell core. On this criterion, Bio-PINN comes closest to the low-$J$ band adjacent to each cell while keeping the far field close to $J\approx 1$.

In the long-distance regime, all methods indicate three pericellular densification bands around the cells. However, the baselines plateau at substantially larger minima, with their colour scales bottoming out around $J\gtrsim 0.5$ and, in some cases, closer to $0.6$. This suggests that near-cell solutions do not reach the neighbourhood of the densified well ($J\approx J_\star$) even when a pericellular low-$J$ band is visible. By contrast, Bio-PINN reaches the densified-well neighbourhood around each cell, producing a dark band with $J$ near $J_\star$. This indicates that the near-boundary transition layers are not only localized but also attained more closely than in the baseline solutions.

In the short-distance regime, overlap of the three pericellular remodelling zones produces a compact densified multicell core, which can be interpreted as a tether network junction in the intercellular ECM. The baselines again do not reach the densified-well neighbourhood near the cell boundaries, and typical minima within the narrow interaction region remain around $0.4$ to $0.5$. Bio-PINN attains smaller values, approaching $J\approx J_\star$, and maintains a clearer separation between the localized interaction region and the surrounding matrix. Overall, the three-cell case highlights the central challenge at $\varepsilon_0=0$. Multiple sharp layers and an emergent multicell interaction core must be optimized simultaneously. Among the methods compared here, Bio-PINN is distinguished not only by reduced artefacts but also by its closer approach to the densified-well neighbourhood indicated by the colour scale.

\subsubsection{Robustness and role of the UQ probe variance}
The qualitative conclusions above persist under representative variations of the resampling cadence (Supplementary Fig.~\ref{fig:ed-three-eps0-period}). Moreover, the UQ probe variance $\rho_{uq}$ provides an interpretable control over the effective spatial scale of the probing proxy. Smaller $\rho_{uq}$ yields sharper, more localized emphasis on the multicell interaction core, but can introduce noisier and more transient updates, whereas larger $\rho_{uq}$ increases spatial averaging and produces smoother but more diffuse updates while preserving the same interaction morphology and low-$J$ attainment (Supplementary Fig.~\ref{fig:ed-three-eps0-uqrho}).
\begin{figure}[t]
  \centering
  % -------------------- long distance (weak coupling) --------------------
  \begin{subfigure}{0.24\linewidth}
    \centering
    \includegraphics[width=\linewidth]{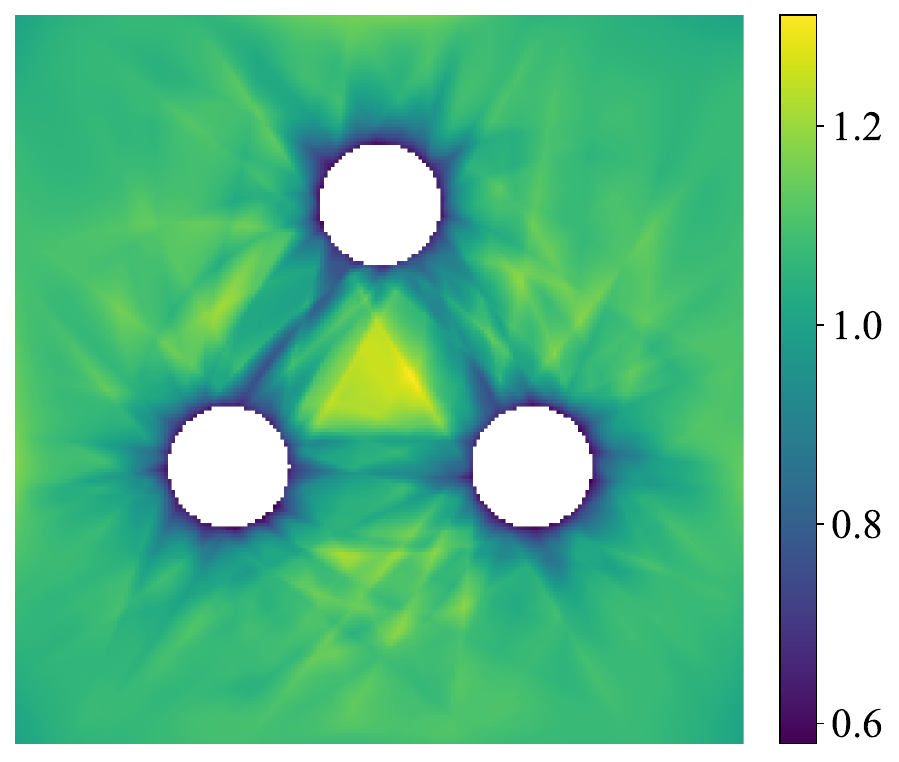}
    \caption{PINN}
  \end{subfigure}\hfill
  \begin{subfigure}{0.24\linewidth}
    \centering
    \includegraphics[width=\linewidth]{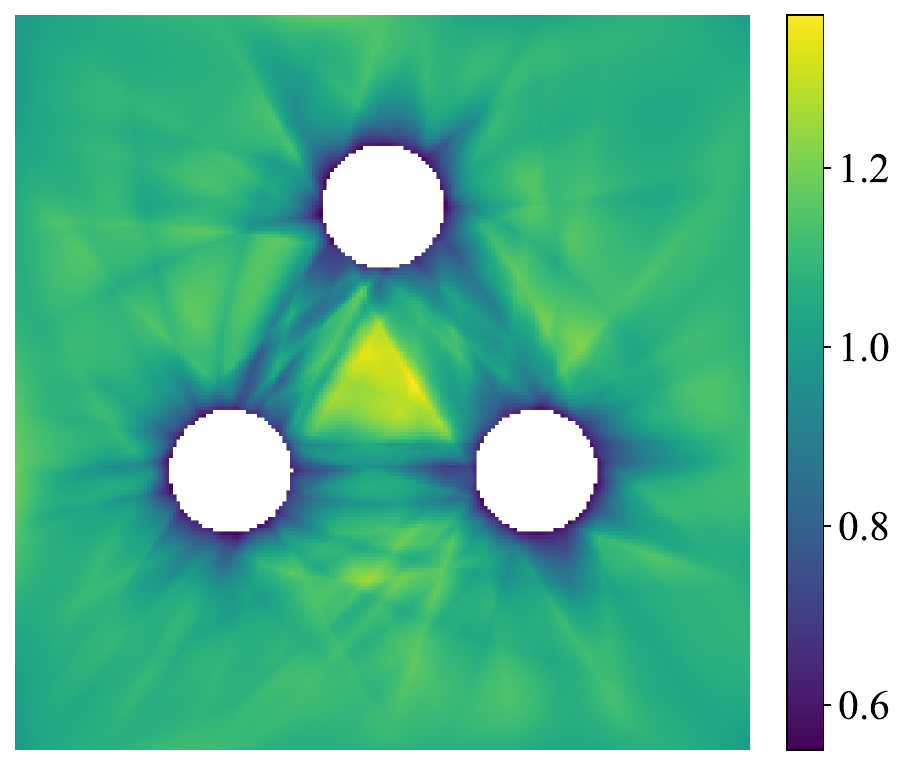}
    \caption{RAR-D}
  \end{subfigure}\hfill
  \begin{subfigure}{0.24\linewidth}
    \centering
    \includegraphics[width=\linewidth]{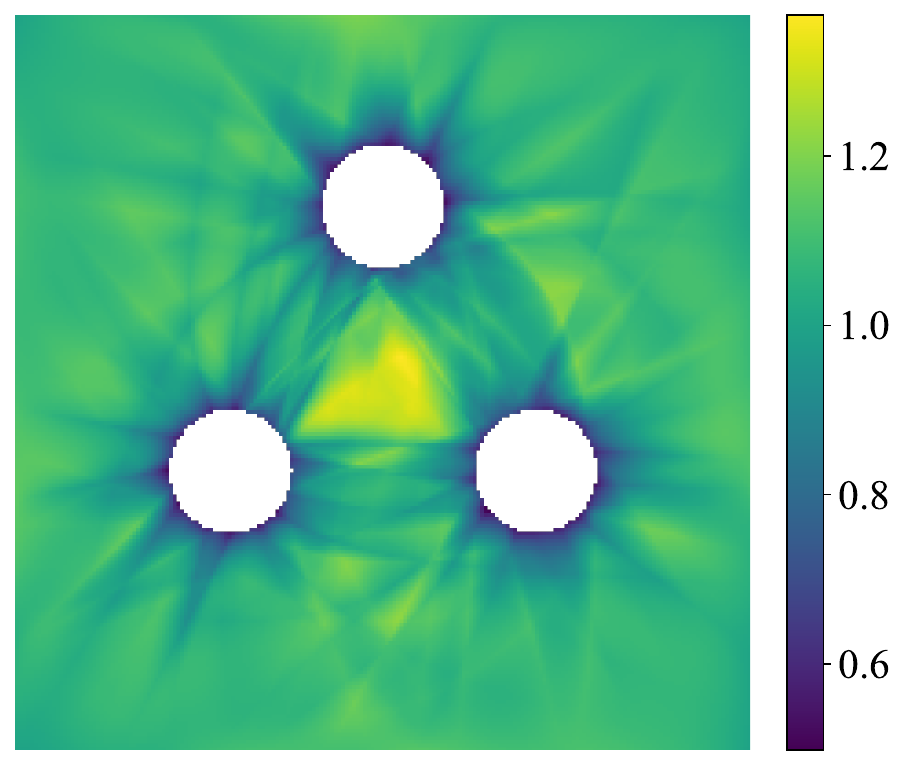}
    \caption{R3}
  \end{subfigure}\hfill
  \begin{subfigure}{0.24\linewidth}
    \centering
    \includegraphics[width=\linewidth]{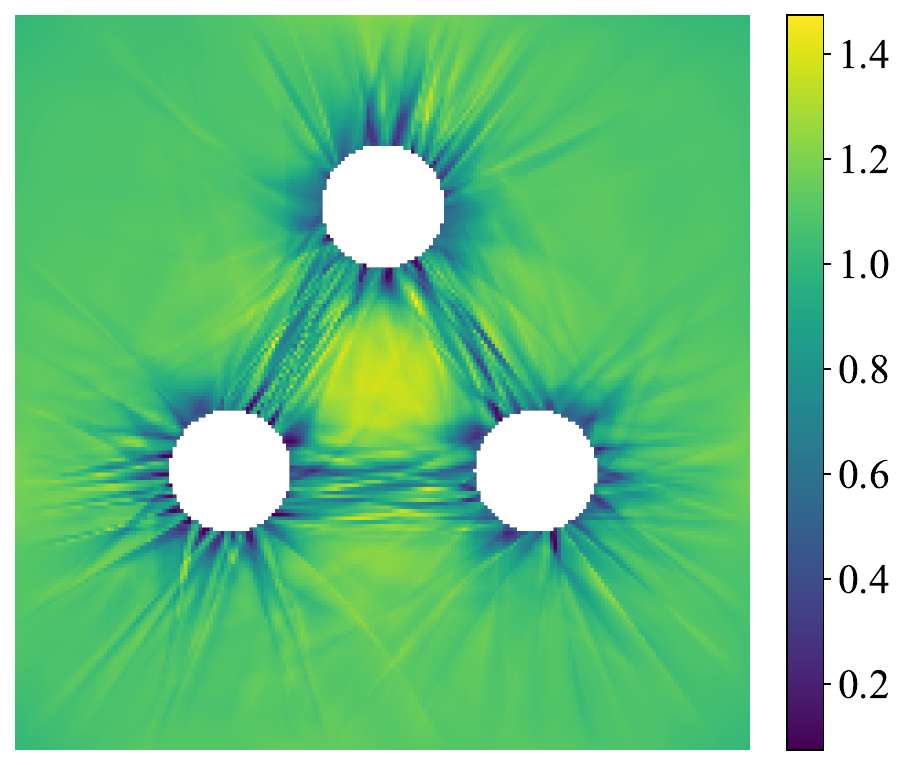}
    \caption{Bio-PINN}
  \end{subfigure}

  \vspace{0.3em}

  % -------------------- short distance (strong coupling) --------------------
  \begin{subfigure}{0.24\linewidth}
    \centering
    \includegraphics[width=\linewidth]{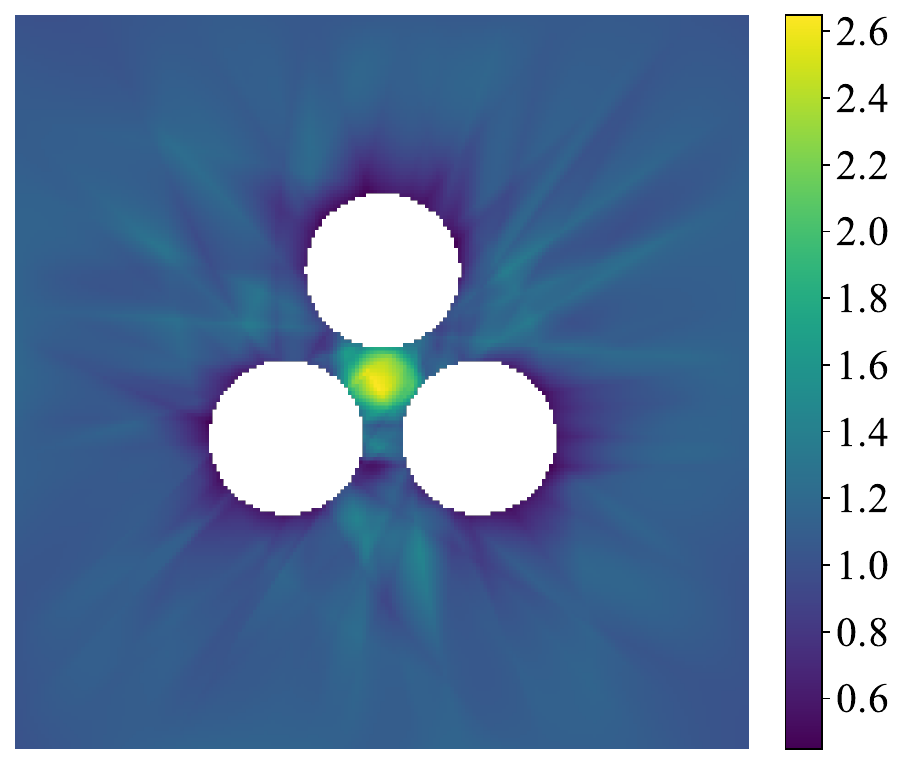}
    \caption{PINN}
  \end{subfigure}\hfill
  \begin{subfigure}{0.24\linewidth}
    \centering
    \includegraphics[width=\linewidth]{Figs/baseline_three_cell_order1_long_star_rad.pdf}
    \caption{RAR-D}
  \end{subfigure}\hfill
  \begin{subfigure}{0.24\linewidth}
    \centering
    \includegraphics[width=\linewidth]{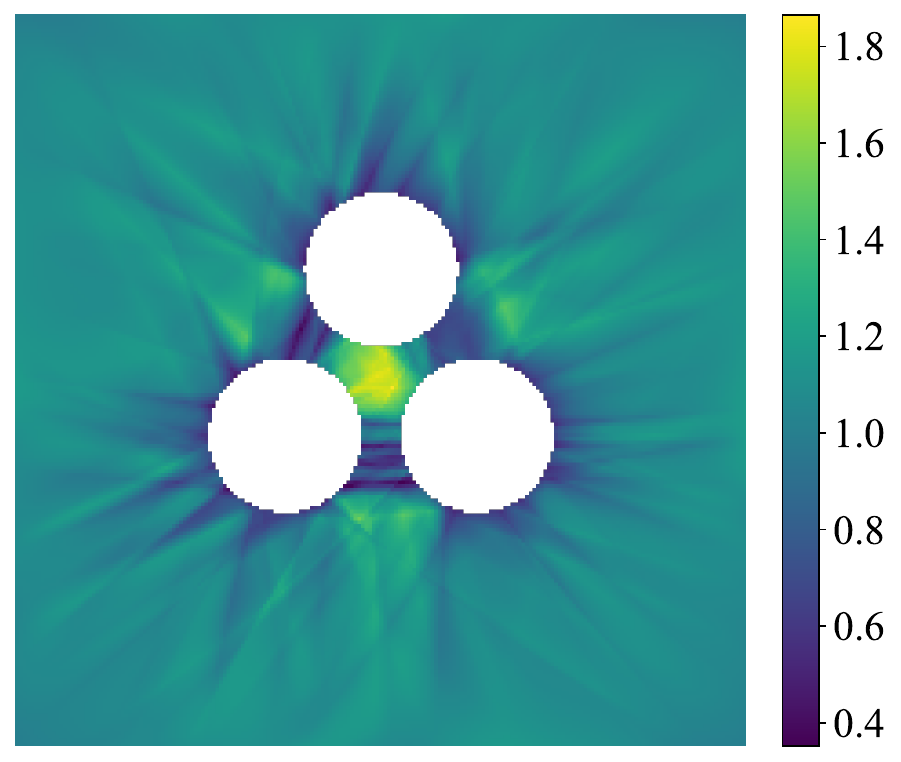}
    \caption{R3}
  \end{subfigure}\hfill
  \begin{subfigure}{0.24\linewidth}
    \centering
    \includegraphics[width=\linewidth]{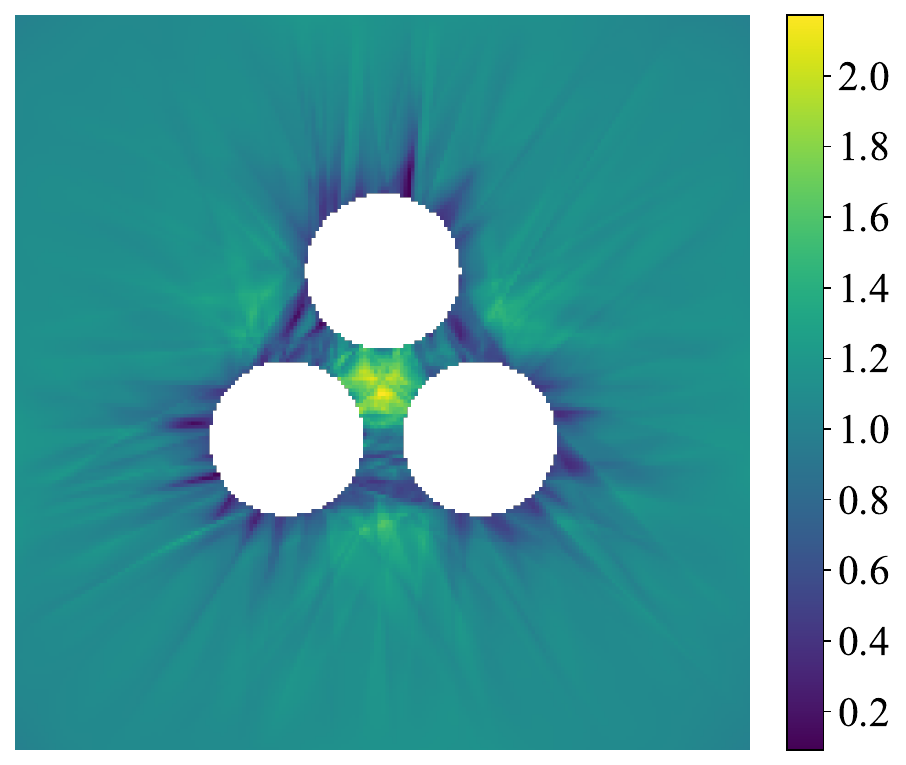}
    \caption{Bio-PINN}
  \end{subfigure}

\caption{Three-cell remodelling stress test without second-gradient regularization ($\varepsilon_0=0$). The Jacobian determinant $J=\det F$ (with $F=\nabla y_\theta$) is compared across methods for a long-distance regime (top row; weak coupling) and a short-distance regime (bottom row; strong coupling). The colour scale indicates whether the densified phase is reached around each cell and in the multicell interaction core. Bio-PINN reaches lower near-cell and interaction-core $J$ values than the baselines while keeping the surrounding matrix closer to the undensified state.}
  \label{fig:eps0-3cells-cmpr}
\end{figure}
\section{Discussion}\label{sec:discussion}

Contractile cells in fibrous extracellular matrices generate remodelling patterns that are difficult to capture numerically because the relevant solutions contain sharp transition layers, pericellular densification, and geometry-dependent intercellular tethers. In this work, we showed that a biomimetic near-to-far training curriculum helps recover these features in a physics-informed setting. The central mechanism is to resolve the mechanically decisive near-cell region before progressively expanding optimization into the bulk, while an uncertainty-guided retain, resample, and release loop reallocates collocation effort towards evolving transition layers and intercellular interaction regions. Across the single-cell, two-cell, and three-cell benchmarks studied here, this combination produced sharper pericellular densification and more faithful tether morphologies than ungated or purely residual-driven adaptive baselines.

The behaviour seen across benchmarks is consistent with the underlying physics of cell-mediated remodelling. In these systems, densification nucleates next to the cell boundary and only then extends outward or into intercellular channels. A training procedure that exposes the full domain too early must simultaneously fit sharp near-cell transition layers and an almost undeformed far field, which makes oversmoothing and misplaced sampling more likely. The near-to-far gate reduces this conflict by prioritizing the region where phase selection is decided, while the UQ-R3 update keeps collocation points near transition layers and tether-forming gaps once those structures appear. This helps explain why Bio-PINN more readily reaches the densified phase in the single-cell pericellular band, the two-cell gap, and the three-cell interaction core. The structural properties established for Retain, Resample, and Release, together with quantitative near-to-far accumulation under gate expansion, support this interpretation by showing that sampling mass is preserved and propagated outward in a controlled way. The repeated use of Fig.~\ref{fig:intro-exp-num} should therefore be understood as a qualitative experimental anchor rather than a quantitative validation dataset. What is encouraging is that the same two motifs recur across experiment and simulation: an outer pericellular band and an intercellular mechanically organized region that remains corridor-like under weaker interaction and sharpens into a tether under stronger interaction. In this sense, the two-cell Bio-PINN benchmarks do not merely produce visually plausible patterns; they recover the experimentally motivated structures that the model is designed to explain.

The present study also has clear boundaries. Our tests focus on two-dimensional perforated-domain benchmarks and primarily assess morphology, densified-phase attainment, and maintenance of the surrounding matrix near the undensified state rather than full error against an independent high-resolution numerical reference. In addition, the uncertainty proxy functions here as a practical guide to interfacial localization, but it is not equivalent to explicit second-gradient regularization. These limitations point to several next steps, including benchmarking against conventional solvers, extending the framework to three-dimensional or data-constrained ECM-remodelling problems, and combining near-to-far adaptive sampling with operator-learning approaches for parameter sweeps. More broadly, the results suggest that physics-aligned spatial curricula may also be useful in other nonconvex and interfacial systems, including fracture, damage, phase-field, and free-boundary models, where decisive structures nucleate locally and then propagate through the domain.

\section{Methods}\label{sec:methods}
The Methods below specify the continuum model, the physics-informed approximation, and the adaptive sampling procedure used to recover pericellular densification and intercellular tether formation. Proofs of the theoretical statements are deferred to Appendix~\ref{app:proofs}.

\subsection{Nonconvex variational model on a perforated domain}\label{sec:methods:model}
We model cell-induced remodelling in a two-dimensional fibrous ECM as a nonconvex variational problem posed on a perforated domain whose holes represent contractile cells. Let $\Omega\subset\mathbb{R}^2$ be a bounded Lipschitz domain and let
$B_i=\{x\in\mathbb{R}^2:\ |x-c_i|<r_c\}$ ($i=1,\dots,N_c$) denote circular cell regions with common radius $r_c>0$. We work on the perforated domain
\[
\Omega_c:=\Omega\setminus\bigcup_{i=1}^{N_c}B_i,\quad
\Gamma_{\mathrm{out}}:=\partial\Omega,\quad
C_i:=\partial B_i,\quad
\Gamma_{\mathrm{in}}:=\bigcup_{i=1}^{N_c} C_i.
\]
A deformation is a map $y:\Omega_c\to\mathbb{R}^2$ with displacement $u:=y-\mathrm{id}$ and deformation gradient $F:=\nabla y$. We write $J:=\det F>0$ and $I_1:=\mathrm{tr}(F^\top F)$. The ratio of deformed to undeformed density is $1/J$, so smaller values of $J$ correspond to stronger densification.

\subsubsection{Bulk strain energy and Jacobian barrier}
We model the collagen ECM as a random network of fibres and adopt a multiscale, orientation-averaged continuum energy. The resulting strain-energy density is non-rank-one-convex, reflecting compression-induced fibre microbuckling and the associated densification phase transition. In our experiments, the single-fibre constitutive law is
\[
S(\lambda)=w'(\lambda)=\mu(\lambda^{5}-\lambda^{3}),\qquad \lambda>0,\ \mu>0,
\]
with primitive $w(\lambda)=\mu(\tfrac{1}{6}\lambda^{6}-\tfrac{1}{4}\lambda^{4})$.
Assuming isotropic in-plane fibre orientations, the macroscopic energy admits the
invariant closed form
\begin{equation}\label{eq:methods:W}
W(F)=\frac{\mu}{96}\Big(5I_1^{3}-9I_1^{2}-12\,I_1J^{2}+12J^{2}+8\Big).
\end{equation}
Owing to compression-induced fibre microbuckling and buckling instabilities, the orientation-averaged bulk energy exhibits a nonconvex multi-well landscape associated with phase transitions. This energy is expressed in terms of the principal stretches $\lambda_1$ and $\lambda_2$ of the deformation gradient $F$, up to the addition of a null Lagrangian that does not affect the Euler--Lagrange equations under fixed boundary data. For the standard parameter values used in collagen ECM simulations, two energetically preferred states arise. The first is the undeformed reference state at $\lambda_1=1$ and $\lambda_2=1$, corresponding to the low-density well. The second is the compressed, densified well, in which the principal stretches are approximately 0.2 and 1.06. The latter corresponds to strong compression in one direction combined with mild extension in the orthogonal direction and gives a characteristic volume ratio $J_\star = \lambda_1\lambda_2 \approx 0.21$. Because the model is isotropic, the densified well admits symmetry-related variants obtained by swapping $\lambda_1$ and $\lambda_2$ and by rotation. In what follows, we refer to the neighbourhood of $J\approx 1$ as the undeformed or low-density well and to the neighbourhood of $J\approx J_\star$ as the compressed or densified well. These two wells represent distinct density states of the collagen network associated with the densification phase transition.

To discourage loss of orientation and interpenetration, we add a Jacobian barrier
\begin{equation}\label{eq:methods:Phi}
\Phi(J)=\exp\!\big[A(b-J)\big],\qquad A\gg1,\;0<b\ll1,
\end{equation}
which is negligible for $J\gtrsim b$ and grows rapidly as $J\downarrow b$.

\subsubsection{Boundary data and contraction penalty}
On the outer boundary, we impose the hard Dirichlet condition $y(x)=x$ for
$x\in\Gamma_{\mathrm{out}}$. On each cell boundary $C_i$, we prescribe a radial
contraction target map
\[
g_i(x)=c_i + (1-u_0)\,(x-c_i),\qquad x\in C_i,
\]
where $u_0\in(0,1)$ is the contraction ratio. We enforce this inner condition
softly via a quadratic penalty
\begin{equation}\label{eq:methods:Pin}
\mathcal{P}_{\mathrm{in}}[y]
=\sum_{i=1}^{N_c}\frac{\gamma_{\mathrm{in}}}{2}\int_{C_i}\|y-g_i\|^{2}\,ds,
\qquad \gamma_{\mathrm{in}}>0.
\end{equation}

\subsubsection{Second-gradient regularization and admissible class}
For $\varepsilon\ge 0$ we consider the energy
\begin{equation}\label{eq:methods:energy}
\mathcal{E}_\varepsilon[y]
=\int_{\Omega_c}\!\Big(W(\nabla y)+\Phi(\det\nabla y)\Big)\,dx
+\frac{\varepsilon^{2}}{2}\int_{\Omega_c}\!\|\nabla^{2}y\|_{F}^{2}\,dx
+\mathcal{P}_{\mathrm{in}}[y],
\end{equation}
over the admissible class
\[
\mathcal{A}_\varepsilon
:=\Bigl\{y\in H^2(\Omega_c;\mathbb{R}^2):
\det\nabla y>0\ \text{a.e.},\ y|_{\Gamma_{\mathrm{out}}}=x\Bigr\}.
\]
The second-gradient term introduces an intrinsic length scale and stabilizes
interfaces of width $O(\varepsilon)$.

% ============================================================
\subsection{Bio-PINN ansatz and weak boundary handling}\label{sec:methods:ritz}
The neural ansatz is chosen to keep the outer matrix boundary fixed while allowing contraction-induced transition layers and tethers to develop inside the perforated domain. We approximate $y$ by a feed-forward network. Let $u_\theta:\Omega_c\to\mathbb{R}^2$ be a fully connected neural network with differentiable activation $\sigma$ (we use \textsc{CELU} in all experiments), so that $\nabla u_\theta$ and $\nabla^2 u_\theta$ exist and are accessible through automatic differentiation. To enforce the hard outer Dirichlet condition $y=x$ on $\Gamma_{\mathrm{out}}$ exactly, we use a lifting construction.
\begin{equation}\label{eq:methods:ytheta}
y_\theta(x)=x+\varphi(x)\,u_\theta(x),
\end{equation}
where $\varphi:\Omega\to\mathbb{R}$ is a smooth shape function satisfying
$\varphi|_{\Gamma_{\mathrm{out}}}=0$ and therefore $y_\theta=x$ on $\Gamma_{\mathrm{out}}$.
Inner boundary contraction on $\Gamma_{\mathrm{in}}=\cup_i C_i$ is imposed softly
by adding $\mathcal{P}_{\mathrm{in}}[y_\theta]$ from \eqref{eq:methods:Pin} to the
training objective. 

% ============================================================
\subsection{Near-to-far distance gating and stage-wise curriculum}\label{sec:methods:gating}
Cell-induced remodelling begins at $\Gamma_{\mathrm{in}}$ and propagates into the bulk. The gate formalizes this spatial progression by revealing the domain from near to far. We encode this progression through a normalized distance field
\begin{equation}\label{eq:methods:dist}
\tilde d(x):=\frac{\mathrm{dist}(x,\Gamma_{\mathrm{in}})}{d_{\max}}\in[0,1],
\qquad
d_{\max}:=\max_{x\in\Omega_c}\mathrm{dist}(x,\Gamma_{\mathrm{in}}).
\end{equation}
At stage $i$, we apply a smooth gate
\begin{equation}\label{eq:methods:softgate}
g_{\gamma_i}(x)=\sigma\!\big(\alpha(\gamma_i-\tilde d(x))\big)\in(0,1),
\end{equation}
where $\sigma$ is the logistic sigmoid, $\alpha>0$ controls steepness, and
$\gamma_i\in\mathbb{R}$ controls the revealed region. In experiments, we use
$\alpha=5.0$, initialize $\gamma_0=-0.5$, and cap per-stage increments by
$\Delta_{\max}=0.05$.

For intuition and analysis, it is convenient to associate each stage with an effective hard-gated set using the clamped gate level $\bar\gamma_i:=\Pi_{[0,1]}(\gamma_i)$.
\begin{equation}\label{eq:methods:hardgate}
A_i=\{x\in\Omega_c:\tilde d(x)\le\bar\gamma_i\},
\qquad
S_{i+1}=A_{i+1}\setminus A_i.
\end{equation}
We advance the gate using a near-to-far curriculum driven by the current
objective value $\mathcal{L}_i$.
\begin{equation}\label{eq:methods:gateupdate}
\gamma_{i+1}
=\gamma_i+\min\!\left\{\Delta_{\max},\ \eta_g\,e^{-c\,\mathcal{L}_i}\right\},
\end{equation}
with $\eta_g>0$ and $c>0$. Importantly, gating is applied only to the bulk density terms $W+\Phi$ and the optional $H^2$ term, and not to the inner boundary penalty $\mathcal{P}_{\mathrm{in}}$, so boundary information is enforced
throughout training.

% ============================================================
\subsection{Uncertainty proxy and gated normalization}\label{sec:methods:uq}
Interfaces, thin layers, and tethers are the regions where the learned deformation is most sensitive during training. We therefore use gradient-based local probing to construct a sampling indicator for likely transition layers and intercellular channels. Given the current parameters at stage $i$, for each $x\in\Omega_c$ we draw $m_{\mathrm{uq}}$ i.i.d.\ Gaussian perturbations $\delta^{(k)}\sim\mathcal{N}(0,\rho_{\mathrm{uq}}I)$ and define
\begin{equation}\label{eq:methods:U}
U_i(x)
:=\operatorname{Var}_{k=1,\dots,m_{\mathrm{uq}}}\!\Big(\|\nabla y_\theta(x+\delta^{(k)})\|_F\Big).
\end{equation}
In experiments we use $m_{\mathrm{uq}}=16$ and $\rho_{\mathrm{uq}}=0.01$.

To stabilize the scale of the proxy and remain consistent with the currently revealed region, we normalize $U_i$ using a gated quantile-shrink rule computed empirically over the current collocation set (Sec.~\ref{sec:methods:r3}). Let $q^{(g)}_\alpha$ denote the weighted $\alpha$-quantile of $\{U_i(x):x\in\mathcal{S}_i\}$ under weights $g_{\gamma_i}(x)$. For fixed $0<\alpha_-<\alpha_+<1$, define
\begin{equation}\label{eq:methods:quantilenorm}
\widetilde U_i(x)
=\operatorname{clip}_{[0,1]}\!\left(
\frac{U_i(x)-q^{(g)}_{\alpha_-}}{q^{(g)}_{\alpha_+}-q^{(g)}_{\alpha_-}}\right),
\end{equation}
which maps scores to $[0,1]$ while suppressing extreme outliers.

% ============================================================
\subsection{UQ-R3 sampling with low-discrepancy resampling}\label{sec:methods:r3}
The collocation set is updated so that points are retained near evolving transition layers while newly revealed regions remain populated. At stage $i$, let $\mathcal{S}_i\subset\Omega_c$ be the current collocation set with fixed budget $N:=|\mathcal{S}_i|$. We compute scores $s_i(x)$ on $\mathcal{S}_i$, where Bio-PINN uses the normalized UQ score $s_i(x)=\widetilde U_i(x)$ and residual-based baselines use a pointwise residual density.

Given retention ratio $\rho\in(0,1)$, we retain the top $\rho$ fraction of points by
score, using gate-weights to define the cutoff. Concretely, let $\tau_i$ be the smallest
threshold such that the total gated weight of points with $s_i(x)\le\tau_i$ accounts for
at least $(1-\rho)$ of the total gated weight.
\begin{equation}\label{eq:methods:tau}
\sum_{x\in\mathcal{S}_i}\mathbf{1}_{\{s_i(x)\le\tau_i\}}\,g_{\gamma_i}(x)
\ \ge\ (1-\rho)\sum_{x\in\mathcal{S}_i} g_{\gamma_i}(x),
\end{equation}
and define the retained set
\begin{equation}\label{eq:methods:Lambda}
R_i:=\{x\in\mathcal{S}_i:\ s_i(x)\ge\tau_i\},\qquad |R_i|\approx \rho N.
\end{equation}
We release the remaining $m_i:=N-|R_i|$ points and resample $m_i$ new points to form
\begin{equation}\label{eq:methods:R3}
\mathcal{S}_{i+1}=R_i\cup C_i,\qquad |C_i|=m_i.
\end{equation}

\subsubsection{Low-discrepancy resampling and shell injection}
New points are generated from low-discrepancy nodes transported into the currently
active region. We draw Hammersley nodes $\{u_{i,j}\}_{j=1}^{m_i}\subset[0,1]^2$ and apply
a (bi-)Lipschitz transport map $T_i:[0,1]^2\to A_i$ to obtain $X_{i,j}=T_i(u_{i,j})$ and
$C_i=\{X_{i,1},\dots,X_{i,m_i}\}$. When shell injection is used, $T_i$ is restricted to map
into the newly opened shell $S_{i+1}=A_{i+1}\setminus A_i$, ensuring exploration of fresh
regions while preserving low-discrepancy coverage up to transport constants.

% ============================================================
\subsection{Gate-weighted empirical objective and training loop}\label{sec:methods:objective}
The training objective emphasizes the portion of the matrix that has already been revealed by the gate while continuing to enforce contraction on the cell boundaries. At stage $i$, we approximate bulk integrals using quasi-Monte Carlo (QMC) quadrature over the collocation set $\mathcal{S}_i$ with fixed budget $N=10^4$ in all experiments. Gating is applied only to bulk densities, and we use a normalized gate-weighted empirical energy.
\begin{align}\label{eq:methods:gatedenergy}
\mathcal{E}^{(g)}_{\varepsilon,i}[y_\theta]
&:=
\frac{\sum_{x\in\mathcal S_i} g_{\gamma_i}(x)\big(W(\nabla y_\theta(x))+\Phi(\det\nabla y_\theta(x))\big)}
{\sum_{x\in\mathcal S_i} g_{\gamma_i}(x)}
+
\frac{\varepsilon^{2}}{2}\,
\frac{\sum_{x\in\mathcal S_i} g_{\gamma_i}(x)\|\nabla^{2}y_\theta(x)\|_{F}^{2}}
{\sum_{x\in\mathcal S_i} g_{\gamma_i}(x)}.
\end{align}
The stage objective adds the inner-boundary penalty.
\begin{equation}\label{eq:methods:Ji}
\mathcal{J}_i(\theta)=\mathcal{E}^{(g)}_{\varepsilon,i}[y_\theta]+\mathcal{P}_{\mathrm{in}}[y_\theta],
\qquad
\mathcal{L}_i:=\mathcal{J}_i(\theta_i).
\end{equation}
All derivatives are computed by automatic differentiation.

\subsubsection{Training schedule}
Training alternates between (i) optimizer steps at fixed $(\gamma_i,\mathcal{S}_i)$, (ii) updating the gate via \eqref{eq:methods:gateupdate}, and (iii) updating the collocation set via UQ-R3 in \eqref{eq:methods:R3}. In practice, we trigger steps (ii) and (iii) every $P$ optimizer iterations. Boundary points on $\Gamma_{\mathrm{out}}$ and $\Gamma_{\mathrm{in}}$ used to evaluate $\mathcal{P}_{\mathrm{in}}$ are sampled randomly and refreshed throughout training.

\subsubsection{Implementation details}
Unless otherwise stated, we use a fully connected network with three hidden layers and width 128, \textsc{CELU} activations, Xavier-uniform initialization, and Adam optimization with learning rate $10^{-3}$ and $(\beta_1,\beta_2)=(0.9,0.999)$. The learning rate is decayed by a factor of $0.9$ every $10{,}000$ iterations. We use early stopping based on a validation estimate of $\mathcal{J}_i$ computed on $N_{\mathrm{val}}=2{,}000$ randomly sampled points. All experiments are implemented in PyTorch and run with a fixed random seed.

\begin{algorithm}[t]
\caption{Bio-PINNs with UQ-R3 and near-to-far distance gating}\label{alg:biopinn}
\KwIn{Initial parameters $\theta_0$, initial gate $\gamma_0$, budget $N$, retention ratio $\rho$,
UQ parameters $(m_{\mathrm{uq}},\rho_{\mathrm{uq}})$, transport maps $\{T_i\}$,
gating parameters $(\eta_g,c,\Delta_{\max})$, resampling period $P$.}
Initialize $\mathcal{S}_0\subset A_0$ with $|\mathcal{S}_0|=N$ (e.g., transported Hammersley nodes)\;
\For{$i=0,1,2,\dots$ until stopping}{
Perform $P$ optimizer steps to approximately minimize $\mathcal{J}_i(\theta)$ on $\mathcal{S}_i$\;
Compute $U_i$ by \eqref{eq:methods:U} and $\widetilde U_i$ by \eqref{eq:methods:quantilenorm} on $\mathcal{S}_i$\;
Compute $\tau_i$ by \eqref{eq:methods:tau}, retain $R_i$ by \eqref{eq:methods:Lambda}, resample $m_i=N-|R_i|$
new points $C_i$ via $T_i$, set $\mathcal{S}_{i+1}=R_i\cup C_i$\;
Update gate $\gamma_{i+1}$ by \eqref{eq:methods:gateupdate} and update $A_{i+1}$ by \eqref{eq:methods:hardgate}\;
}
\end{algorithm}

\subsection{Theoretical properties of UQ-R3 under gating}\label{sec:methods:theory}
The statements below formalize how retained and resampled points behave as the near-to-far gate expands. All proofs are deferred to Appendix~\ref{app:proofs}. For clarity, the results are stated in the hard-gate setting, where $\mu(\cdot)$ denotes the Lebesgue measure on $\mathbb R^2$.

Recall the proxy maximizer level and near-optimal band
\[
\widetilde U_i^\star:=\sup_{x\in A_i}\widetilde U_i(x),
\qquad
B_i(\varepsilon):=\Bigl\{x\in A_i:\ \widetilde U_i^\star-\widetilde U_i(x)\le \varepsilon\Bigr\},
\]
the retained-set mean
\[
\overline U_i
:=\frac{1}{|\mathcal S_i\cap\Lambda_i|}
\sum_{x\in\mathcal S_i\cap\Lambda_i}\widetilde U_i(x),
\]
and the effective tolerance $\varepsilon_i:=\min\{\varepsilon,\ \widetilde U_i^\star-\tau_i\}\in(0,\varepsilon]$.

\begin{assumption}[Regularity needed for discrepancy bounds]\label{ass:reg-proxy-transport}
For each stage $i$, assume the following.
\begin{enumerate}[label=(A\arabic*)]
\item $\widetilde U_i$ is Lipschitz on $A_i$.
\item $T_i:[0,1]^2\to A_i$ is bi-Lipschitz and a.e.\ differentiable. Writing
$J_{T_i}(x):=|\det\nabla T_i^{-1}(x)|$, define the (stage-uniform) transport distortion constant
\[
C_T:=\sup_i\,\mu(A_i)\Bigl\|J_{T_i}-\frac{1}{\mu(A_i)}\Bigr\|_{L^\infty(A_i)}\in[0,1).
\]
\item The low-discrepancy nodes satisfy a star-discrepancy bound
$D^\ast(\{u_{i,j}\})\lesssim (\log m_i)^2/m_i$ (e.g.\ two-dimensional Hammersley points).
\item The measurable sets used below (in particular $B_i(\varepsilon_i)$, $\Lambda_i$, and the shell
$S_{i+1}:=A_{i+1}\setminus A_i$) have finite perimeter, so that BV-mollification/Koksma-Hlawka
estimates apply after pullback by $T_i$.
\end{enumerate}
\end{assumption}

\begin{proposition}[Budget identity for \textsc{Release}]\label{prop:release-budget}
Let $\mathcal S_i\subset A_i$ be the bulk collocation set with budget $N=|\mathcal S_i|$.
Define the retained set $R_i=\mathcal S_i\cap\Lambda_i$ and update
$\mathcal S_{i+1}=R_i\cup C_i$ with $|C_i|=N-|R_i|$. Then the budget is preserved.
\[
|\mathcal S_{i+1}|=N\qquad\text{for all stages }i.
\]
\end{proposition}

\begin{theorem}[Non-emptiness and discrepancy-controlled coverage for \textsc{Resample}]\label{thm:resample-nonempty}
Fix a stage $i$ and let $B\subset A_i$ be measurable with finite perimeter. Define its relative area
\[
p_B:=\frac{\mu(B)}{\mu(A_i)}\in[0,1],
\qquad
\widetilde B:=T_i^{-1}(B)\subset[0,1]^2.
\]
Under Assumption~\ref{ass:reg-proxy-transport}, there exists a constant $C_1>0$
(depending only on $T_i$ and geometric bounds such as $\operatorname{Per}(B)$) such that, for every $m_i$,
\begin{equation}\label{eq:nonempty-45}
\#(C_i\cap B)\ \ge\ m_i\Bigl(|\widetilde B|-C_1\sqrt{D^\ast(\{u_{i,j}\})}\Bigr)_+ .
\end{equation}
Moreover,
\begin{equation}\label{eq:tildeB-lb}
|\widetilde B|\ \ge\ (1-C_T)\,p_B.
\end{equation}
In particular, if
\[
D^\ast(\{u_{i,j}\})\ \le\ \frac{(1-C_T)^2 p_B^{\,2}}{4C_1^{\,2}},
\qquad
m_i\ \ge\ \frac{2}{(1-C_T)\,p_B},
\]
then $\#(C_i\cap B)\ge 1$. If $\{u_{i,j}\}$ are two-dimensional Hammersley points with
$D^\ast(m)\le C_{\mathrm{Ham}}(\log m)^2/m$, it suffices to take
\[
m_i\ \ge\ \kappa\,
\max\!\left\{
\frac{2}{(1-C_T)\,p_B},\
\frac{4C_1^{\,2}}{(1-C_T)^2 p_B^{\,2}}
\left[\log\!\left(\frac{2C_1}{(1-C_T)\,p_B}\right)\right]^2
\right\},
\qquad (\kappa\ge 1),
\]
which guarantees $\#(C_i\cap B)\ge 1$.
\end{theorem}

\begin{theorem}[Hit, no-early-exit, and accumulation under good rounds for \textsc{Retain}]\label{thm:retain-accumulation}
Fix $\varepsilon>0$.
Under Assumption~\ref{ass:reg-proxy-transport}, there exists $m_0(\varepsilon)$ such that for any round with
$m_i:=|C_i|\ge m_0(\varepsilon)$ (a \emph{good round}), we have
\begin{equation}\label{eq:hit}
\#(C_i\cap B_i(\varepsilon_i))\ge 1.
\end{equation}

\emph{(No-early-exit).}
Fix a good round $i$ and let $x\in C_i\cap B_i(\varepsilon_i)$.
If for subsequent rounds $j\ge i$ before ``success'' (i.e.\ while $\overline U_j<\widetilde U_i^\star-\varepsilon$) either
\begin{enumerate}[label=(\roman*)]
\item (\emph{Fixed proxy}) $\widetilde U_j\equiv \widetilde U_i$, or
\item (\emph{Small drift}) $\inf_{y\in B_i(\varepsilon_i)} \widetilde U_j(y)\ge \widetilde U_i^\star-\varepsilon_i$
and $\tau_j\le \widetilde U_i^\star-\varepsilon_i$,
\end{enumerate}
then $x\in\Lambda_j$ for all such $j$ and hence cannot be released before success.

\emph{(Long-term accumulation).}
If the set of good rounds $\{i:\ m_i\ge m_0(\varepsilon)\}$ is infinite, then
\[
\liminf_{n\to\infty}\overline U_n\ \ge\ \widetilde U^\star-\varepsilon ,
\]
i.e.\ along infinitely many good rounds, the retained-set mean approaches the optimal proxy level
(up to an arbitrarily small slack), where $\widetilde U^\star$ denotes the corresponding maximizer level.
\end{theorem}

\begin{corollary}[Guaranteed shell injection]\label{cor:shell-injection}
Let the gate expand from $A_i$ to $A_{i+1}$ and define the shell
$S_{i+1}:=A_{i+1}\setminus A_i$ with relative measure
$p_S:=\mu(S_{i+1})/\mu(A_{i+1})\in(0,1]$.
Let $\widetilde S_{i+1}:=T_{i+1}^{-1}(S_{i+1})\subset[0,1]^2$ and set
$C_{1,S}\lesssim \operatorname{Lip}(T_{i+1}^{-1})\,\operatorname{Per}(S_{i+1})$.
If the next round uses $m_{i+1}$ two-dimensional Hammersley nodes
(so that $D^\ast(m)\le C_{\mathrm{Ham}}(\log m)^2/m$), it suffices to choose
\[
m_{i+1} \;\ge\; \kappa\,
\max\!\left\{
\frac{2}{(1-C_T)\,p_S},\;
\frac{4 C_{1,S}^{\,2}}{(1-C_T)^2 p_S^{\,2}}
\left[\log\!\left(\frac{2 C_{1,S}}{(1-C_T)\,p_S}\right)\right]^2
\right\},\qquad (\kappa\ge 1),
\]
which guarantees $\#(C_{i+1}\cap S_{i+1})\ge 1$.
\end{corollary}

\begin{theorem}[Dynamic injection/accumulation under monotone gating]\label{thm:near-to-far}
Let the training budget be fixed at $N$ and define
\[
\nu_i := \frac1N\sum_{x\in\mathcal S_i}\delta_x,\qquad
A_i\subset A_{i+1},\qquad S_{i+1}:=A_{i+1}\setminus A_i .
\]
Then for any measurable $B\subset A_{i+1}$, letting $p_B:=\mu(B)/\mu(A_{i+1})\in[0,1]$, the following statements hold.

\emph{(1) Decomposition identity, retained plus resampled.}
\begin{equation}\label{eq:decomp}
\nu_{i+1}(B)
=\frac{\#(\mathcal S_i\cap\Lambda_i\cap B)}{N}
+\frac{\#(C_{i+1}\cap B)}{N}.
\end{equation}

\emph{(2) Quantitative lower bound, strong form.}
There exists $C_1>0$ (depending only on geometric bounds for $T_{i+1}$ and $\operatorname{Per}(B)$) such that
\begin{equation}\label{eq:strongLB}
\nu_{i+1}(B)\ \ge\
\frac{\#(\mathcal S_i\cap\Lambda_i\cap B)}{N}
+ \frac{m_{i+1}}{N}\Bigl((1-C_T)\,p_B - C_1\sqrt{D^\ast(\{u_{i+1,j}\})}\Bigr)_+ .
\end{equation}

\emph{(3) Concise bound, conservative form.}
After absorbing constants,
\begin{equation}\label{eq:conciseLB}
\nu_{i+1}(B)\ \ge\
\frac{\#(\mathcal S_i\cap\Lambda_i\cap B)}{N}
+ \frac{m_{i+1}}{N}\Bigl(\tilde p_B - \widetilde C_1\sqrt{D^\ast(\{u_{i+1,j}\})}\Bigr)_+,
\qquad \tilde p_B:=(1-C_T)p_B .
\end{equation}

\emph{In particular (shell injection).}
For $B=S_{i+1}=A_{i+1}\setminus A_i$ (hence $\mathcal S_i\cap\Lambda_i\cap S_{i+1}=\varnothing$),
\begin{equation}\label{eq:shell-injection}
\nu_{i+1}(S_{i+1})
\ \ge\ \frac{m_{i+1}}{N}\Bigl((1-C_T)p_S - \widetilde C_1\sqrt{D^\ast(\{u_{i+1,j}\})}\Bigr)_+,
\qquad p_S:=\frac{\mu(S_{i+1})}{\mu(A_{i+1})} .
\end{equation}
Hence, if
\[
D^\ast(\{u_{i+1,j}\}) \le \Bigl(\tfrac{(1-C_T)p_S}{2\widetilde C_1}\Bigr)^{\!2},
\qquad m_{i+1} \ge \tfrac{2}{(1-C_T)p_S},
\]
then $\nu_{i+1}(S_{i+1}) \ge \dfrac{m_{i+1}}{2N}(1-C_T)p_S>0$, i.e., each round injects
a positive mass into the newly opened shell.
\end{theorem}

\begin{remark}[Where proofs and stronger statements appear]\label{rem:proofs-location}
Appendix~\ref{app:proofs} contains proofs and refined variants of Theorems~\ref{thm:resample-nonempty} to \ref{thm:near-to-far}, including explicit constants and conditions under proxy drift.
\end{remark}

		\section*{Data availability}
        The scripts used to generate the benchmark configurations analysed in this study are available in the public GitHub repository \url{https://github.com/linanci123/Paper-PINN}.

		\section*{Code availability}
        The training and post-processing code used in this study is available in the public GitHub repository \url{https://github.com/linanci123/Paper-PINN}.

% =========================
% Appendix: Proofs
% =========================
\appendix
\section{Proofs for Section~\ref{sec:methods:theory}}\label{app:proofs}

% -------------------------
\subsection{Proof of Theorem~\ref{thm:retain-accumulation}}\label{app:proof-accum}

\begin{proof}
We prove the three claims in turn.
The hit property~\eqref{eq:hit} follows immediately from Theorem~\ref{thm:resample-nonempty}.
In particular, consider a good round $i$ and a point $x\in C_i\cap B_i(\varepsilon_i)$. By the definition of
$B_i(\varepsilon_i)$ and $\widetilde U_i^\star$, we have
\[
\widetilde U_i(x)\ \ge\ \widetilde U_i^\star-\varepsilon_i\ \ge\ \tau_i,
\]
and thus $x\in\Lambda_i\subset \mathcal S_{i+1}$. Next, we verify no-early-exit. For the same good round $i$ and point $x\in C_i\cap B_i(\varepsilon_i)$, let $j\ge i$ be any round before ``success'', i.e.,
$\overline U_j<\widetilde U_i^\star-\varepsilon$. We claim that $x$ cannot be removed at round $j$ provided either the proxy is fixed across rounds
($\widetilde U_j\equiv \widetilde U_i$) and $\tau_j\le \widetilde U_i^\star-\varepsilon_i$, or the inter-round drift is small in the sense that
\[
\inf_{y\in B_i(\varepsilon_i)} \widetilde U_j(y)\ \ge\ \widetilde U_i^\star-\varepsilon_i
\quad\text{and}\quad
\tau_j\ \le\ \widetilde U_i^\star-\varepsilon_i .
\]
Indeed, in the first case we have \[\widetilde U_j(x)=\widetilde U_i(x)\ge \widetilde U_i^\star-\varepsilon_i\ge\tau_j, \]
while in the second case the drift condition gives $\widetilde U_j(x)\ge \widetilde U_i^\star-\varepsilon_i\ge\tau_j$ directly.
In either situation $x\in\Lambda_j$, and therefore $x\in\Lambda_j\subset \mathcal S_{j+1}$,
which is precisely the no-early-exit property.
Finally, we show long-term accumulation. Assume that the set of good rounds
$I_\varepsilon:=\{i:\ m_i\ge m_0(\varepsilon)\}$ is infinite.
Suppose, for contradiction, that
\[
\liminf_{n\to\infty} \overline U_n\ <\ \widetilde U^\star-\varepsilon .
\]
Then there exists an infinite subsequence $\{j_\ell\}$ such that
$\overline U_{j_\ell}<\widetilde U^\star-\varepsilon$ for all $\ell$.
We choose an increasing infinite subsequence $\{i_k\}\subset I_\varepsilon$ with
$i_1<j_1<i_2<j_2<\cdots$.
By the hit property, we can select $x_k\in C_{i_k}\cap B_{i_k}(\varepsilon_{i_k})$ for each $k$.
Since none of the rounds $j_\ell$ is successful, we can repeatedly apply no-early-exit to conclude that,
for every $r\ge1$,
\[
\{x_1,\dots,x_r\}\ \subseteq\ \mathcal S_{j_r+1}\cap \Lambda_{j_r}.
\]
Consequently we obtain $|\mathcal S_{j_r+1}\cap \Lambda_{j_r}|\ge r$, which contradicts the fixed budget
$|\mathcal S_{j_r+1}|=N$ once we take $r>N$.
Therefore, the contradiction shows that
\[
\liminf_{n\to\infty} \overline U_n\ \ge\ \widetilde U^\star-\varepsilon .
\]
\end{proof}

% -------------------------
\subsection{Proof of Theorem~\ref{thm:resample-nonempty}}\label{app:proof-resample}

\begin{proof}
Let $\widetilde B:=T_i^{-1}(B)\subset[0,1]^2$. We first relate $|\widetilde B|$ to $p_B$.
By change of variables, for any measurable $E\subset A_i$, we have
\[
|T_i^{-1}(E)|=\int_E J_{T_i}(x)\,dx,
\qquad\text{and}\qquad
\int_{A_i}J_{T_i}(x)\,dx=1 .
\]
Taking $E=B$ and writing $p_B=\mu(B)/\mu(A_i)$ yield
\[
\bigl||\widetilde B|-p_B\bigr|
= \Bigl|\int_{B}\Bigl(J_{T_i}(x)-\tfrac{1}{\mu(A_i)}\Bigr)\,dx\Bigr|
\le C_T\,p_B ,
\]
which implies
\[
|\widetilde B|\ \ge\ (1-C_T)\,p_B .
\]
This proves \eqref{eq:tildeB-lb}.
Next we estimate the empirical mass of $\widetilde B$ under the points $\{u_{i,j}\}_{j=1}^{m_i}$.
Since $B$ has finite perimeter and $T_i$ is bi-Lipschitz, $\widetilde B$ has finite perimeter as well.
For $\delta\in(0,1)$, we take $f_\delta\in C^1([0,1]^2)$ to be a BV-mollification of $\mathbf 1_{\widetilde B}$
satisfying
\[
0\le f_\delta\le \mathbf 1_{\widetilde B},\qquad
\int_{[0,1]^2}f_\delta = |\widetilde B|-\mathcal O(\delta\,\operatorname{Per}(\widetilde B)),\qquad
V_{HK}(f_\delta)\le C\,\frac{\operatorname{Per}(\widetilde B)}{\delta}.
\]
Applying the Koksma-Hlawka inequality gives
\[
\Bigl|\frac1{m_i}\sum_{j=1}^{m_i} f_\delta(u_{i,j})
-\int_{[0,1]^2}\! f_\delta\Bigr|
\le V_{HK}(f_\delta)\,D^{\ast}(\{u_{i,j}\})
\le C\,\frac{\operatorname{Per}(\widetilde B)}{\delta}\,D^{\ast}(\{u_{i,j}\}).
\]
Using $f_\delta\le \mathbf 1_{\widetilde B}$ together with
$\int f_\delta=|\widetilde B|-\mathcal O(\delta\,\operatorname{Per}(\widetilde B))$, we obtain
\[
\frac1{m_i}\sum_{j=1}^{m_i}\mathbf 1_{\widetilde B}(u_{i,j})
\ \ge\ |\widetilde B|
-C_2\,\operatorname{Per}(\widetilde B)\,\delta
-C_1\,\frac{\operatorname{Per}(\widetilde B)}{\delta}\,D^{*}(\{u_{i,j}\}).
\]
Choosing $\delta=\sqrt{D^{*}(\{u_{i,j}\})}$ and
absorbing $\operatorname{Per}(\widetilde B)$ into $C_1$ yield
\[
\frac1{m_i}\sum_{j=1}^{m_i}\mathbf 1_{\widetilde B}(u_{i,j})
\ \ge\ |\widetilde B| - C_1\sqrt{D^{*}(\{u_{i,j}\})}.
\]
Multiplying by $m_i$ gives \eqref{eq:nonempty-45}.
Finally, we derive the ``at least one hit'' criterion. From \eqref{eq:tildeB-lb} we have
$|\widetilde B|\ge (1-C_T)p_B$. If
\[
D^{*}(\{u_{i,j}\})\le \frac{(1-C_T)^2p_B^2}{4C_1^{\,2}}
\quad\text{and}\quad
m_i\ge \frac{2}{(1-C_T)p_B},
\]
then \eqref{eq:nonempty-45} implies $\#(C_i\cap B)\ge1$.
The stated Hammersley-node condition follows by inserting the discrepancy estimate
$D^{*}(m)\le C_{\mathrm{Ham}}(\log m)^2/m$.
\end{proof}

% -------------------------
\subsection{Proof of Corollary~\ref{cor:shell-injection}}\label{app:proof-shell}

\begin{proof}
We apply Theorem~\ref{thm:resample-nonempty} with $B=S_{i+1}=A_{i+1}\setminus A_i$ and the next-round
transport map $T_{i+1}$. We then choose $m_{i+1}$ sufficiently large so that the sufficient conditions
in Theorem~\ref{thm:resample-nonempty} are satisfied. The conclusion of Theorem~\ref{thm:resample-nonempty} yields
\[
\#(C_{i+1}\cap S_{i+1})\ge 1,
\]
which proves the statement.
\end{proof}

% -------------------------
\subsection{Proof of Proposition~\ref{prop:release-budget}}\label{app:proof-release}

\begin{proof}
Recall that $R_i=\mathcal S_i\cap\Lambda_i$ and that $C_i$ is chosen to have cardinality
$|C_i|=N-|R_i|$. Consequently, we have
\(
\mathcal S_{i+1}=R_i\cup C_i
\)
and the union is disjoint by construction. Therefore, we conclude
\[
|\mathcal S_{i+1}|=|R_i|+|C_i|=|R_i|+(N-|R_i|)=N,
\]
as desired.
\end{proof}

% -------------------------
\subsection{Proof of Theorem~\ref{thm:near-to-far}}\label{app:proof-dynamic}

\begin{proof}
We begin by proving \eqref{eq:decomp}. By construction, the next-round set is the disjoint union
\[
\mathcal S_{i+1}=(\mathcal S_i\cap\Lambda_i)\cup C_{i+1}.
\]
Therefore, for any $B\subset A_{i+1}$, we have
\[
\#(\mathcal S_{i+1}\cap B)
= \#(\mathcal S_i\cap\Lambda_i\cap B) + \#(C_{i+1}\cap B).
\]
Dividing by $N$ gives \eqref{eq:decomp}.
Next, we derive the lower bound \eqref{eq:strongLB}. Applying Theorem~\ref{thm:resample-nonempty} on $A_{i+1}$ and, for a given $B\subset A_{i+1}$, defining
\(
\widetilde B:=T_{i+1}^{-1}(B)\subset[0,1]^2,
\)
the Theorem~\ref{thm:resample-nonempty} yields
\[
\#(C_{i+1}\cap B)
\ \ge\ m_{i+1}\Bigl(|\widetilde B|-C_1\sqrt{D^{*}(\{u_{i+1,j}\})}\Bigr)_+ .
\]
On the other hand, the equal-area deviation estimate in the proof of Theorem~\ref{thm:resample-nonempty} gives
\[
|\widetilde B|\ge (1-C_T)p_B.
\]
Substituting this into the previous inequality, dividing by $N$, and combining with \eqref{eq:decomp} yields \eqref{eq:strongLB}.
Finally, \eqref{eq:conciseLB} follows by absorbing fixed multiplicative factors into a new constant $\widetilde C_1$. For the shell case, take $B=S_{i+1}=A_{i+1}\setminus A_i$. Since
\[
\mathcal S_i\cap\Lambda_i\cap S_{i+1}=\varnothing,
\]
\eqref{eq:shell-injection} follows directly from \eqref{eq:decomp}.
\end{proof}

%% please use the heading ``Supplementary''.   %%
%%=============================================%%




\bibliography{sn-bibliography}% common bib file + added references
% Supplementary Information
% Place after main text (and before/after references depending on template)
% ============================================================
\section*{Acknowledgements}
		Wenju Zhao was supported by the National Key R\&D Program of China (No.~2023YFA1008903-3) and the Natural Science Foundation of Shandong Province (No.~ZR2023ZD38). Zhiwen Zhang was supported by the National Natural Science Foundation of China (Project No.~92470103), Hong Kong RGC Grants 17304324 and 17300325, the Seed Funding Programme for Basic Research (HKU), and the Hong Kong RGC Research Fellow Scheme 2025.

\section*{Author contributions}
		A.L., Z.Z., and W.J.Z. conceived and designed the study. A.L. carried out the implementation and data generation. All authors contributed to data analysis and interpretation, drafted the manuscript, and approved the final version.
\section*{Competing interests}
		The authors declare no competing interests.

\section*{Additional information}
		\textbf{Correspondence and requests for materials} should be addressed to 
Zhiwen Zhang and Wenju Zhao.

\clearpage

\section*{Supplementary Information}

% --- Numbering for supplementary items ---
\setcounter{figure}{0}
\setcounter{table}{0}
\renewcommand{\thefigure}{S\arabic{figure}}
\renewcommand{\thetable}{S\arabic{table}}

% ============================================================
\begin{table}[t]
	\centering
	\caption{Default hyperparameters used in all experiments. Network architecture, optimizer settings, collocation budget, quadrature, gating schedule, R3 policy, and UQ-proxy parameters.}
	\label{tab:ed-defaults}
	\begin{tabular}{@{}ll@{}}
		\toprule
		\textbf{Component} & \textbf{Default setting} \\
		\midrule
		Network & 3 hidden layers, width 128, \textsc{CELU} activation \\
		Optimizer & Adam ($10^{-3}$, $\beta_1{=}0.9$, $\beta_2{=}0.999$); LR$\times 0.9$/10k iters \\
		Collocation budget & $N=10^4$ points, updated by R3 \\
		Quadrature & QMC (Hammersley) for bulk terms \\
		Gating & $g_{\gamma_i}(x)=\sigma(\alpha(\gamma_i-\tilde d(x)))$; $\alpha{=}5.0$, $\gamma_0{=}-0.5$, $\Delta_{\max}{=}0.05$ \\
		R3 policy & Retain fraction $\rho = 0.5$ \\
		UQ proxy & Var.\ of $\|\nabla y_\theta\|_F$ under Gaussian probes ($m_{\mathrm{uq}}{=}16$, $\rho_{\mathrm{uq}}{=}0.01$) \\
		\bottomrule
	\end{tabular}
\end{table}

% ============================================================
\begin{figure}[t]
	\centering
	\begin{subfigure}{0.24\linewidth}
		\centering
		\includegraphics[width=\linewidth]{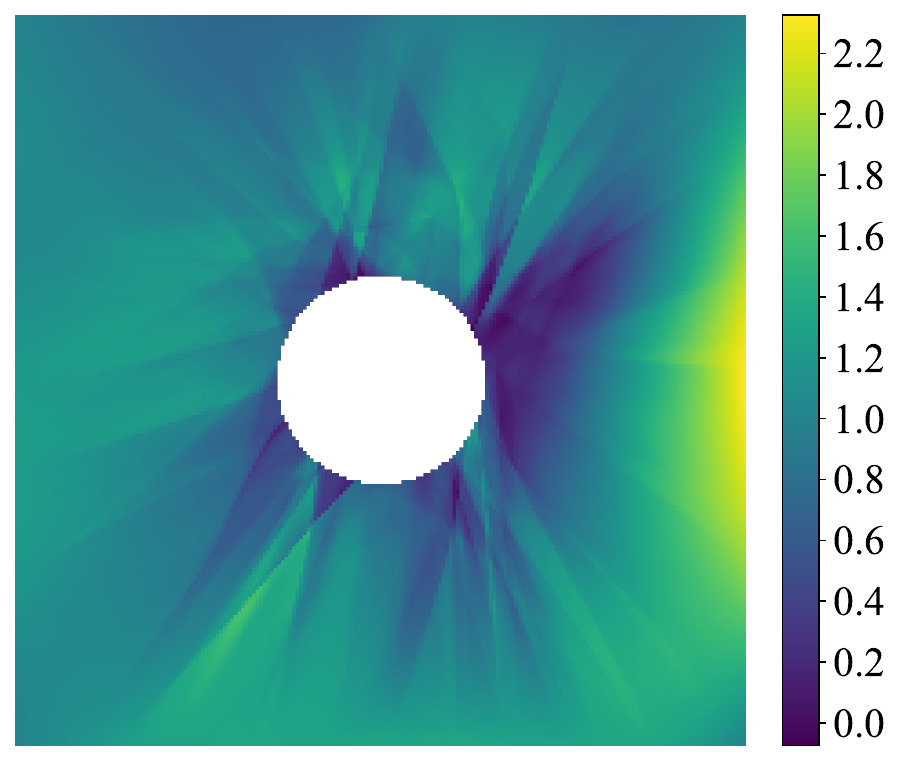}
		\caption{$P=100$}
	\end{subfigure}\hfill
	\begin{subfigure}{0.24\linewidth}
		\centering
		\includegraphics[width=\linewidth]{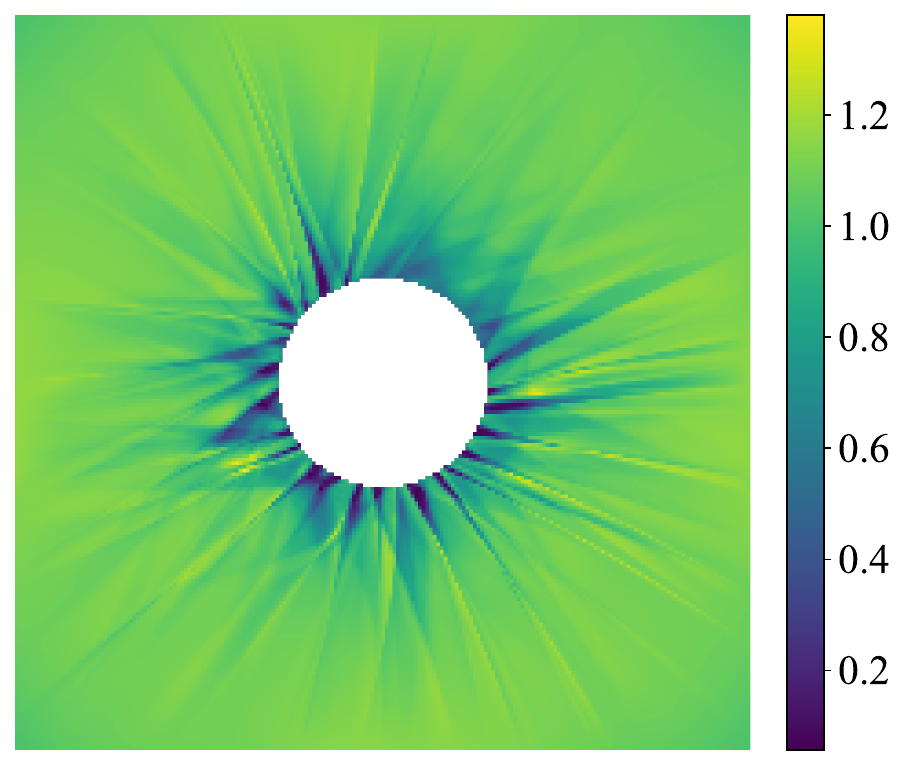}
		\caption{$P=400$}
	\end{subfigure}\hfill
	\begin{subfigure}{0.24\linewidth}
		\centering
		\includegraphics[width=\linewidth]{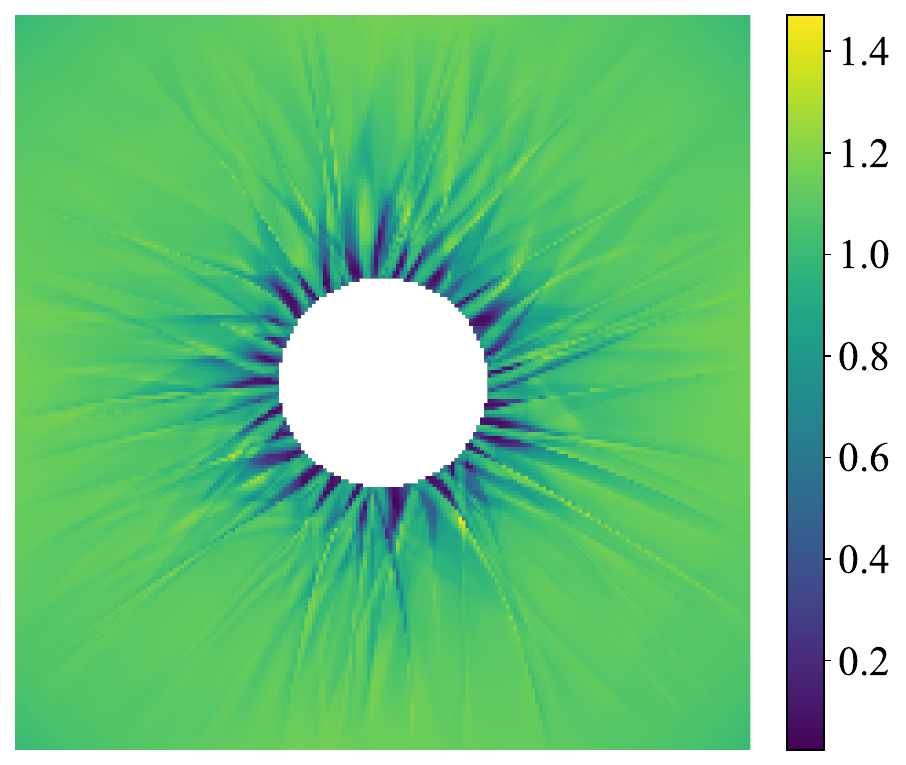}
		\caption{$P=1600$}
	\end{subfigure}\hfill
	\begin{subfigure}{0.24\linewidth}
		\centering
		\includegraphics[width=\linewidth]{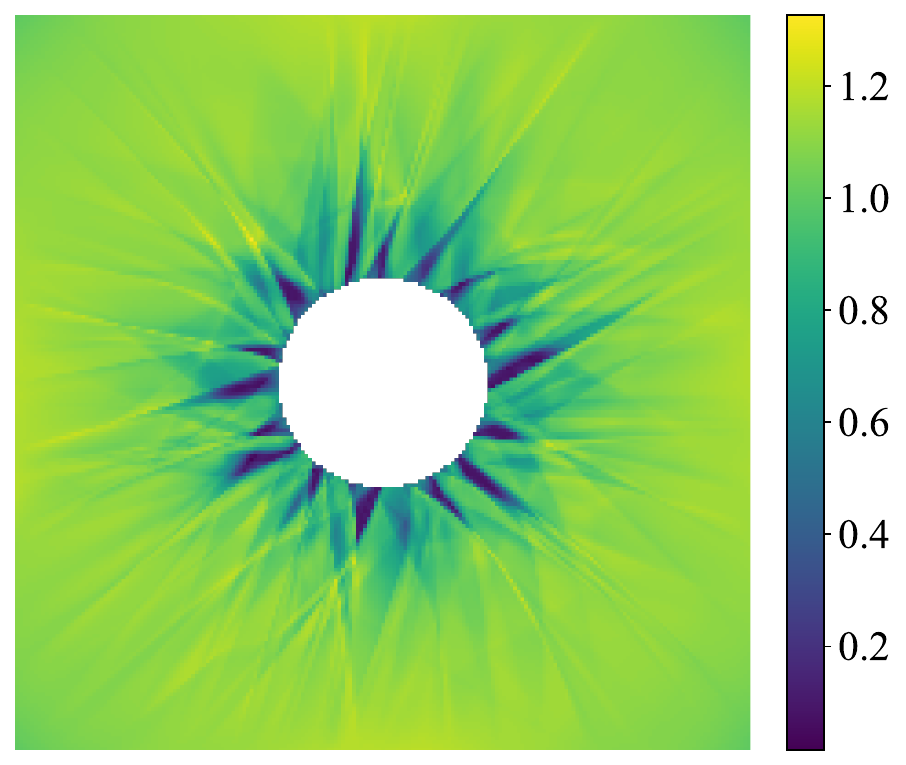}
		\caption{$P=3200$}
	\end{subfigure}
\caption{Effect of the resampling period $P$ in the single-cell non-regularized regime ($\varepsilon_0=0$). Each panel shows the Jacobian determinant $J=\det F$ (with $F=\nabla y_\theta$).}
	\label{fig:ed-single-eps0-period}
\end{figure}
% ============================================================
\begin{figure}[t]
	\centering
	% --- Subfigure 1: hl=3 ---
	\begin{subfigure}[b]{0.24\linewidth}
		\centering
		\includegraphics[width=\linewidth]{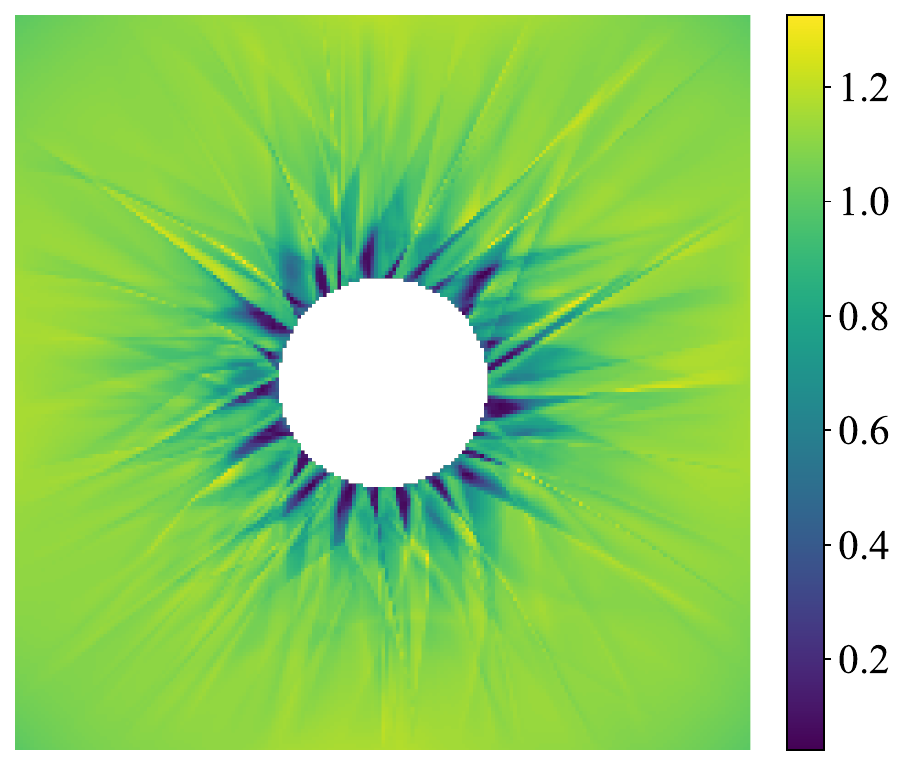}
		\caption{$h_l=3$}
		\label{fig:hl3}
	\end{subfigure}
	\hfill
	% --- Subfigure 2: hl=5 ---
	\begin{subfigure}[b]{0.24\linewidth}
		\centering
		\includegraphics[width=\linewidth]{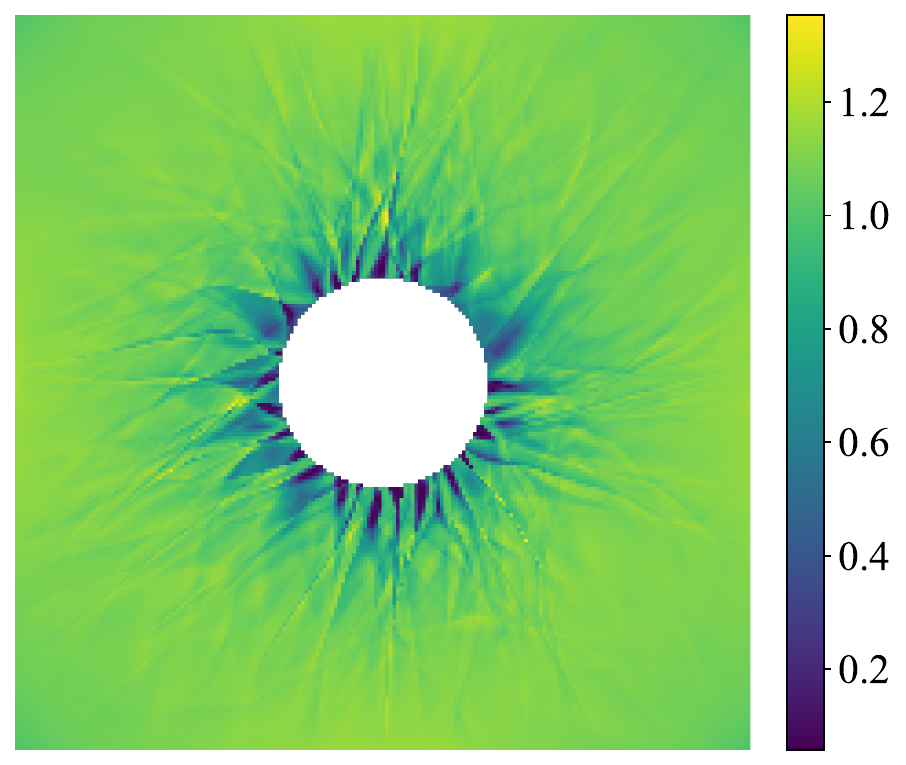}
		\caption{$h_l=5$}
		\label{fig:hl5}
	\end{subfigure}
	\hfill
	% --- Subfigure 3: hl=7 ---
	\begin{subfigure}[b]{0.24\linewidth}
		\centering
		\includegraphics[width=\linewidth]{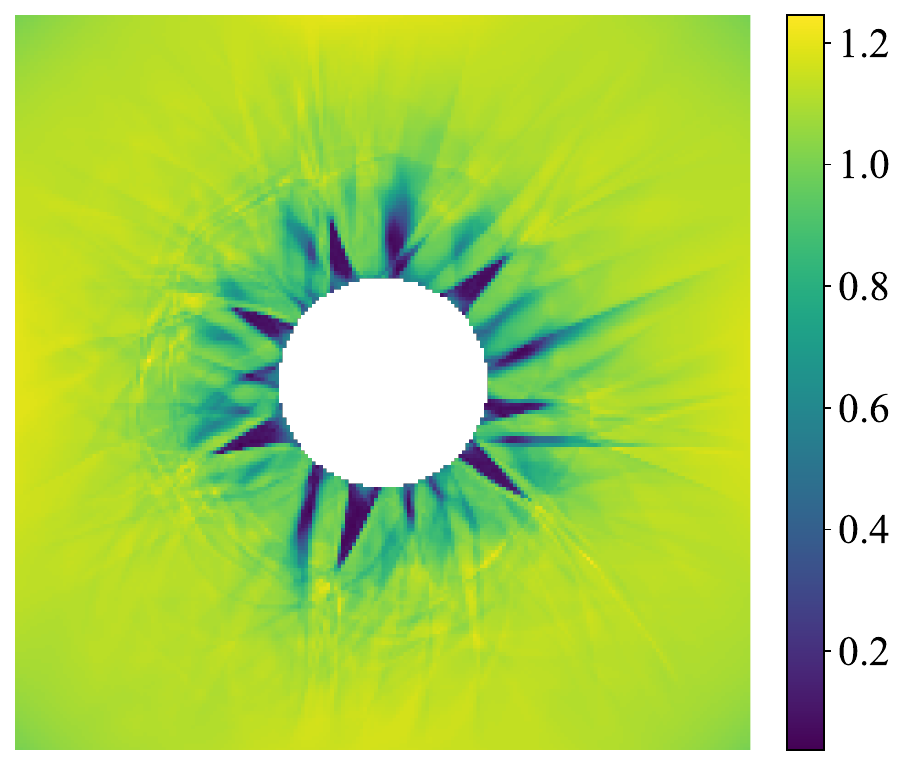}
		\caption{$h_l=7$}
		\label{fig:hl7}
	\end{subfigure}
	\hfill
	% --- Subfigure 4: hl=9 ---
	\begin{subfigure}[b]{0.24\linewidth}
		\centering
		\includegraphics[width=\linewidth]{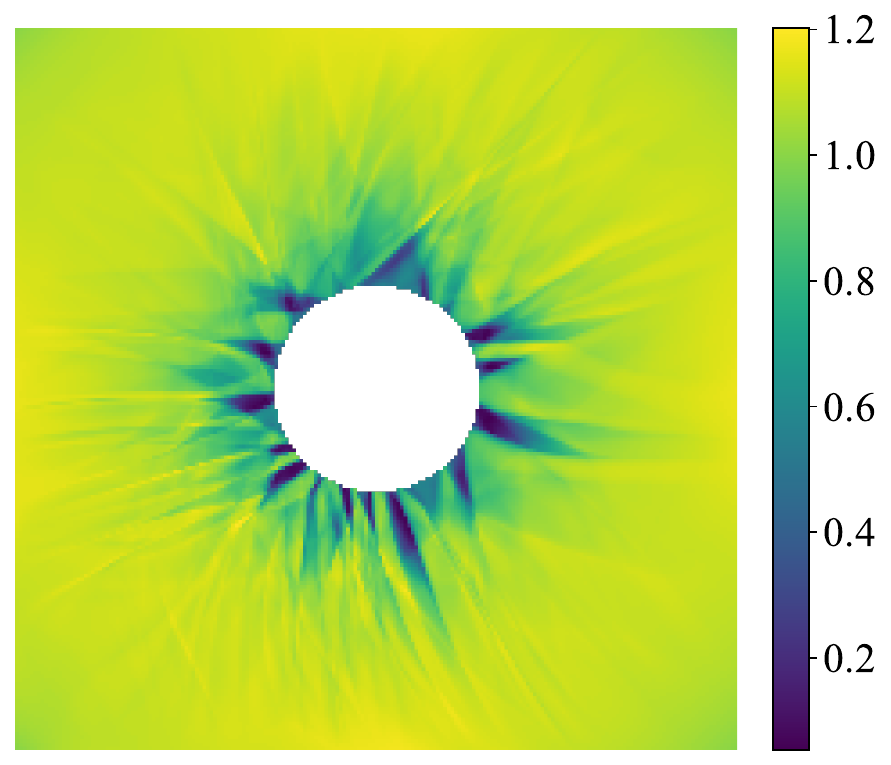}
		\caption{$h_l=9$}
		\label{fig:hl9}
	\end{subfigure}
	
\caption{Effect of network depth $h_l$ in the single-cell non-regularized regime ($\varepsilon_0=0$). Each panel shows the Jacobian determinant $J=\det F$ (with $F=\nabla y_\theta$) for different hidden-layer depths.}
	\label{fig:ed-single-eps0-depth}
\end{figure}

% ============================================================
\begin{figure}[t]
	\centering
	% --- Subfigure 1: width=64 ---
	\begin{subfigure}[b]{0.32\linewidth}
		\centering
		\includegraphics[width=\linewidth]{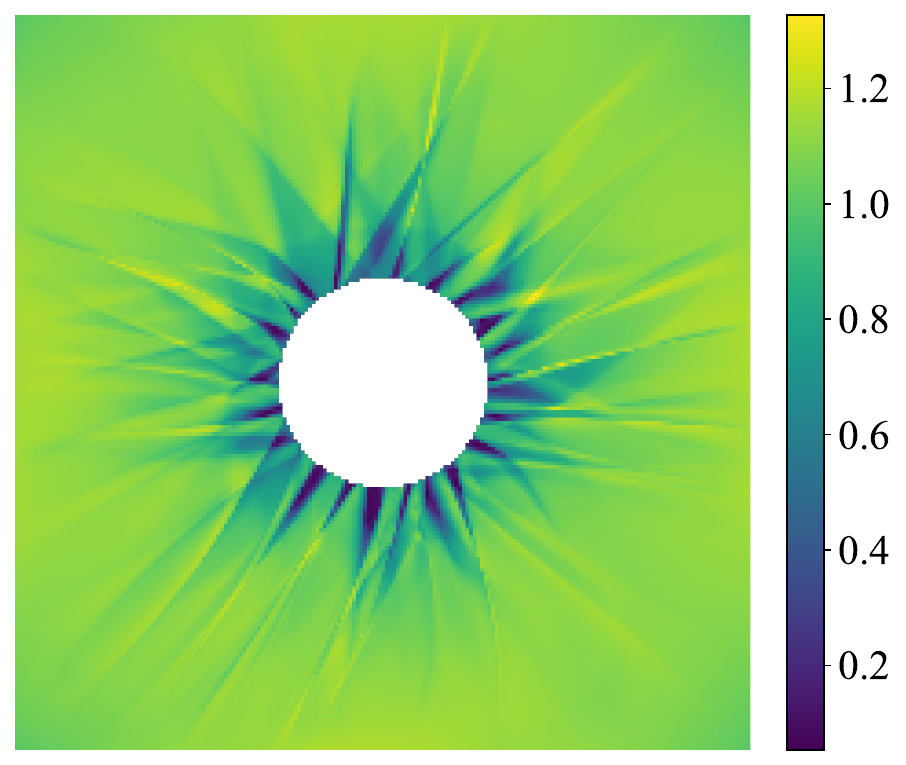}
		\caption{$width=64$}
		\label{fig:width64}
	\end{subfigure}
	\hfill
	% --- Subfigure 2: width=128 ---
	\begin{subfigure}[b]{0.32\linewidth}
		\centering
		\includegraphics[width=\linewidth]{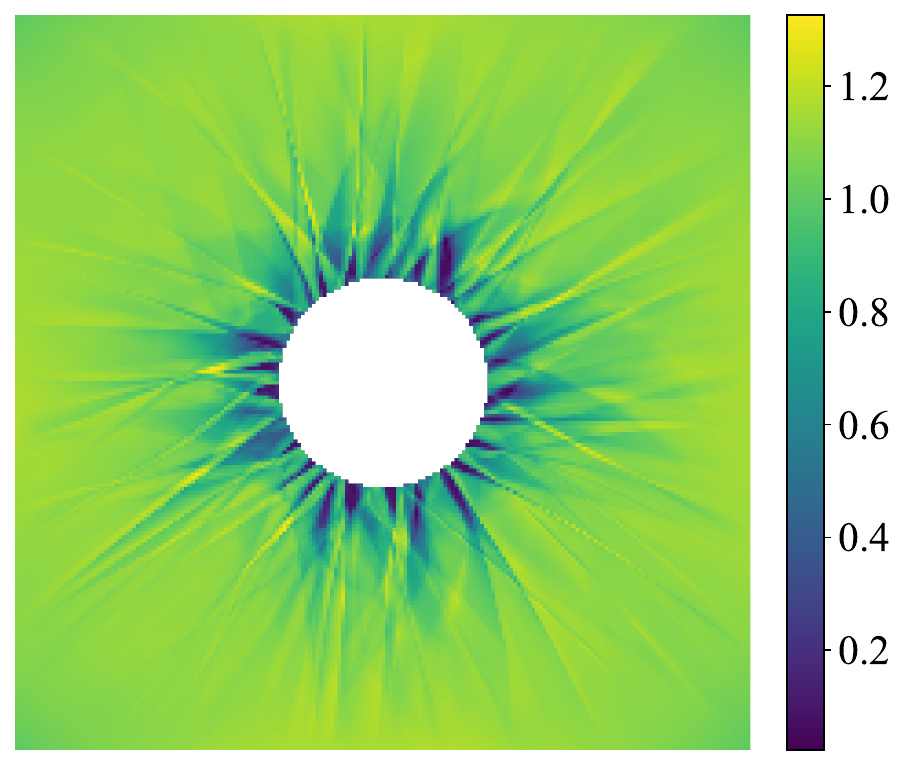}
		\caption{$width=128$}
		\label{fig:width128}
	\end{subfigure}
	\hfill
	% --- Subfigure 3: width=256 ---
	\begin{subfigure}[b]{0.32\linewidth}
		\centering
		\includegraphics[width=\linewidth]{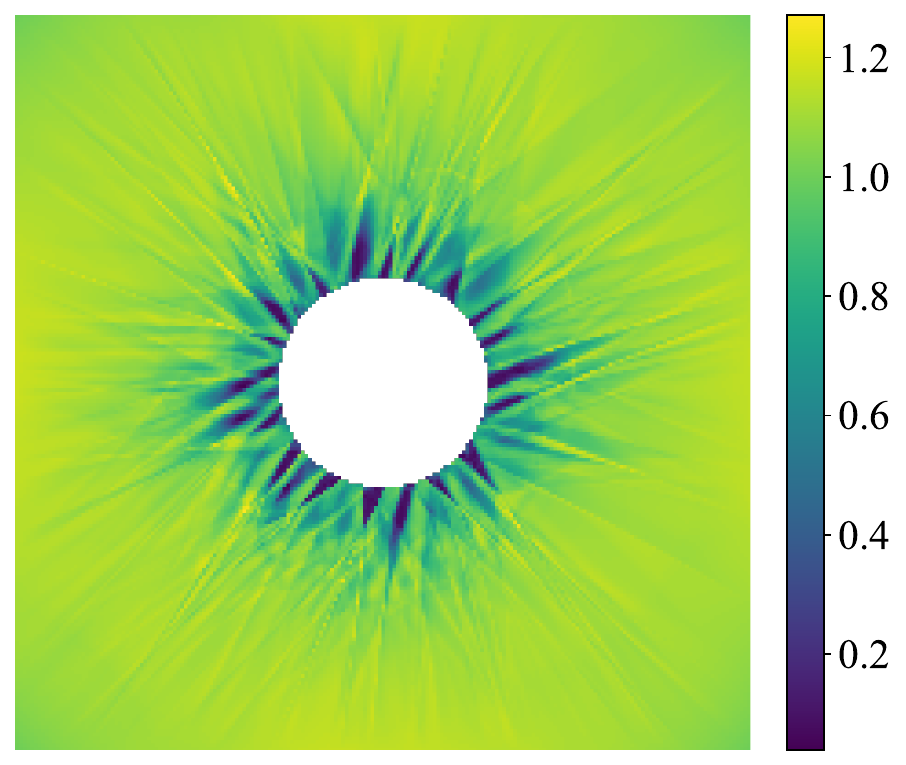}
		\caption{$width=256$}
		\label{fig:width256}
	\end{subfigure}
	
\caption{Effect of network width in the single-cell non-regularized regime ($\varepsilon_0=0$). Each panel shows the Jacobian determinant $J=\det F$ (with $F=\nabla y_\theta$) for different hidden-layer widths.}
	\label{fig:ed-single-eps0-width}
\end{figure}

% ============================================================
\begin{figure}[t]
	\centering
	% --- Subfigure 1: uq_m=8 ---
	\begin{subfigure}[b]{0.24\linewidth}
		\centering
		\includegraphics[width=\linewidth]{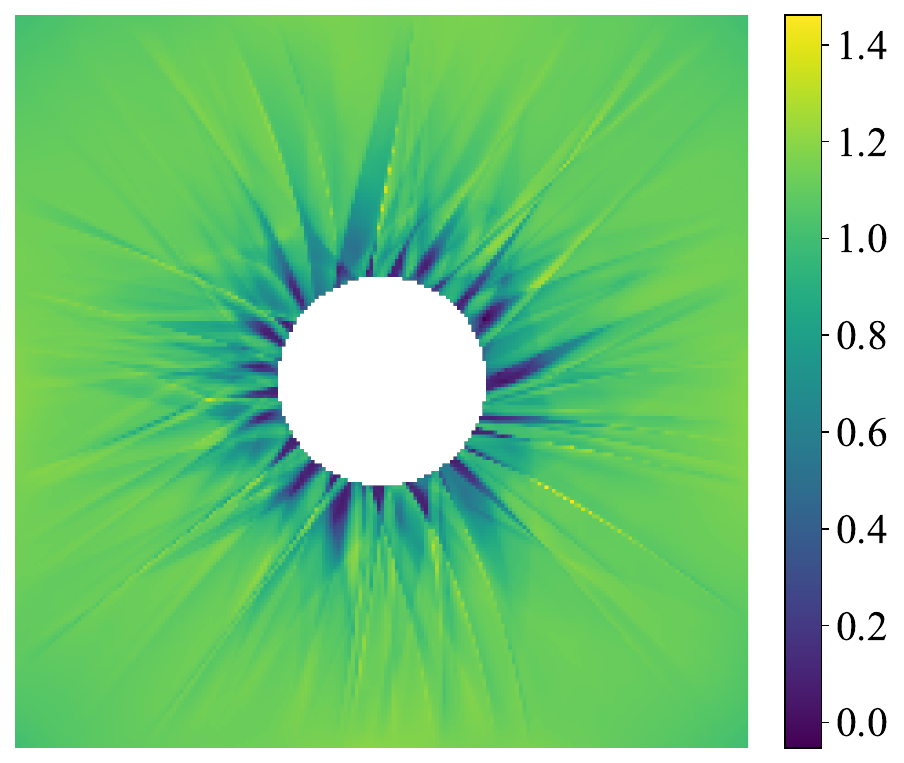}
		\caption{$uq_{m}=8$}
		\label{fig:uqm8}
	\end{subfigure}
	\hfill
	% --- Subfigure 2: uq_m=16 ---
	\begin{subfigure}[b]{0.24\linewidth}
		\centering
		\includegraphics[width=\linewidth]{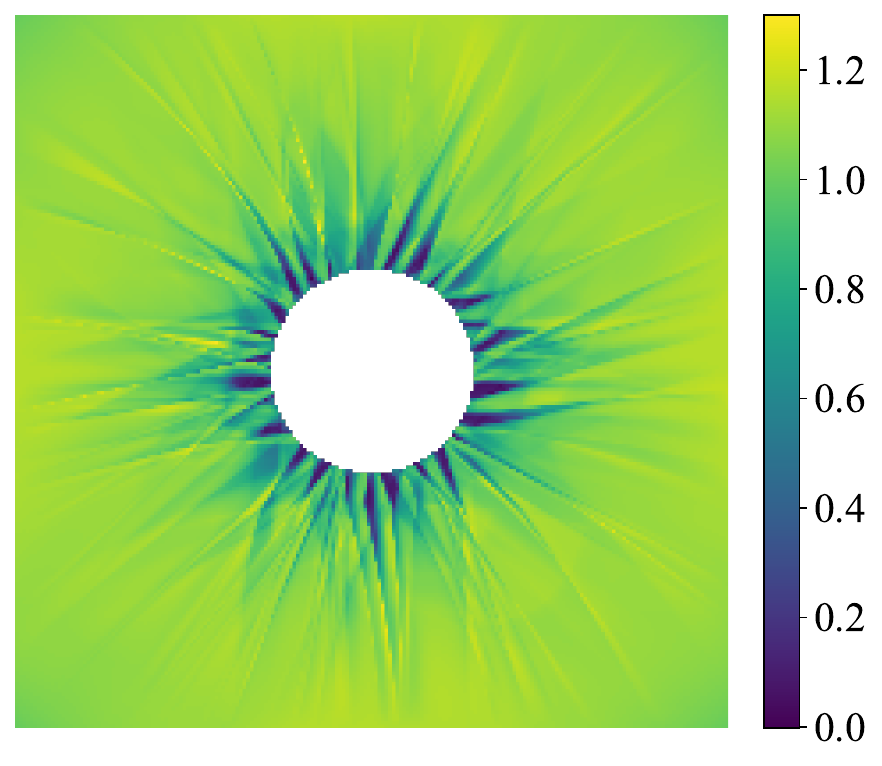}
		\caption{$uq_{m}=16$}
		\label{fig:uqm16}
	\end{subfigure}
	\hfill
	% --- Subfigure 3: uq_m=32 ---
	\begin{subfigure}[b]{0.24\linewidth}
		\centering
		\includegraphics[width=\linewidth]{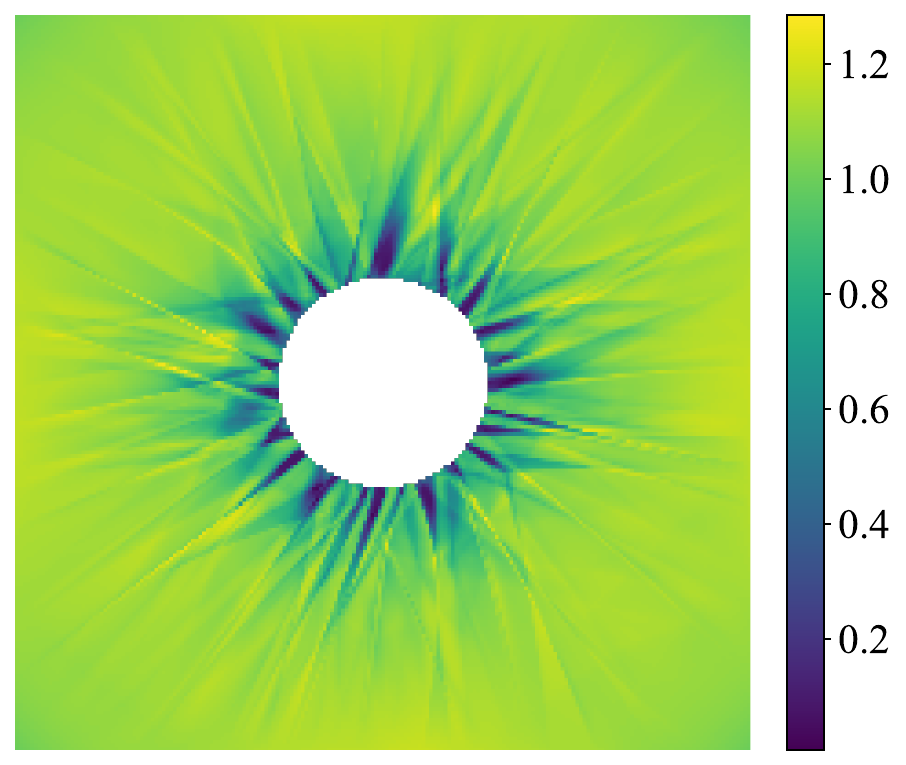}
		\caption{$uq_{m}=32$}
		\label{fig:uqm32}
	\end{subfigure}
	\hfill
	% --- Subfigure 4: uq_m=64 ---
	\begin{subfigure}[b]{0.24\linewidth}
		\centering
		\includegraphics[width=\linewidth]{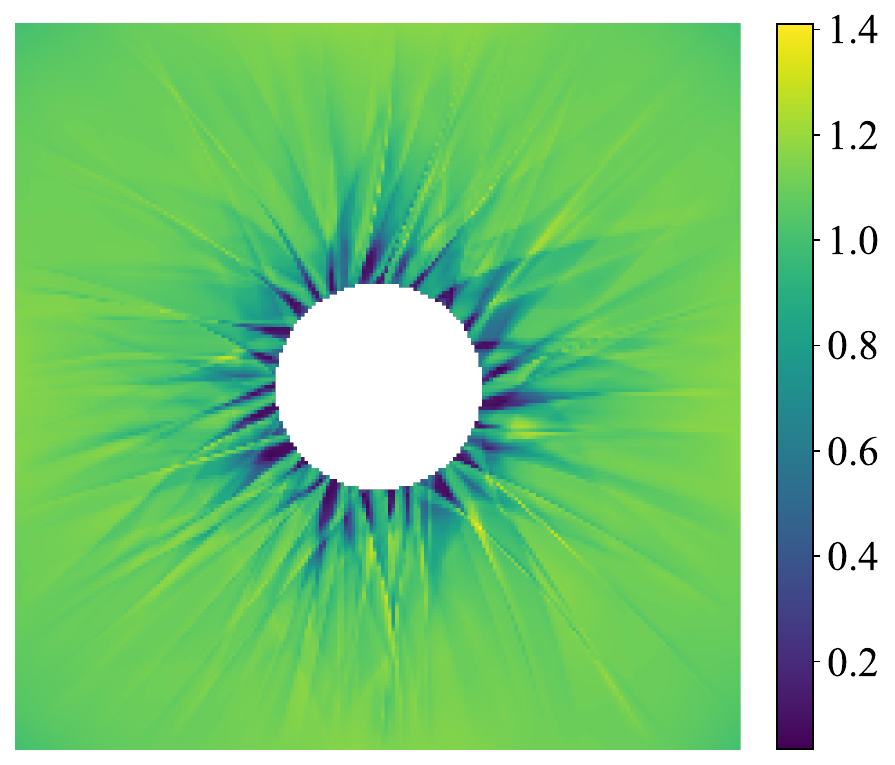}
		\caption{$uq_{m}=64$}
		\label{fig:uqm64}
	\end{subfigure}
	
	\caption{Effect of the UQ probe count $m_{\mathrm{uq}}$ in the single-cell non-regularized regime ($\varepsilon_0=0$). Each panel shows the Jacobian determinant $J=\det F$ (with $F=\nabla y_\theta$).}
	\label{fig:ed-single-eps0-uqm}
\end{figure}

% ============================================================
\begin{figure}[t]
	\centering
	% --- Subfigure 1: rho=0.01 ---
	\begin{subfigure}[b]{0.24\linewidth}
		\centering
		\includegraphics[width=\linewidth]{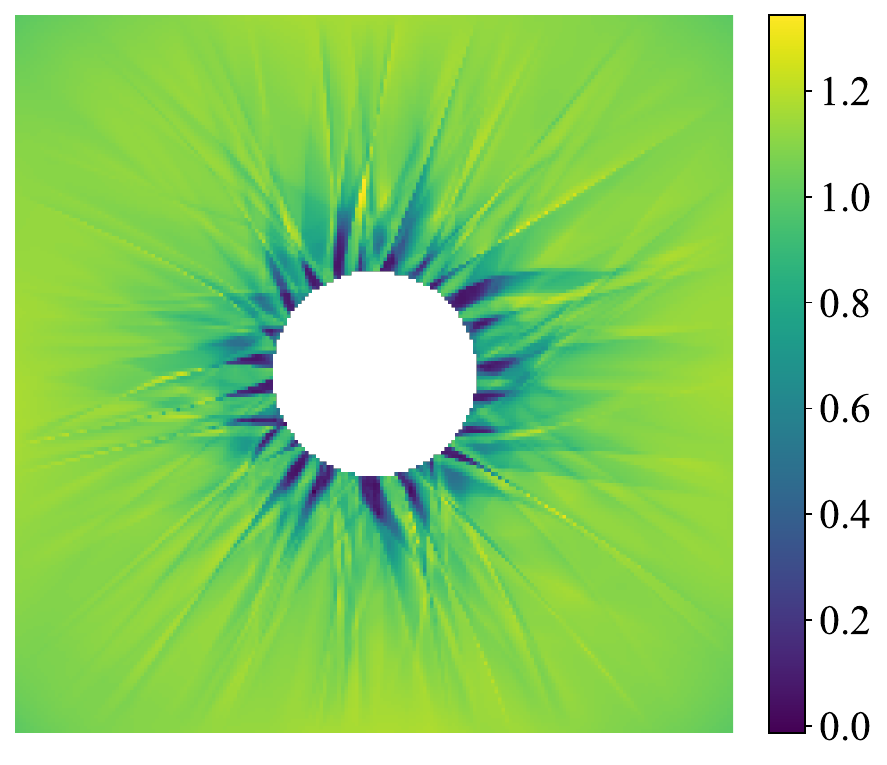}
		\caption{$\rho_{uq}=0.01$}
		\label{fig:rho0.01}
	\end{subfigure}
	\hfill
	% --- Subfigure 2: rho=0.02 ---
	\begin{subfigure}[b]{0.24\linewidth}
		\centering
		\includegraphics[width=\linewidth]{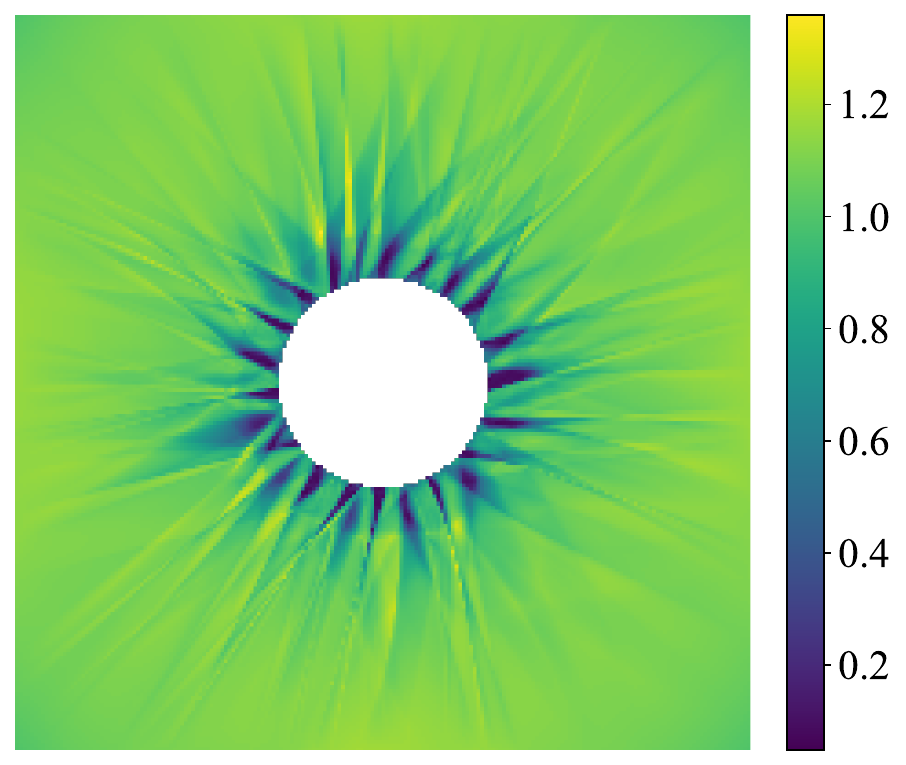}
		\caption{$\rho_{uq}=0.02$}
		\label{fig:rho0.02}
	\end{subfigure}
	\hfill
	% --- Subfigure 3: rho=0.04 ---
	\begin{subfigure}[b]{0.24\linewidth}
		\centering
		\includegraphics[width=\linewidth]{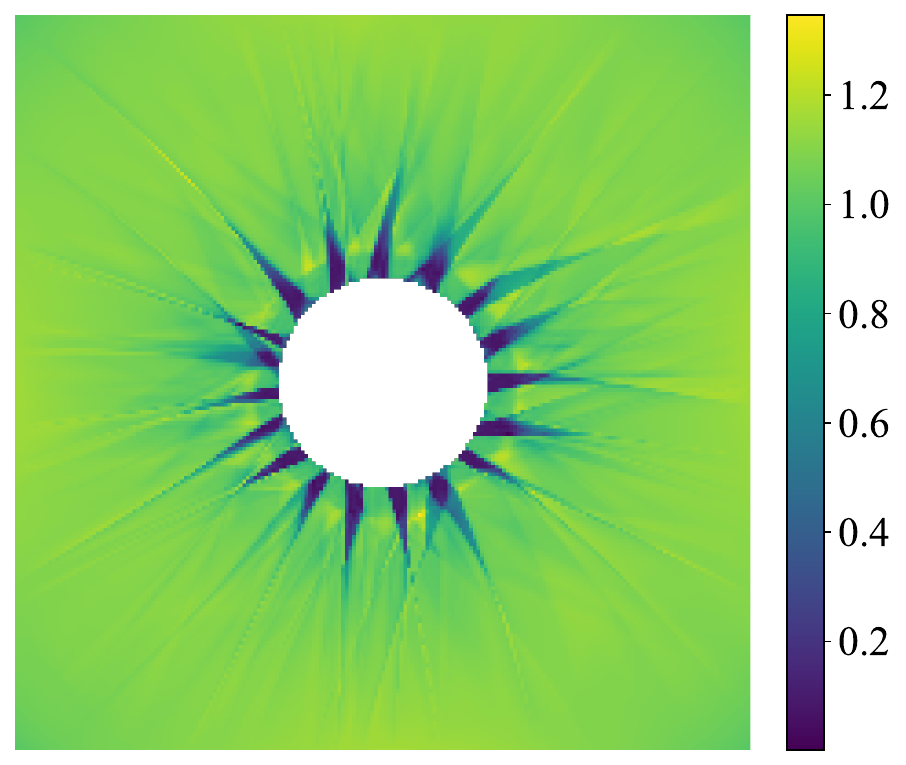}
		\caption{$\rho_{uq}=0.04$}
		\label{fig:rho0.04}
	\end{subfigure}
	\hfill
	% --- Subfigure 4: rho=0.08 ---
	\begin{subfigure}[b]{0.24\linewidth}
		\centering
		\includegraphics[width=\linewidth]{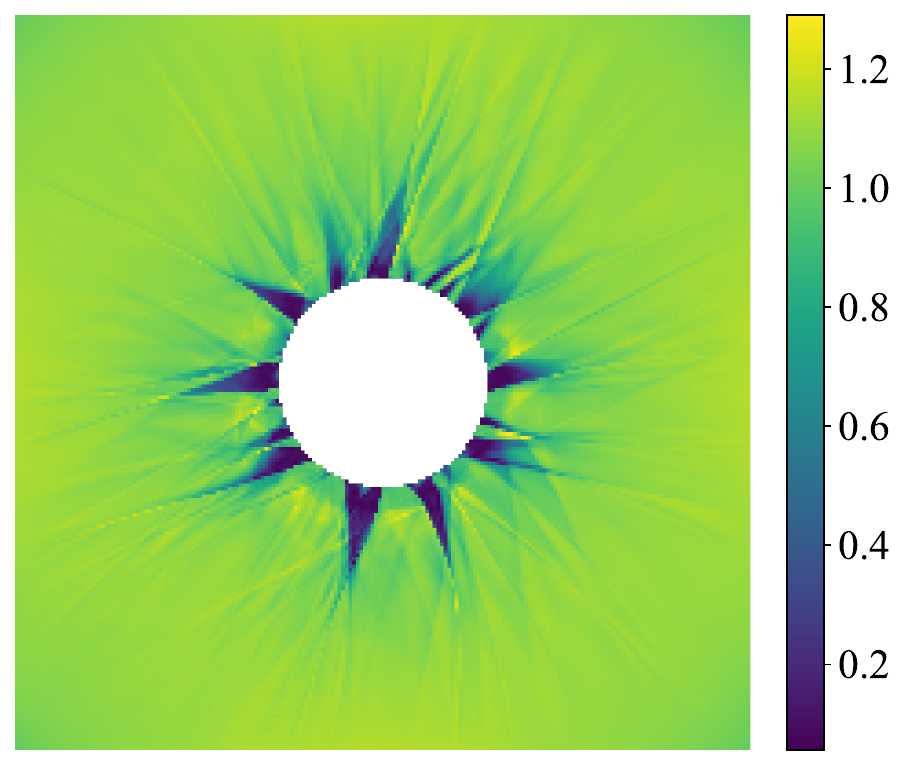}
		\caption{$\rho_{uq}=0.08$}
		\label{fig:rho0.08}
	\end{subfigure}
	
\caption{Effect of the UQ probe variance $\rho_{\mathrm{uq}}$ in the single-cell non-regularized regime ($\varepsilon_0=0$). Each panel shows the Jacobian determinant $J=\det F$ (with $F=\nabla y_\theta$).}
	\label{fig:ed-single-eps0-uqrho}
\end{figure}

% ============================================================
\begin{figure}[t]
	\centering
	% --- Subfigure 1: P=100 ---
	\begin{subfigure}[b]{0.24\linewidth}
		\centering
		\includegraphics[width=\linewidth]{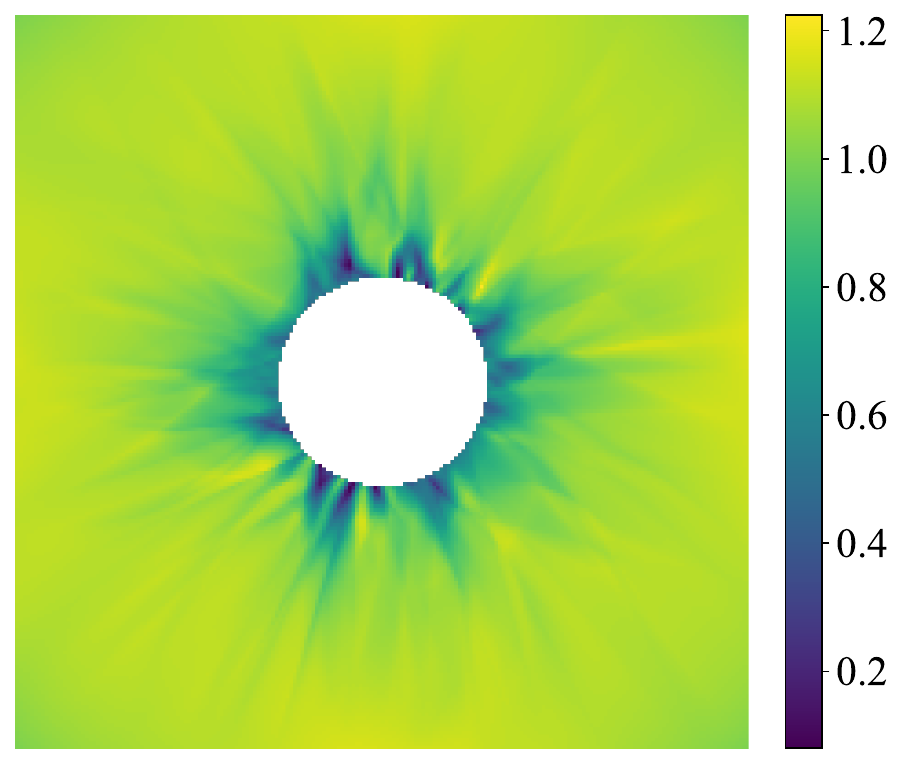}
		\caption{$P=100$}
		\label{fig:P100}
	\end{subfigure}
	\hfill
	% --- Subfigure 2: P=400 ---
	\begin{subfigure}[b]{0.24\linewidth}
		\centering
		\includegraphics[width=\linewidth]{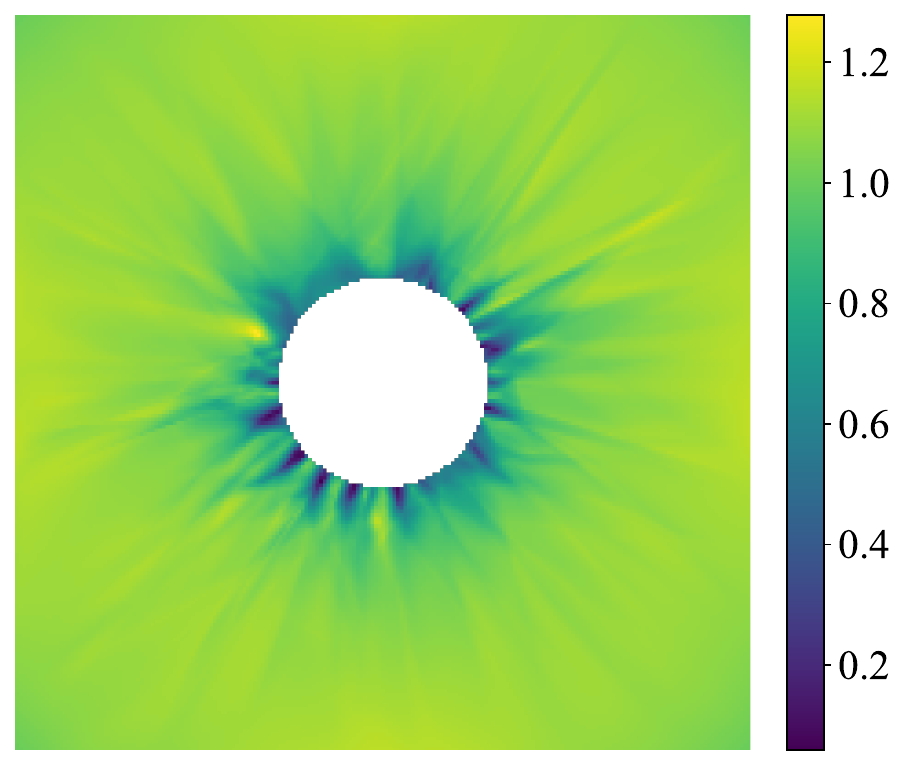}
		\caption{$P=400$}
		\label{fig:P400}
	\end{subfigure}
	\hfill
	% --- Subfigure 3: P=1600 ---
	\begin{subfigure}[b]{0.24\linewidth}
		\centering
		\includegraphics[width=\linewidth]{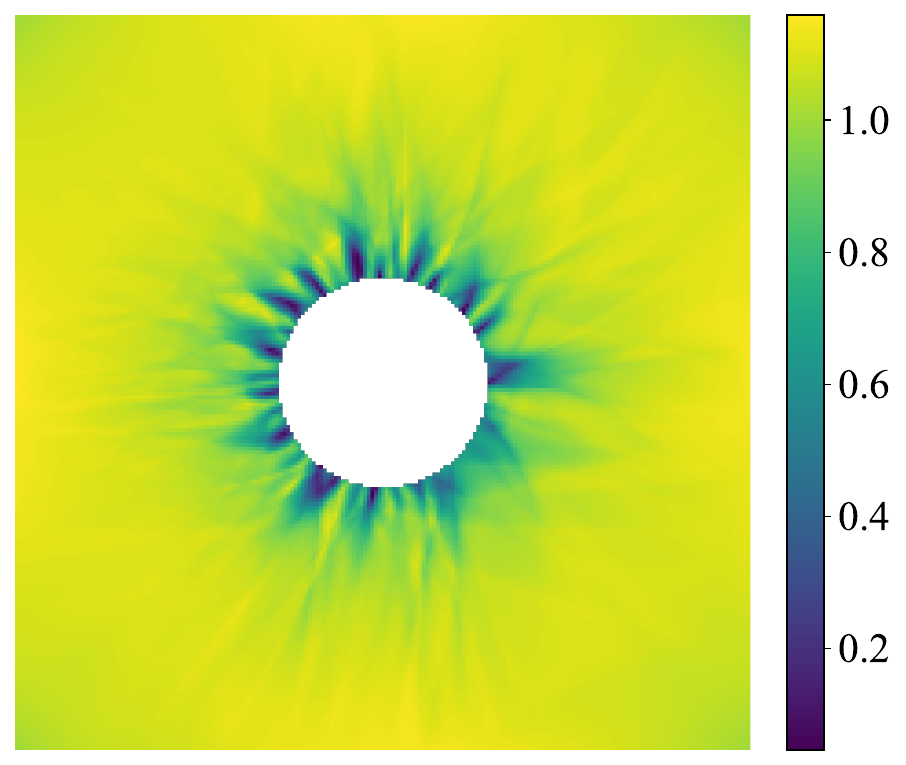}
		\caption{$P=1600$}
		\label{fig:P1600}
	\end{subfigure}
	\hfill
	% --- Subfigure 4: P=3200 ---
	\begin{subfigure}[b]{0.24\linewidth}
		\centering
		\includegraphics[width=\linewidth]{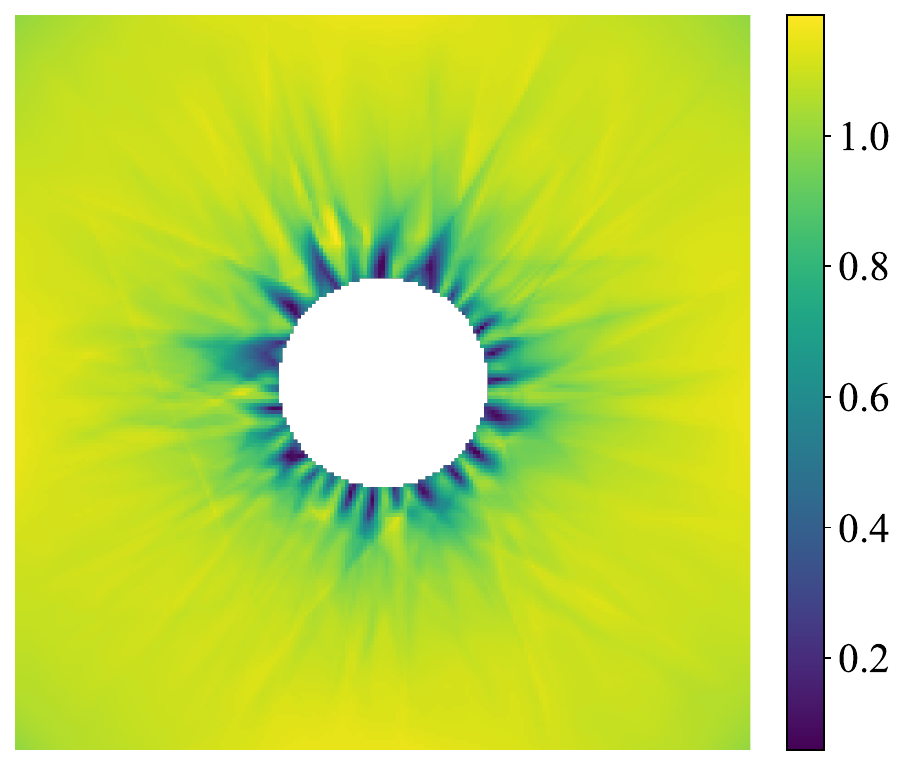}
		\caption{$P=3200$}
		\label{fig:P3200}
	\end{subfigure}
	
\caption{Effect of the resampling period $P$ in the single-cell weakly regularized regime ($\varepsilon_0=0.01\,r_c$). Each panel shows the Jacobian determinant $J=\det F$ (with $F=\nabla y_\theta$).}
	\label{fig:ed-single-eps001rc-period}
\end{figure}

% ============================================================
\begin{figure}[t]
	\centering
	% --- Subfigure 1: hl=3 ---
	\begin{subfigure}[b]{0.24\linewidth}
		\centering
		\includegraphics[width=\linewidth]{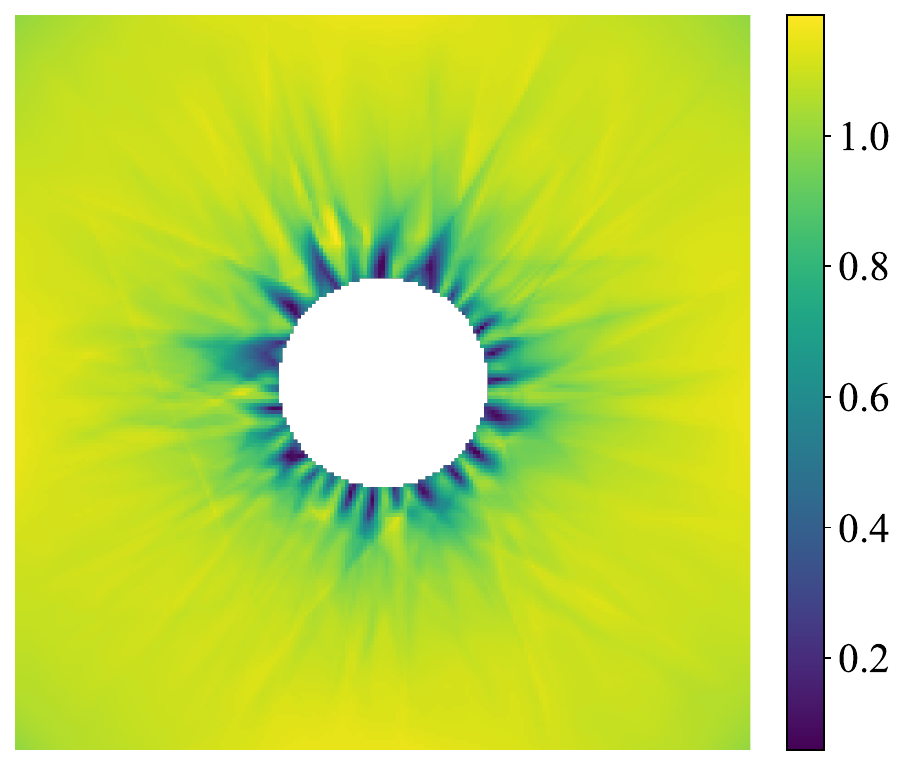}
		\caption{$h_l=3$}
		\label{fig:hl3_eps001}
	\end{subfigure}
	\hfill
	% --- Subfigure 2: hl=5 ---
	\begin{subfigure}[b]{0.24\linewidth}
		\centering
		\includegraphics[width=\linewidth]{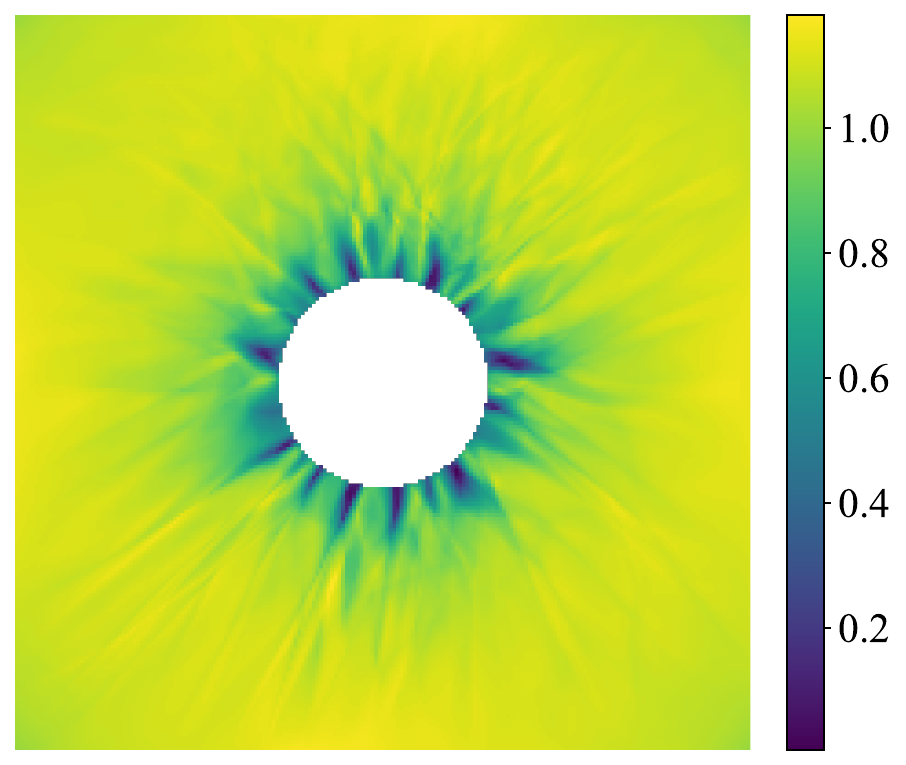}
		\caption{$h_l=5$}
		\label{fig:hl5_eps001}
	\end{subfigure}
	\hfill
	% --- Subfigure 3: hl=7 ---
	\begin{subfigure}[b]{0.24\linewidth}
		\centering
		\includegraphics[width=\linewidth]{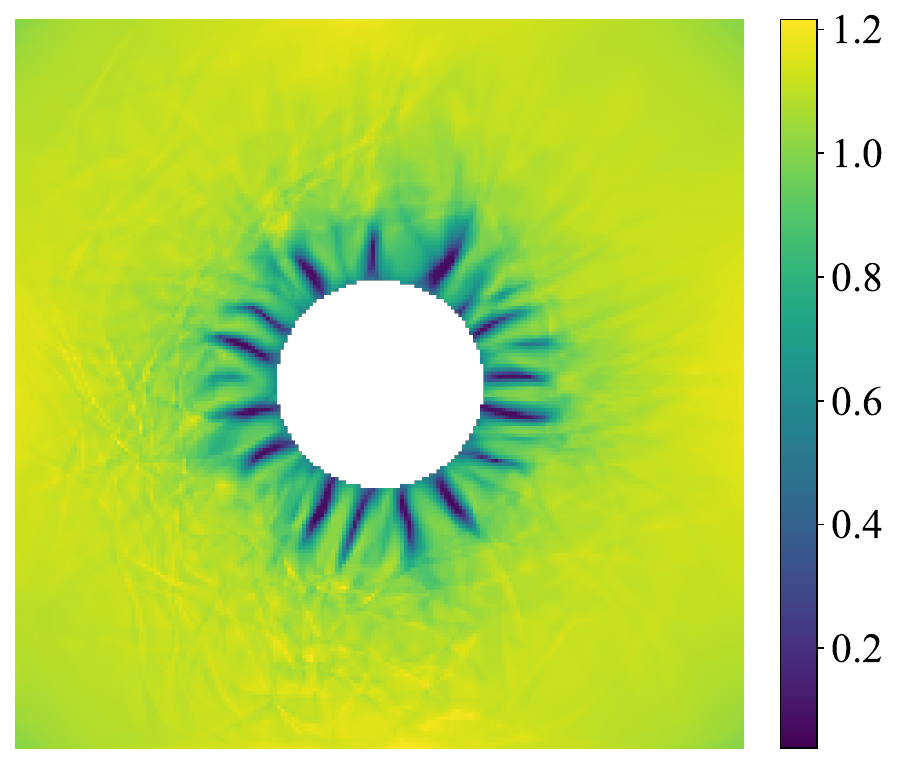}
		\caption{$h_l=7$}
		\label{fig:hl7_eps001}
	\end{subfigure}
	\hfill
	% --- Subfigure 4: hl=9 ---
	\begin{subfigure}[b]{0.24\linewidth}
		\centering
		\includegraphics[width=\linewidth]{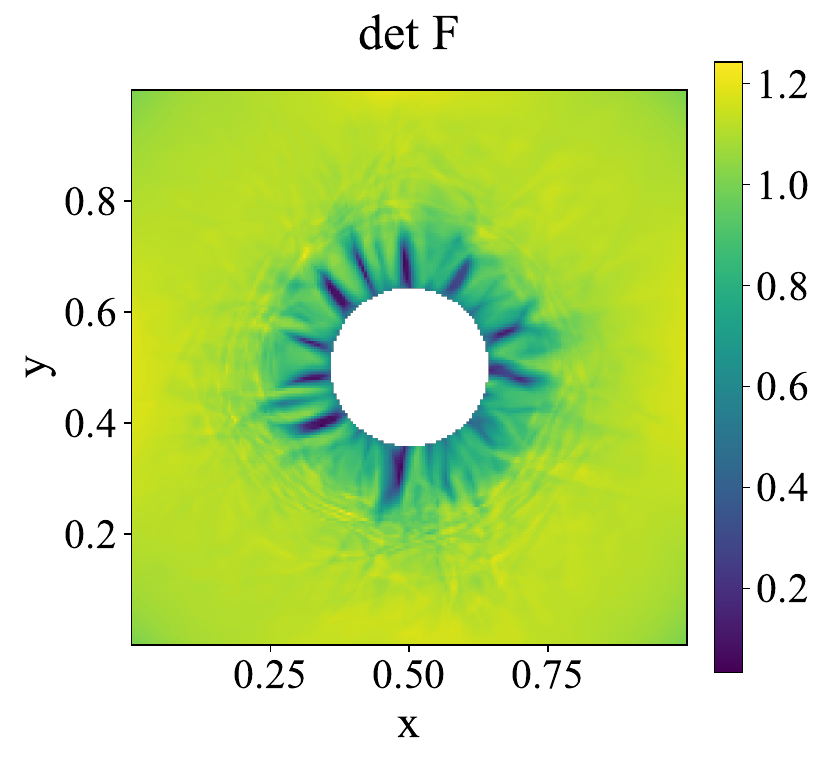}
		\caption{$h_l=9$}
		\label{fig:hl9_eps001}
	\end{subfigure}
	
\caption{Effect of network depth $h_l$ in the single-cell weakly regularized regime ($\varepsilon_0=0.01\,r_c$). Each panel shows the Jacobian determinant $J=\det F$ (with $F=\nabla y_\theta$) for different hidden-layer depths.}
	\label{fig:ed-single-eps001rc-depth}
\end{figure}

% ============================================================
\begin{figure}[t]
	\centering
	% --- Subfigure 1: width=64 ---
	\begin{subfigure}[b]{0.32\linewidth}
		\centering
		\includegraphics[width=\linewidth]{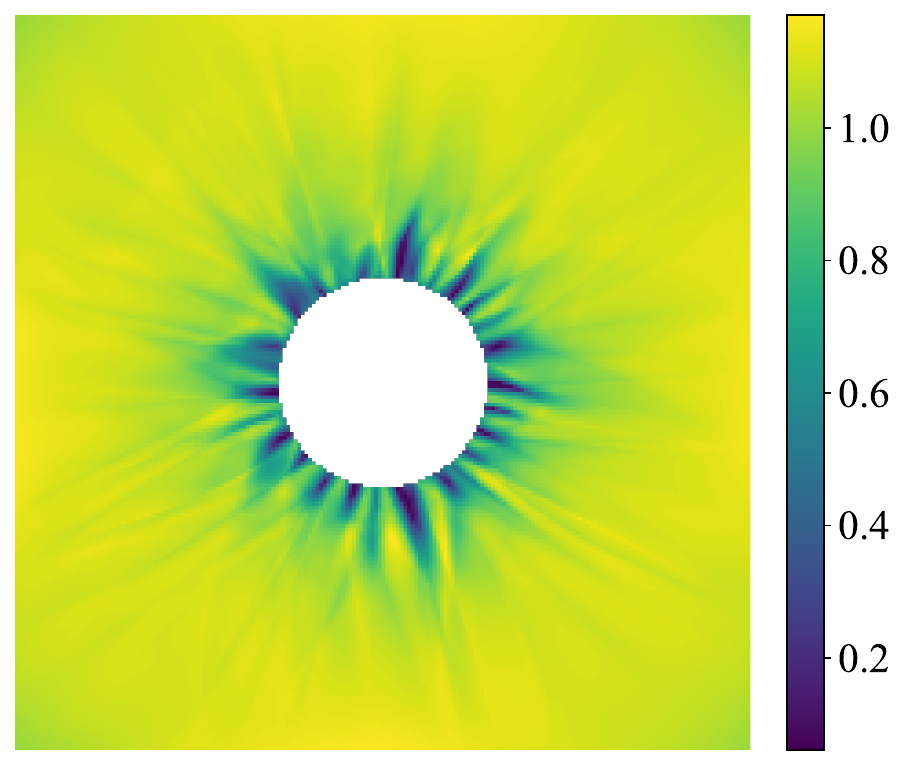}
		\caption{$width=64$}
		\label{fig:width64_eps001}
	\end{subfigure}
	\hfill
	% --- Subfigure 2: width=128 ---
	\begin{subfigure}[b]{0.32\linewidth}
		\centering
		\includegraphics[width=\linewidth]{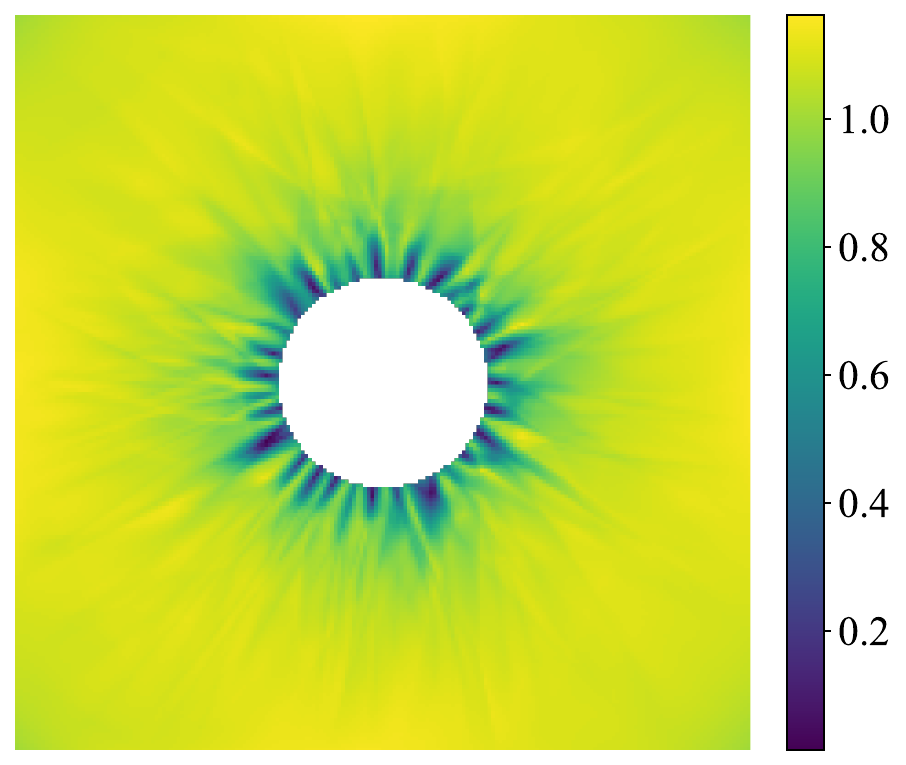}
		\caption{$width=128$}
		\label{fig:width128_eps001}
	\end{subfigure}
	\hfill
	% --- Subfigure 3: width=256 ---
	\begin{subfigure}[b]{0.32\linewidth}
		\centering
		\includegraphics[width=\linewidth]{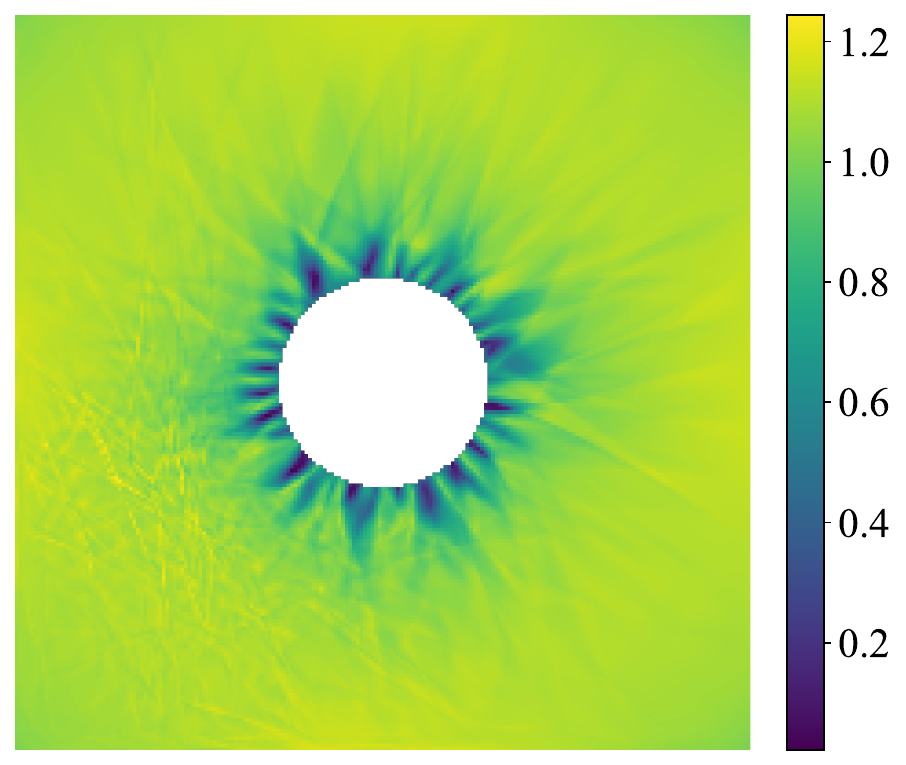}
		\caption{$width=256$}
		\label{fig:width256_eps001}
	\end{subfigure}
	
\caption{Effect of network width in the single-cell weakly regularized regime ($\varepsilon_0=0.01\,r_c$). Each panel shows the Jacobian determinant $J=\det F$ (with $F=\nabla y_\theta$) for different hidden-layer widths.}
	\label{fig:ed-single-eps001rc-width}
\end{figure}

% ============================================================
\begin{figure}[t]
	\centering
	\setlength{\tabcolsep}{2pt}
	\renewcommand{\arraystretch}{0}
	\begin{tabular}{cccc}
		\multicolumn{4}{c}{short distance, $d = 2.5r_c$} \\[2pt]
		\includegraphics[width=0.24\linewidth]{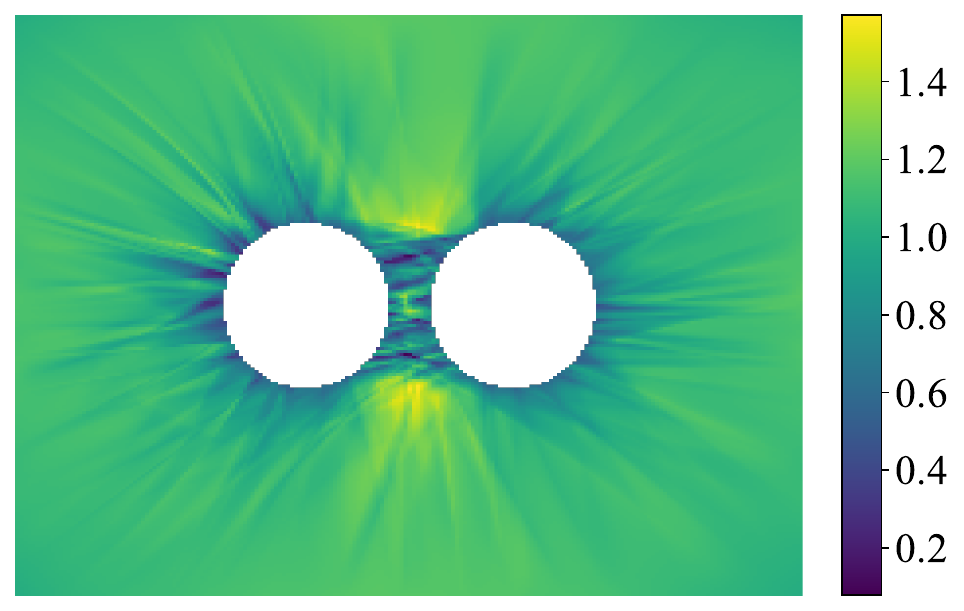} &
		\includegraphics[width=0.24\linewidth]{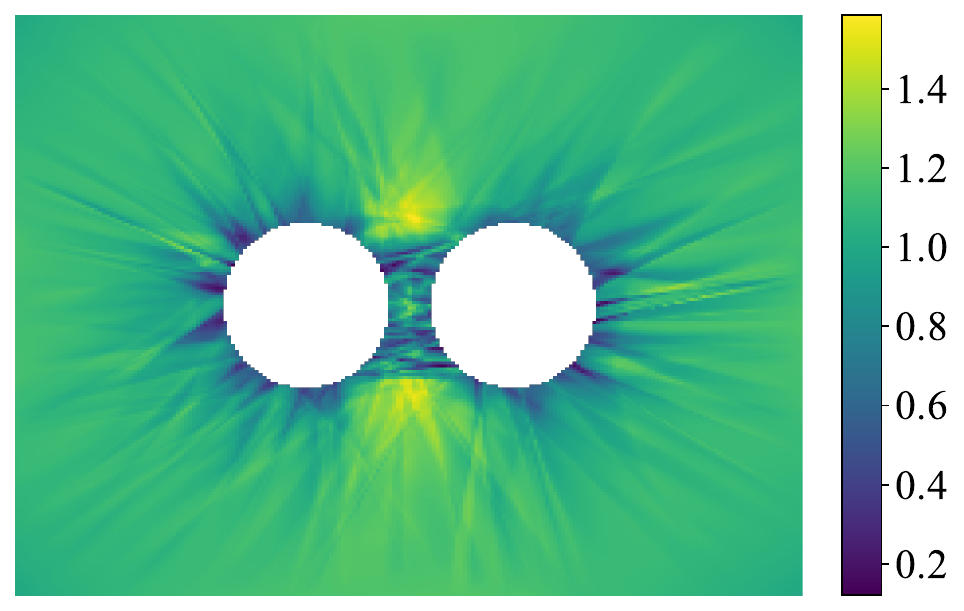} &
		\includegraphics[width=0.24\linewidth]{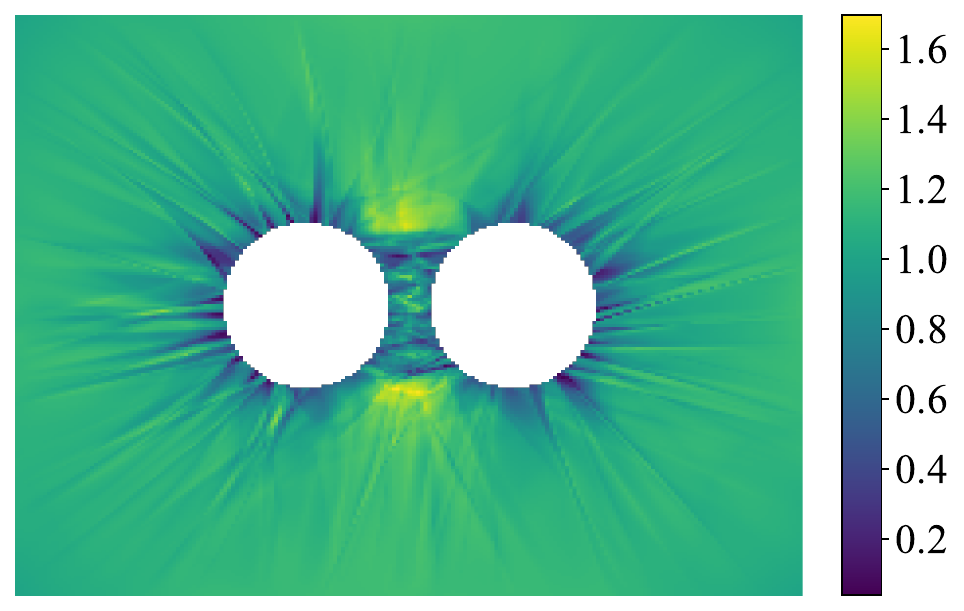} &
		\includegraphics[width=0.24\linewidth]{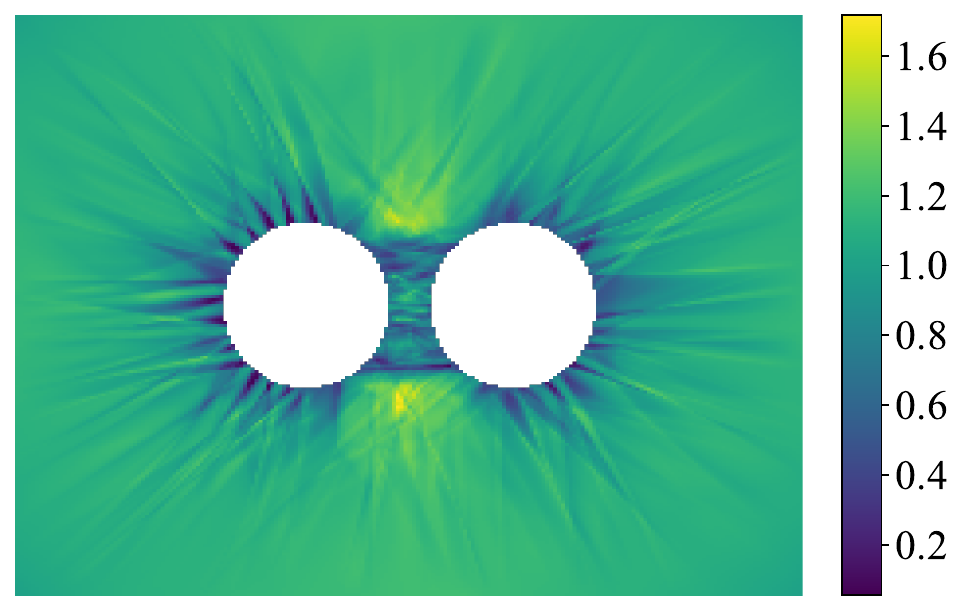} \\
		$P{=}100$ & $P{=}400$ & $P{=}1600$ & $P{=}3200$
		\\[6pt]
		\multicolumn{4}{c}{long distance, $d = 5r_c$} \\[2pt]
		\includegraphics[width=0.24\linewidth]{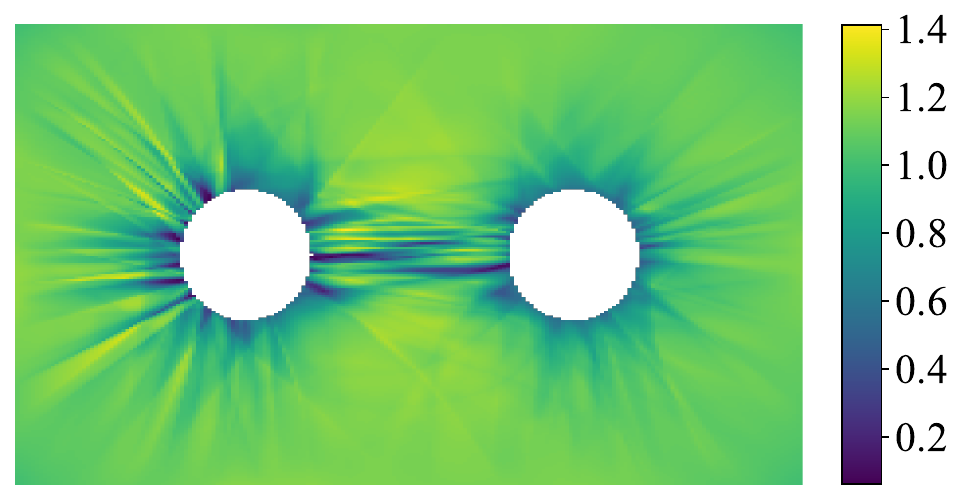} &
		\includegraphics[width=0.24\linewidth]{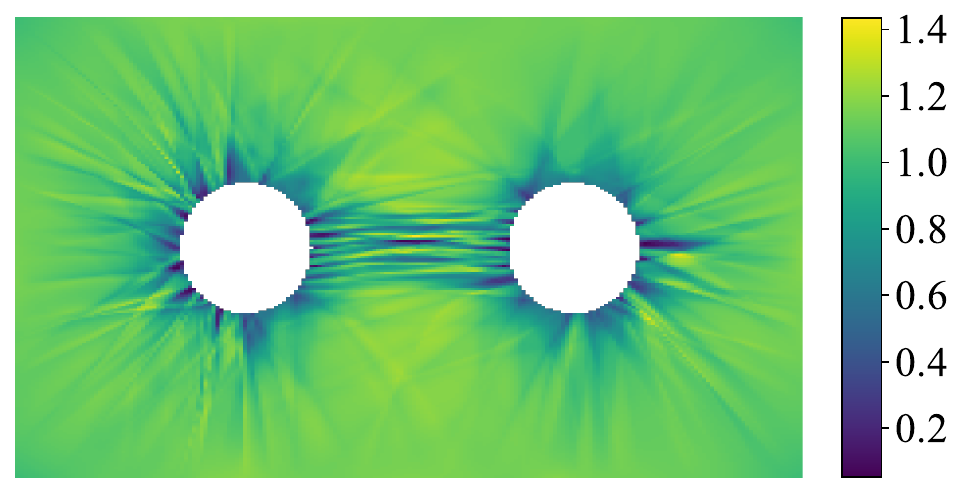} &
		\includegraphics[width=0.24\linewidth]{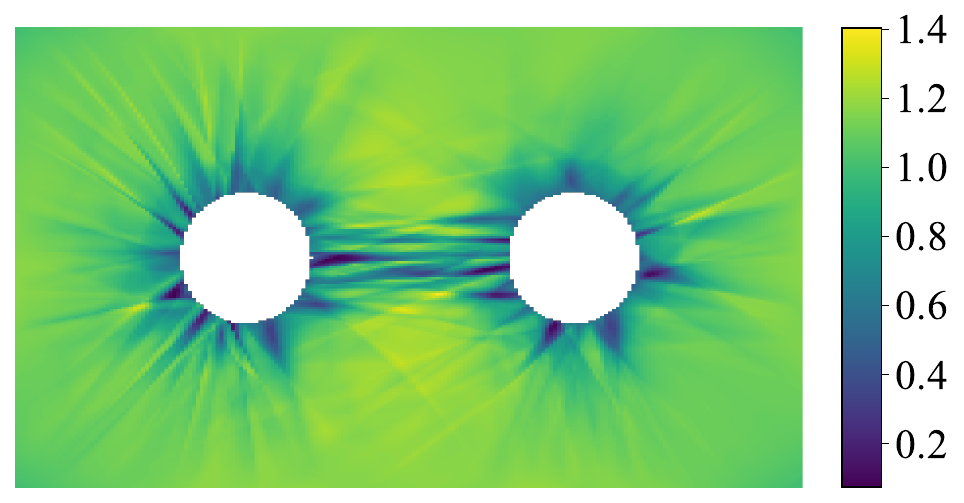} &
		\includegraphics[width=0.24\linewidth]{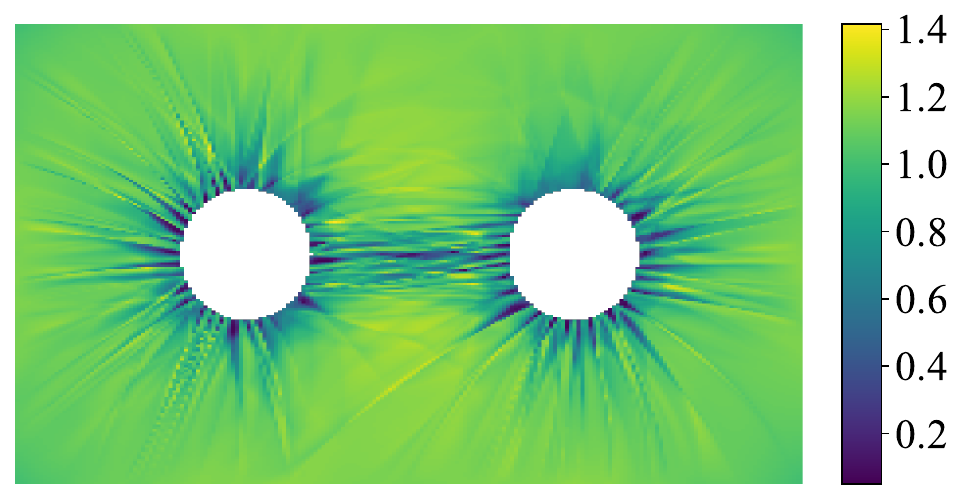} \\
		$P{=}100$ & $P{=}400$ & $P{=}1600$ & $P{=}3200$
	\end{tabular}
\caption{Effect of the resampling period $P$ in the two-cell non-regularized regime ($\varepsilon_0=0$). The Jacobian determinant $J=\det F$ (with $F=\nabla y_\theta$) is shown for the short-distance regime ($d=2.5\,r_c$; top) and the long-distance regime ($d=5\,r_c$; bottom).}
	\label{fig:ed-two-eps0-period}
\end{figure}

% ============================================================
\begin{figure}[t] % 如果是双栏论文，建议用figure*横跨两栏以看清细节
	\centering
	
	% --- 第一行：Short Gap ---
	\begin{subfigure}{0.23\linewidth}
		\centering
		\includegraphics[width=\linewidth]{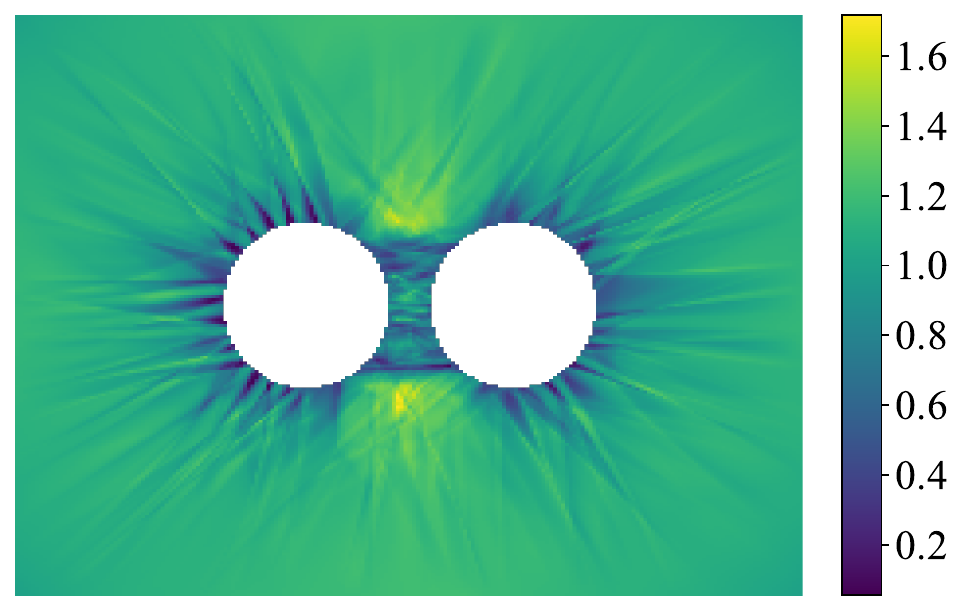}
		\caption{$h_l=3$}
	\end{subfigure}\hfill
	\begin{subfigure}{0.23\linewidth}
		\centering
		\includegraphics[width=\linewidth]{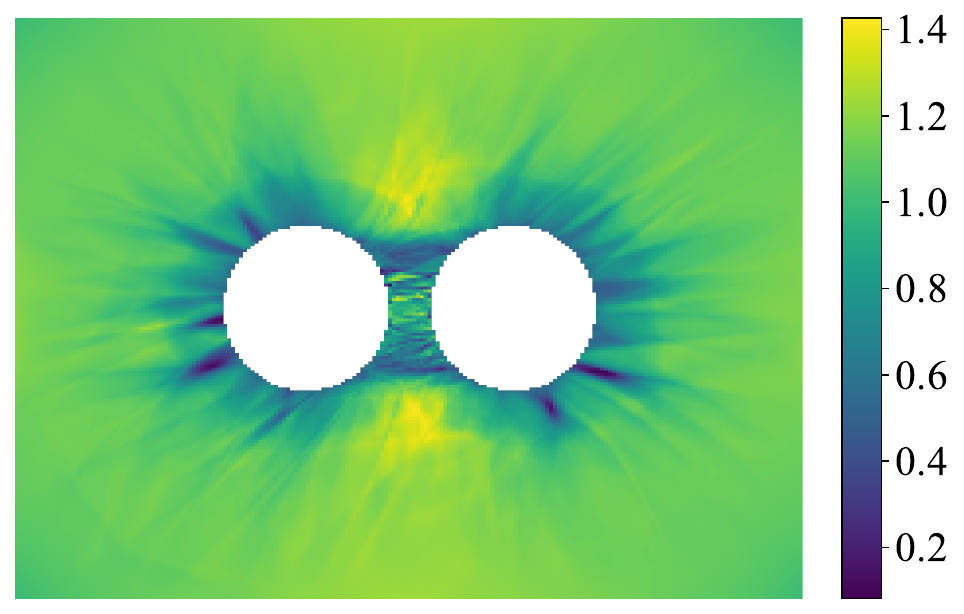}
		\caption{$h_l=5$}
	\end{subfigure}\hfill
	\begin{subfigure}{0.23\linewidth}
		\centering
		\includegraphics[width=\linewidth]{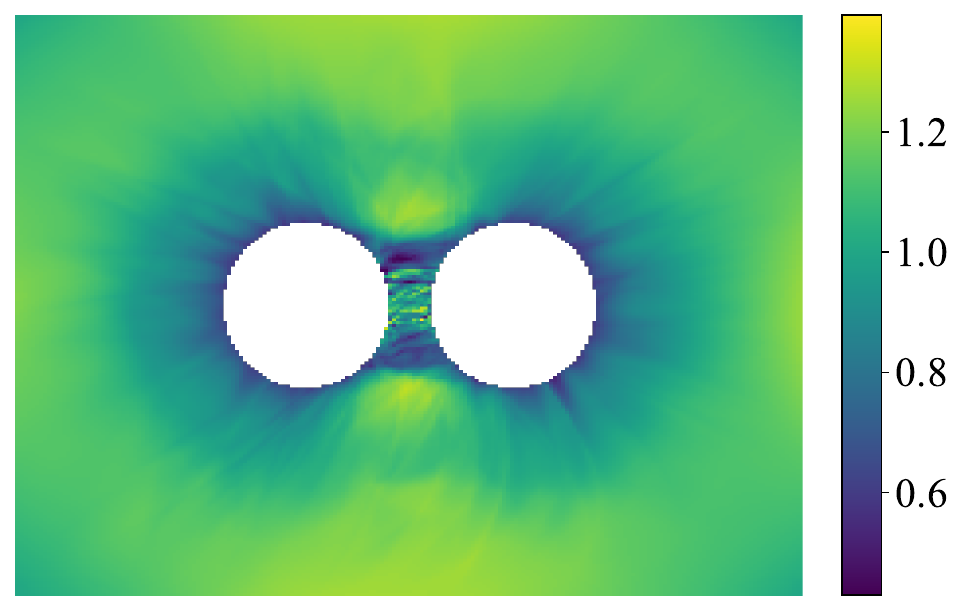}
		\caption{$h_l=7$}
	\end{subfigure}\hfill
	\begin{subfigure}{0.23\linewidth}
		\centering
		\includegraphics[width=\linewidth]{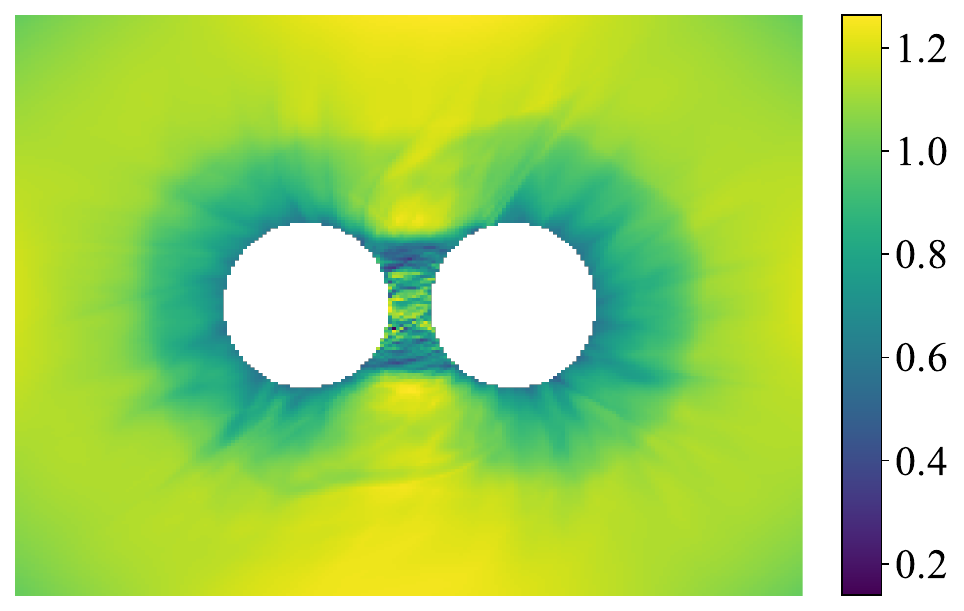}
		\caption{$h_l=9$}
	\end{subfigure}

	\vspace{1em} % 行间距

	% --- 第二行：Long Gap ---
	\begin{subfigure}{0.23\linewidth}
		\centering
		\includegraphics[width=\linewidth]{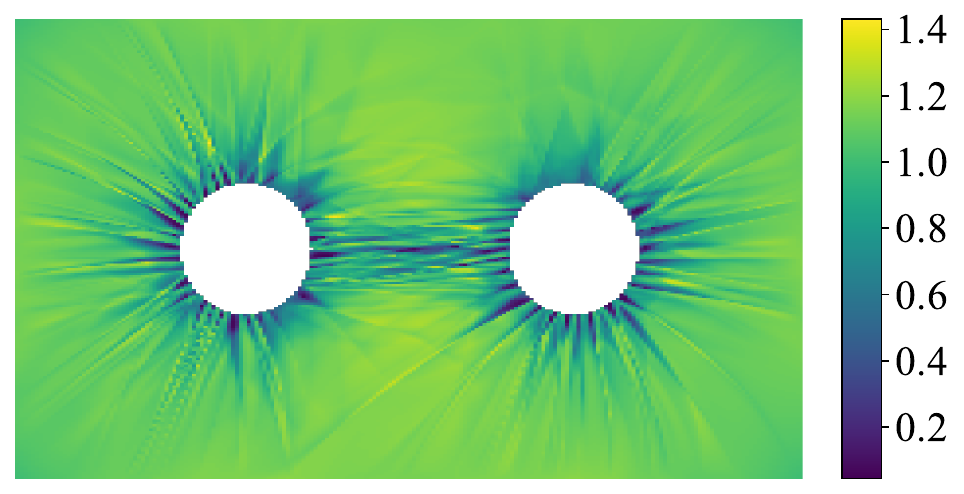}
		\caption{$h_l=3$}
	\end{subfigure}\hfill
	\begin{subfigure}{0.23\linewidth}
		\centering
		\includegraphics[width=\linewidth]{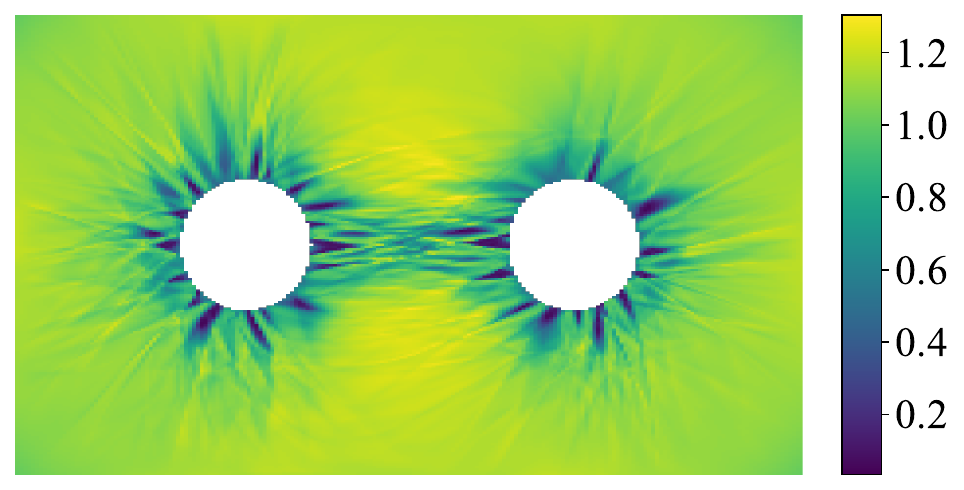}
		\caption{$h_l=5$}
	\end{subfigure}\hfill
	\begin{subfigure}{0.23\linewidth}
		\centering
		\includegraphics[width=\linewidth]{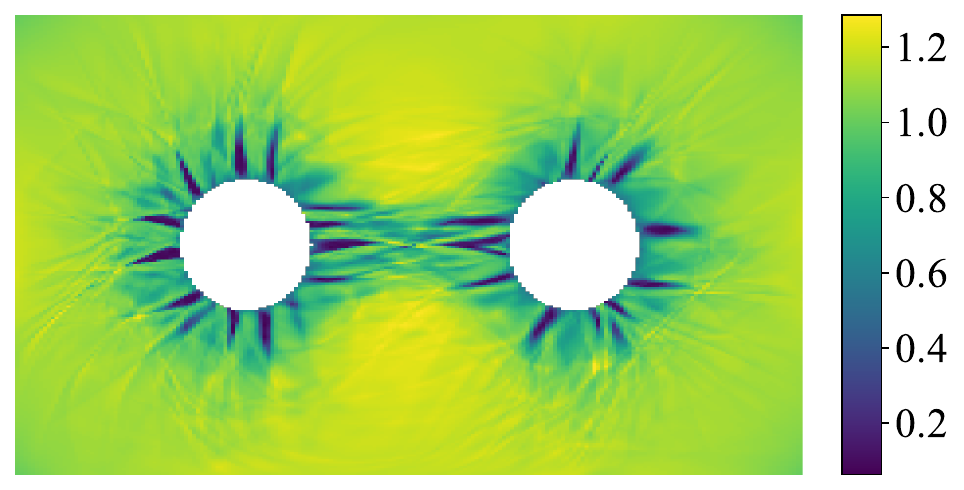}
		\caption{$h_l=7$}
	\end{subfigure}\hfill
	\begin{subfigure}{0.23\linewidth}
		\centering
		\includegraphics[width=\linewidth]{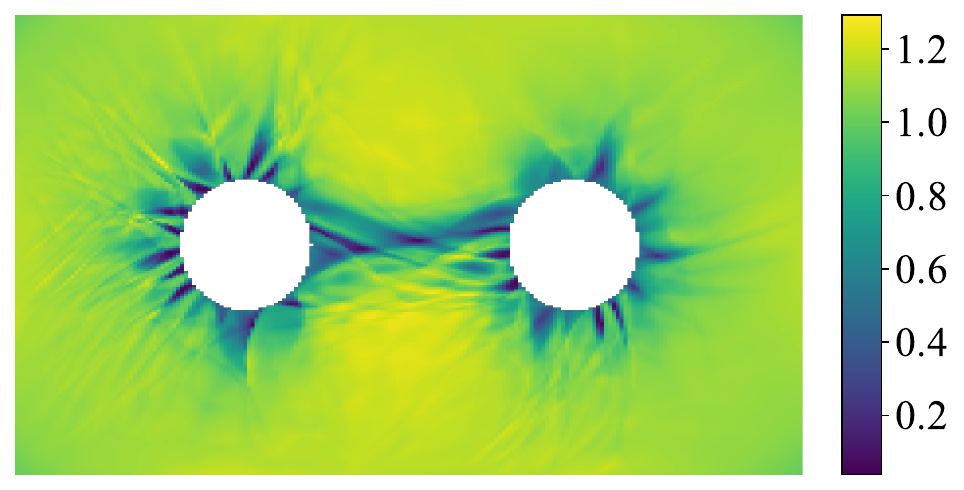}
		\caption{$h_l=9$}
	\end{subfigure}

\caption{Effect of network depth $h_l$ in the two-cell non-regularized regime ($\varepsilon_0=0$). The Jacobian determinant $J=\det F$ (with $F=\nabla y_\theta$) is shown for the short-distance regime ($d=2.5\,r_c$; top row) and the long-distance regime ($d=5\,r_c$; bottom row).}
	\label{fig:ed-two-eps0-depth}
\end{figure}

% ============================================================
\begin{figure}[t]
	\centering
	% =========================================
	% First Row: Short Gap (d=2.5rc)
	% =========================================
	\begin{subfigure}[b]{0.32\linewidth}
		\centering
		\includegraphics[width=\linewidth]{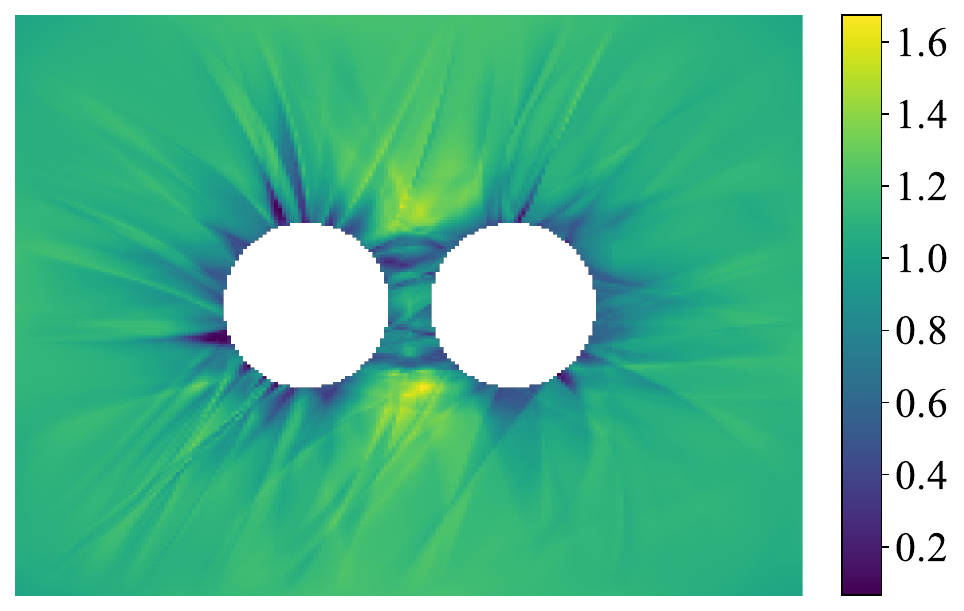}
		\caption{$width=64$}
		\label{fig:short_width64}
	\end{subfigure}
	\hfill
	\begin{subfigure}[b]{0.32\linewidth}
		\centering
		\includegraphics[width=\linewidth]{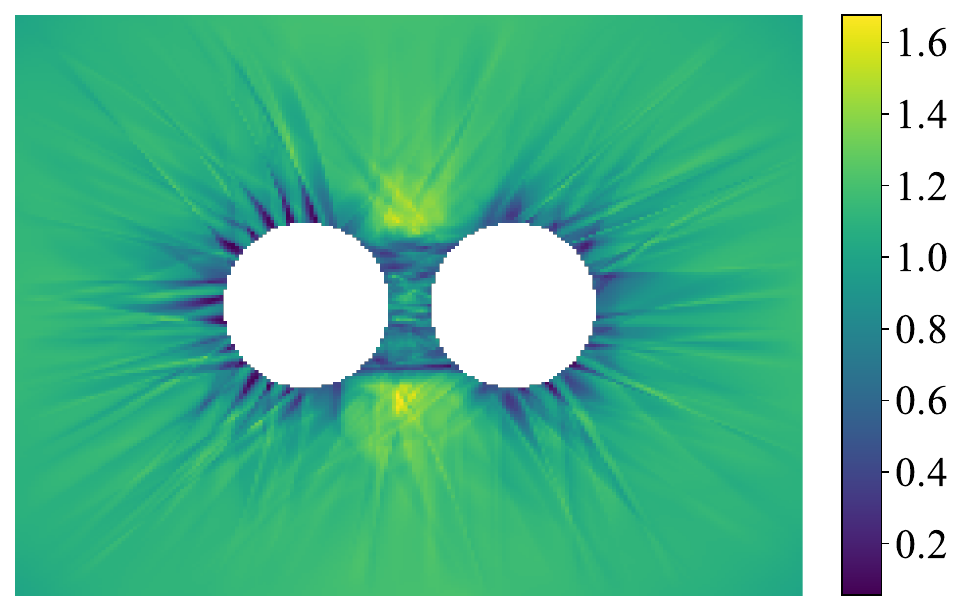}
		\caption{$width=128$}
		\label{fig:short_width128}
	\end{subfigure}
	\hfill
	\begin{subfigure}[b]{0.32\linewidth}
		\centering
		\includegraphics[width=\linewidth]{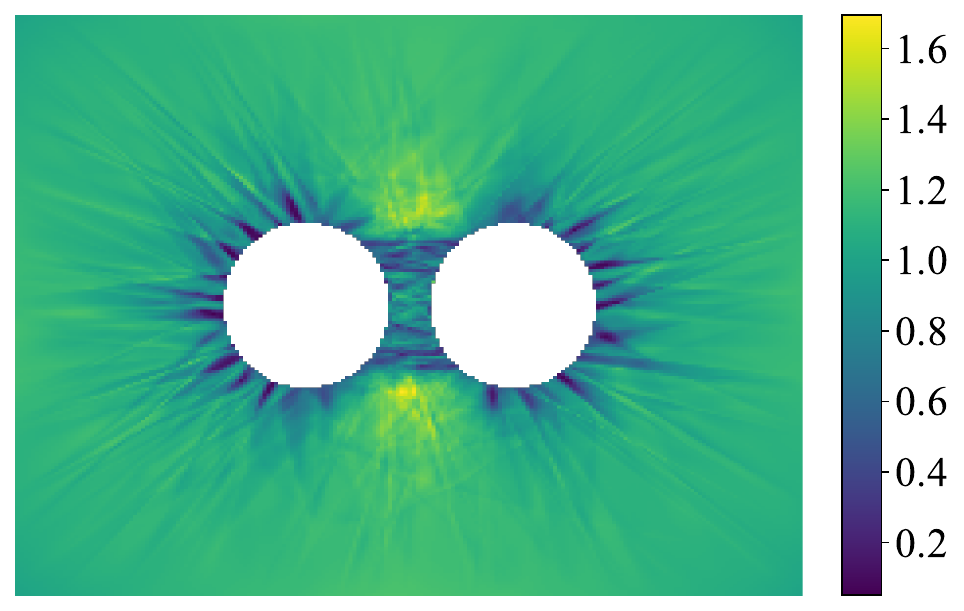}
		\caption{$width=256$}
		\label{fig:short_width256}
	\end{subfigure}
	
	% --- Space between rows ---
	\vspace{1em} 
	
	% =========================================
	% Second Row: Long Gap (d=5rc)
	% =========================================
	\begin{subfigure}[b]{0.32\linewidth}
		\centering
		\includegraphics[width=\linewidth]{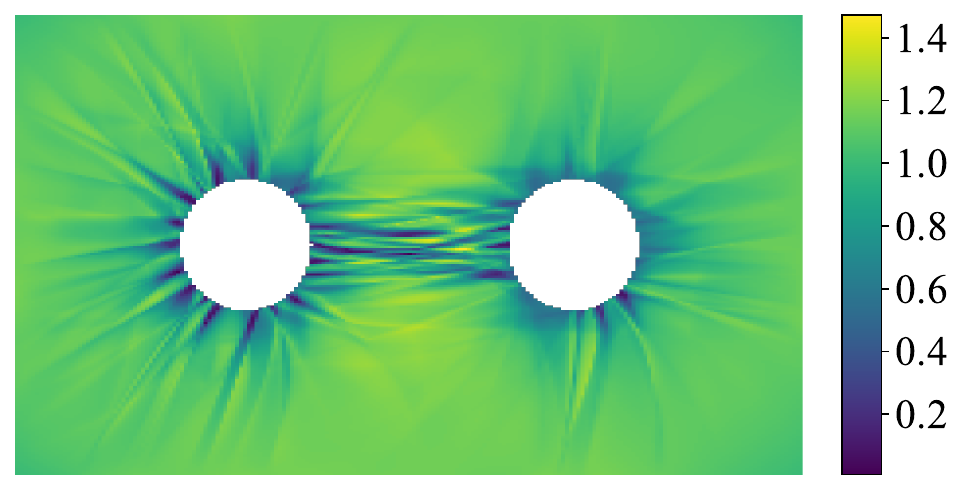}
		\caption{$width=64$}
		\label{fig:long_width64}
	\end{subfigure}
	\hfill
	\begin{subfigure}[b]{0.32\linewidth}
		\centering
		\includegraphics[width=\linewidth]{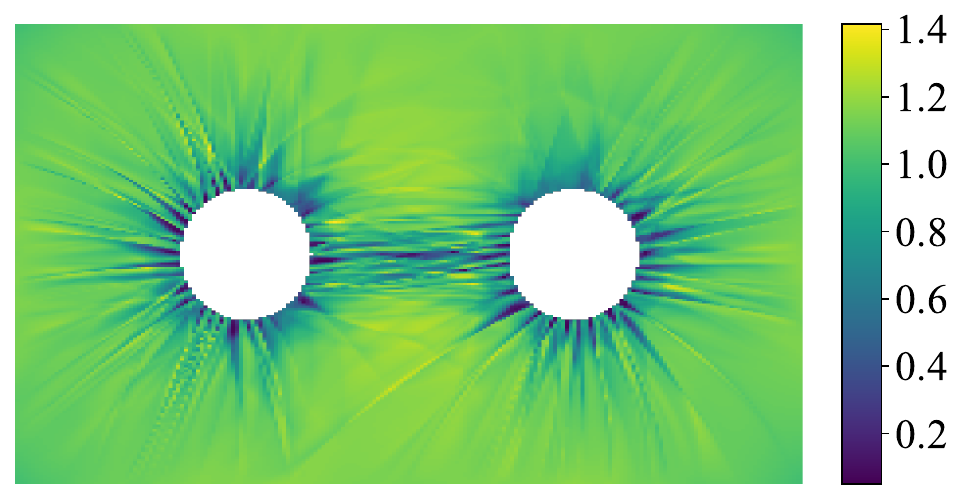}
		\caption{$width=128$}
		\label{fig:long_width128}
	\end{subfigure}
	\hfill
	\begin{subfigure}[b]{0.32\linewidth}
		\centering
		\includegraphics[width=\linewidth]{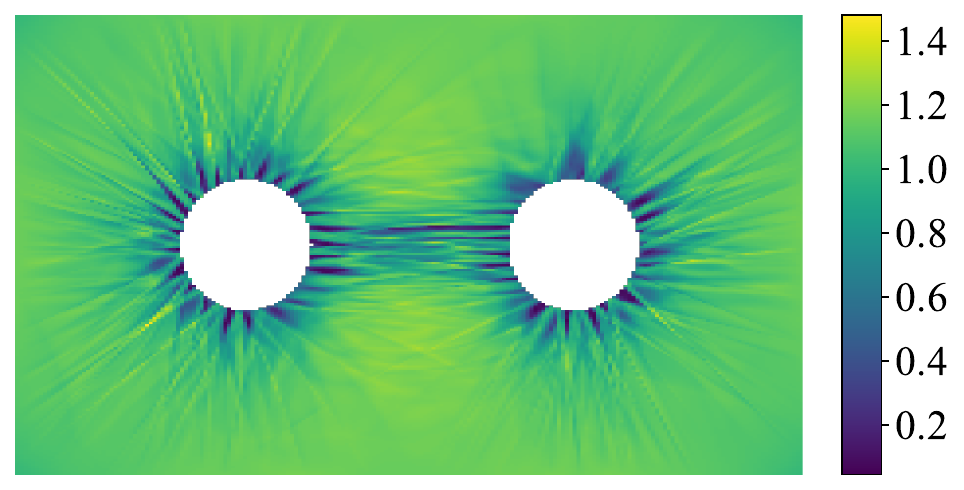}
		\caption{$width=256$}
		\label{fig:long_width256}
	\end{subfigure}
	
\caption{Effect of network width in the two-cell non-regularized regime ($\varepsilon_0=0$). The Jacobian determinant $J=\det F$ (with $F=\nabla y_\theta$) is shown for the short-distance regime ($d=2.5\,r_c$; top row) and the long-distance regime ($d=5\,r_c$; bottom row).}
	\label{fig:ed-two-eps0-width}
\end{figure}

% ============================================================
\begin{figure}[t]
	\centering
	% =========================================
	% First Row: Short Gap (d=2.5rc)
	% =========================================
	% --- Subfigure 1 ---
	\begin{subfigure}[b]{0.24\linewidth}
		\centering
		\includegraphics[width=\linewidth]{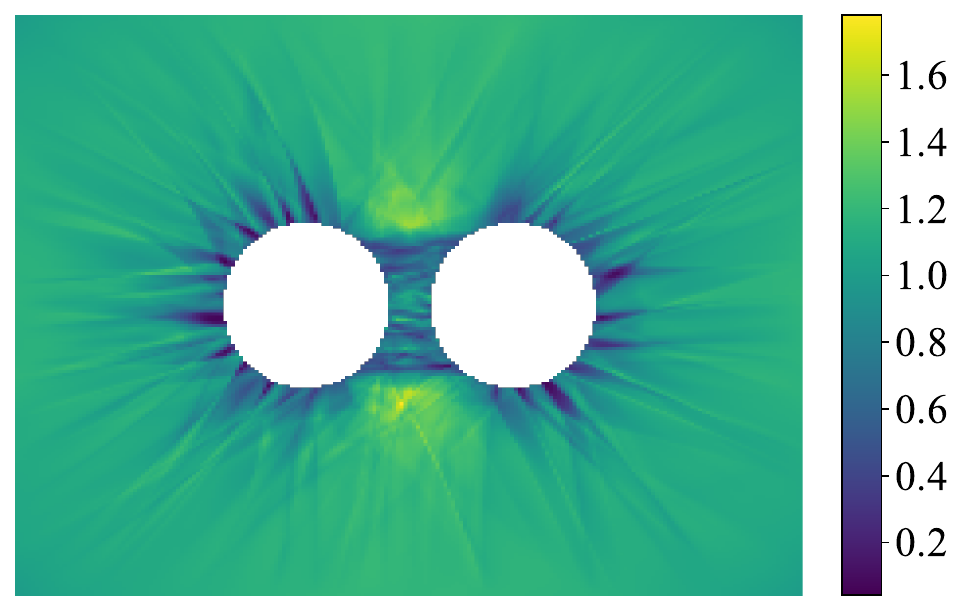}
		\caption{$\rho_{uq}=0.01$}
		\label{fig:short_rho0.01}
	\end{subfigure}
	\hfill
	% --- Subfigure 2 ---
	\begin{subfigure}[b]{0.24\linewidth}
		\centering
		\includegraphics[width=\linewidth]{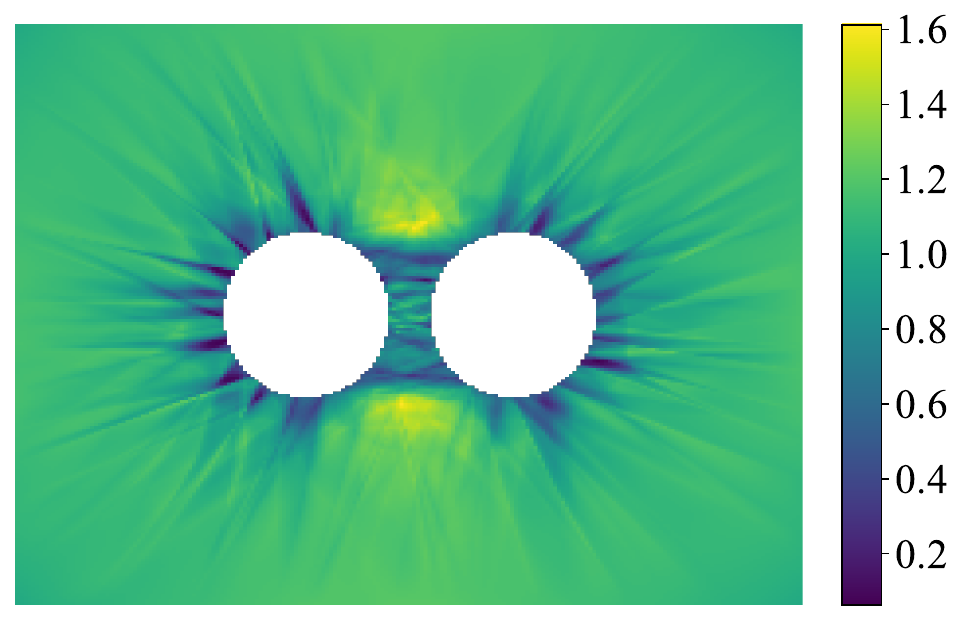}
		\caption{$\rho_{uq}=0.02$}
		\label{fig:short_rho0.02}
	\end{subfigure}
	\hfill
	% --- Subfigure 3 ---
	\begin{subfigure}[b]{0.24\linewidth}
		\centering
		\includegraphics[width=\linewidth]{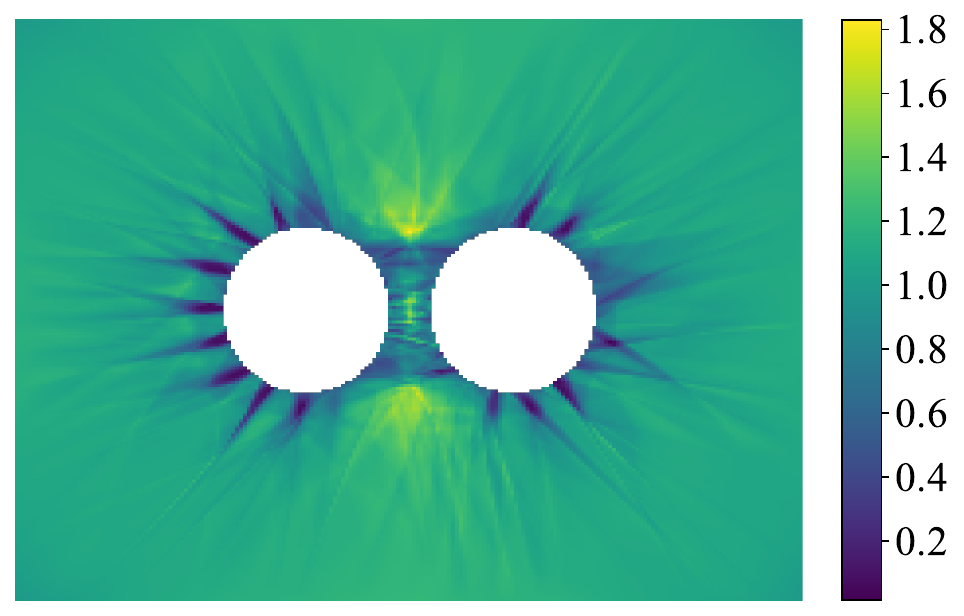}
		\caption{$\rho_{uq}=0.04$}
		\label{fig:short_rho0.04}
	\end{subfigure}
	\hfill
	% --- Subfigure 4 ---
	\begin{subfigure}[b]{0.24\linewidth}
		\centering
		\includegraphics[width=\linewidth]{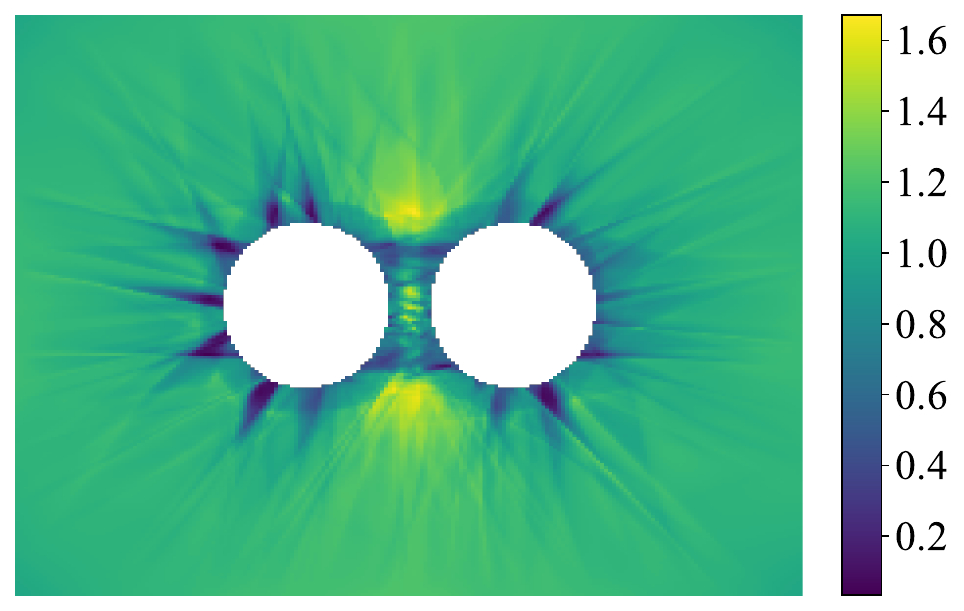}
		\caption{$\rho_{uq}=0.08$}
		\label{fig:short_rho0.08}
	\end{subfigure}
	
	% --- Vertical space between rows ---
	\vspace{1em}

	% =========================================
	% Second Row: Long Gap (d=5rc)
	% =========================================
	% --- Subfigure 5 ---
	\begin{subfigure}[b]{0.24\linewidth}
		\centering
		\includegraphics[width=\linewidth]{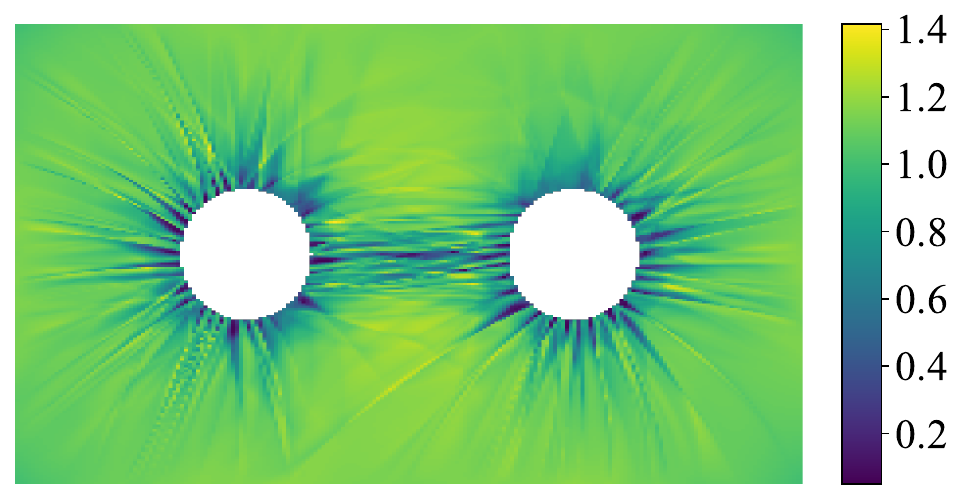}
		\caption{$\rho_{uq}=0.01$}
		\label{fig:long_rho0.01}
	\end{subfigure}
	\hfill
	% --- Subfigure 6 ---
	\begin{subfigure}[b]{0.24\linewidth}
		\centering
		\includegraphics[width=\linewidth]{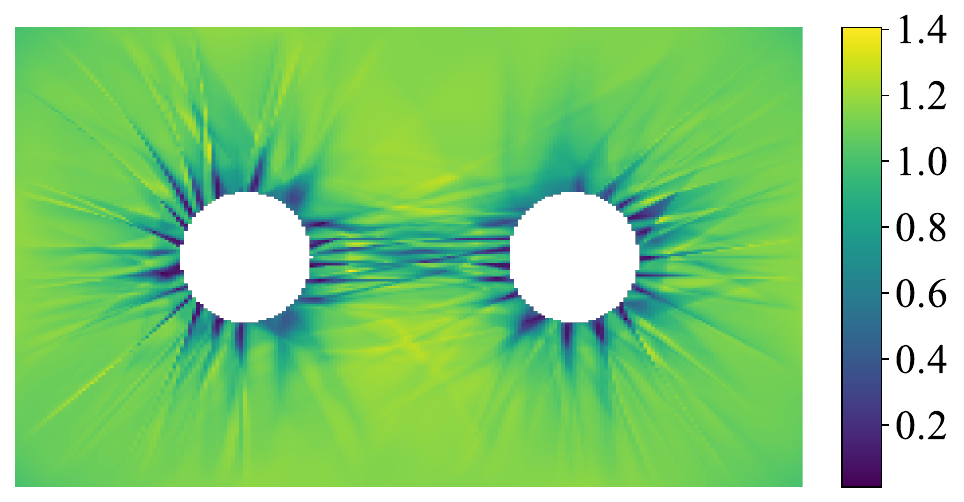}
		\caption{$\rho_{uq}=0.02$}
		\label{fig:long_rho0.02}
	\end{subfigure}
	\hfill
	% --- Subfigure 7 ---
	\begin{subfigure}[b]{0.24\linewidth}
		\centering
		\includegraphics[width=\linewidth]{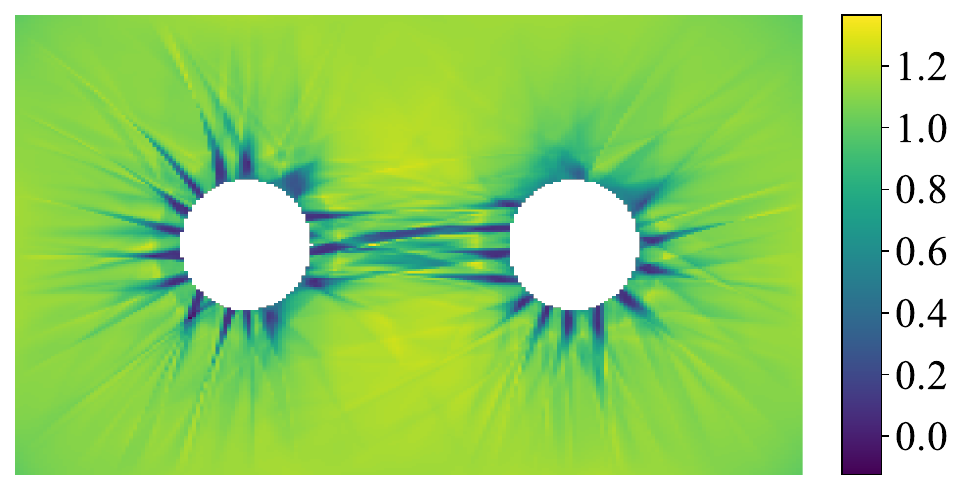}
		\caption{$\rho_{uq}=0.04$}
		\label{fig:long_rho0.04}
	\end{subfigure}
	\hfill
	% --- Subfigure 8 ---
	\begin{subfigure}[b]{0.24\linewidth}
		\centering
		\includegraphics[width=\linewidth]{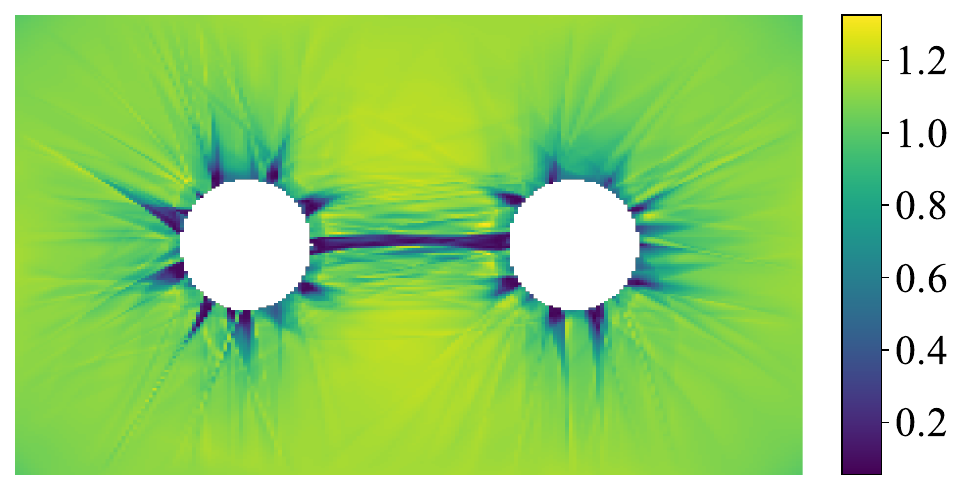}
		\caption{$\rho_{uq}=0.08$}
		\label{fig:long_rho0.08}
	\end{subfigure}
	
\caption{Effect of the UQ probe variance $\rho_{\mathrm{uq}}$ in the two-cell non-regularized regime ($\varepsilon_0=0$). The Jacobian determinant $J=\det F$ (with $F=\nabla y_\theta$) is shown for the short-distance regime ($d=2.5\,r_c$; top row) and the long-distance regime ($d=5\,r_c$; bottom row).}
	\label{fig:ed-two-eps0-uqrho}
\end{figure}

% ============================================================
\begin{figure}[t]
	\centering
	% =========================================
	% First Row: Short Gap (d=2.5rc)
	% =========================================
	% --- Subfigure 1 ---
	\begin{subfigure}[b]{0.24\linewidth}
		\centering
		\includegraphics[width=\linewidth]{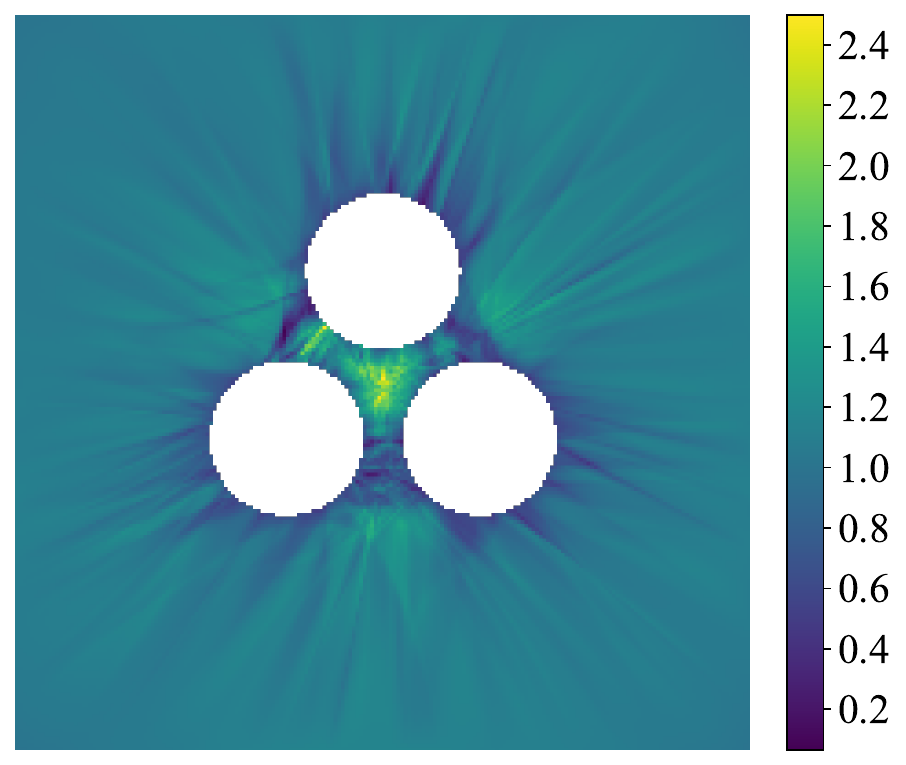}
		\caption{$P=100$}
		\label{fig:three_short_P100}
	\end{subfigure}
	\hfill
	% --- Subfigure 2 ---
	\begin{subfigure}[b]{0.24\linewidth}
		\centering
		\includegraphics[width=\linewidth]{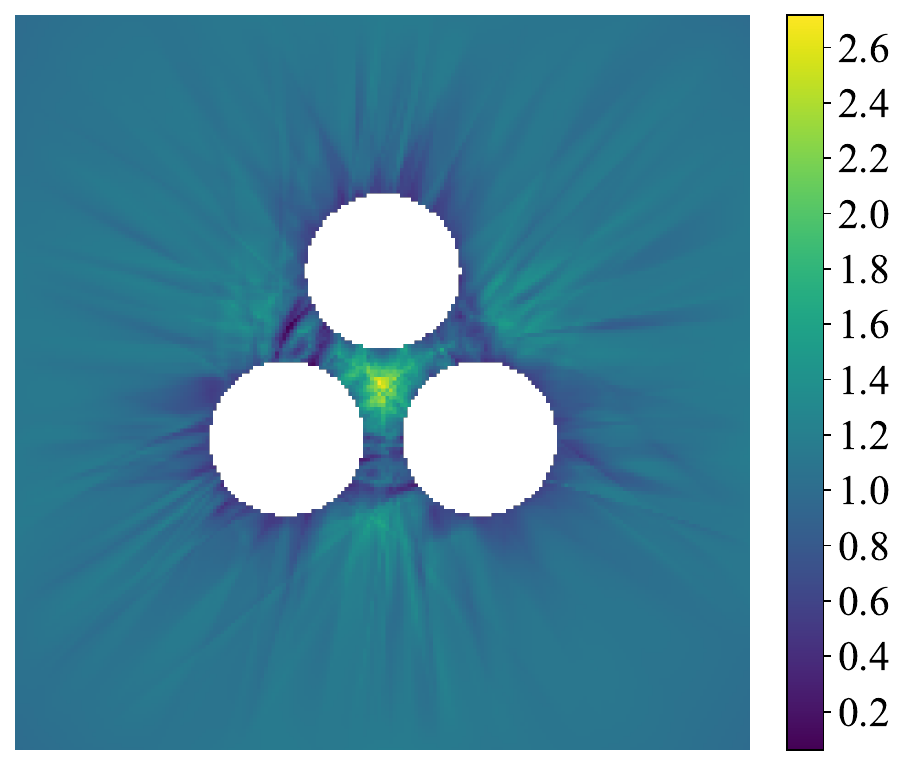}
		\caption{$P=400$}
		\label{fig:three_short_P400}
	\end{subfigure}
	\hfill
	% --- Subfigure 3 ---
	\begin{subfigure}[b]{0.24\linewidth}
		\centering
		\includegraphics[width=\linewidth]{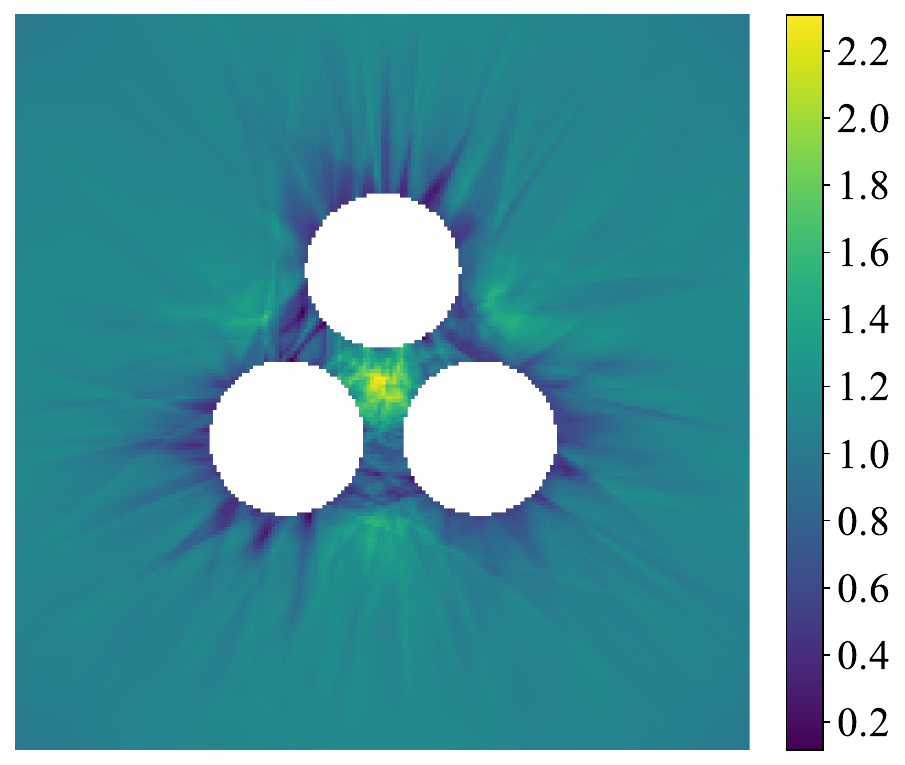}
		\caption{$P=1600$}
		\label{fig:three_short_P1600}
	\end{subfigure}
	\hfill
	% --- Subfigure 4 ---
	\begin{subfigure}[b]{0.24\linewidth}
		\centering
		\includegraphics[width=\linewidth]{Figs/Three_epsilon0_period/F_fields_short_P3200.pdf}
		\caption{$P=3200$}
		\label{fig:three_short_P3200}
	\end{subfigure}
	
	% --- Vertical space between rows ---
	\vspace{1em}

	% =========================================
	% Second Row: Long Gap (d=5rc)
	% =========================================
	% --- Subfigure 5 ---
	\begin{subfigure}[b]{0.24\linewidth}
		\centering
		\includegraphics[width=\linewidth]{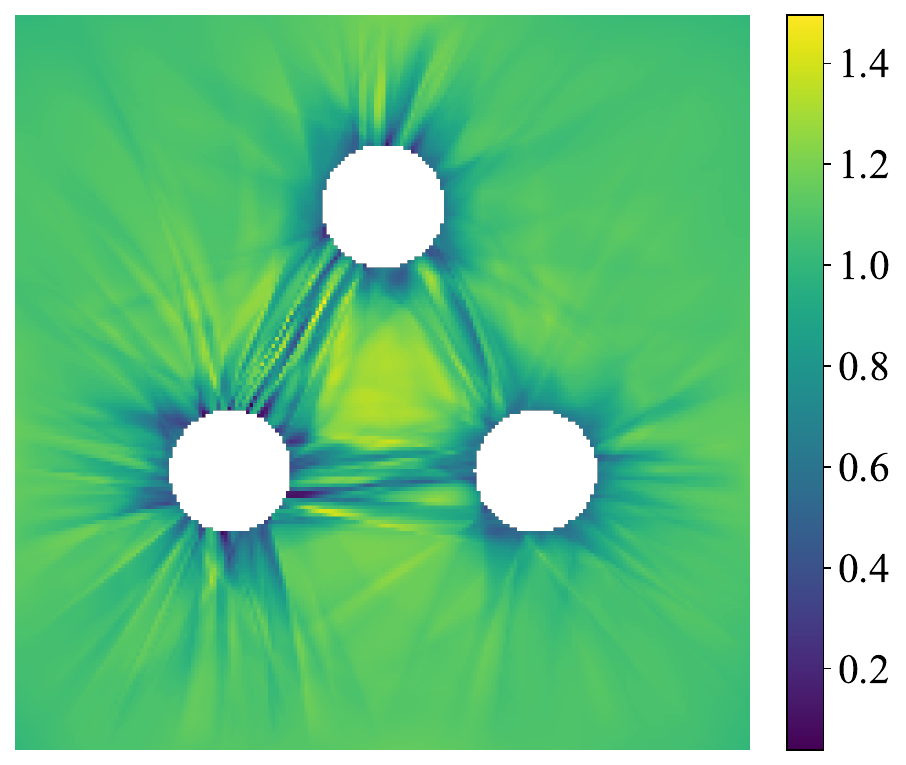}
		\caption{$P=100$}
		\label{fig:three_long_P100}
	\end{subfigure}
	\hfill
	% --- Subfigure 6 ---
	\begin{subfigure}[b]{0.24\linewidth}
		\centering
		\includegraphics[width=\linewidth]{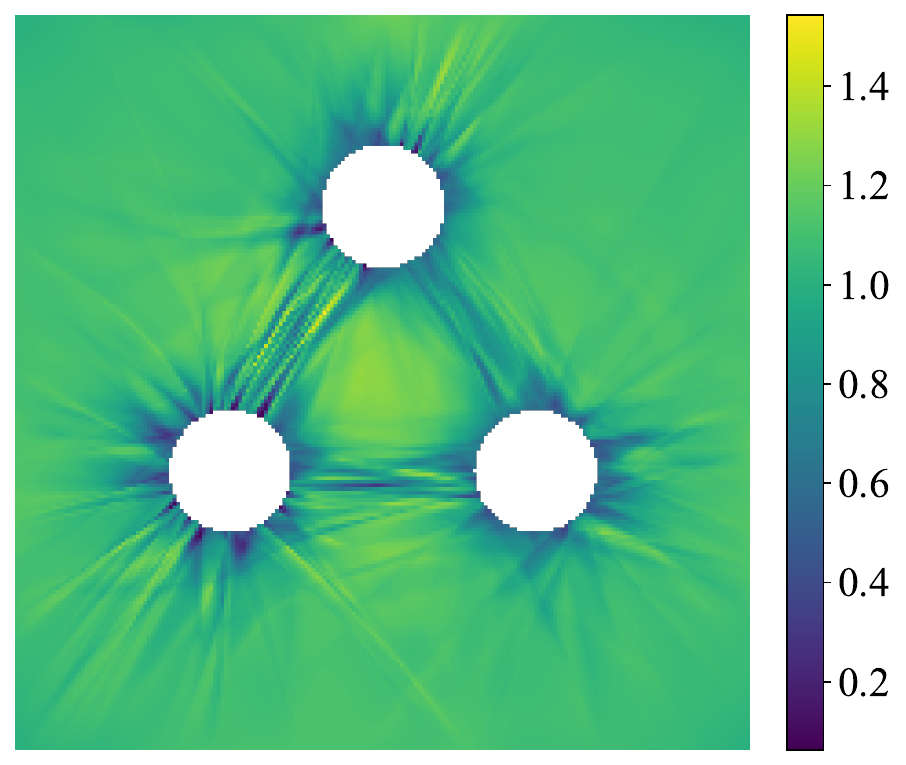}
		\caption{$P=400$}
		\label{fig:three_long_P400}
	\end{subfigure}
	\hfill
	% --- Subfigure 7 ---
	\begin{subfigure}[b]{0.24\linewidth}
		\centering
		\includegraphics[width=\linewidth]{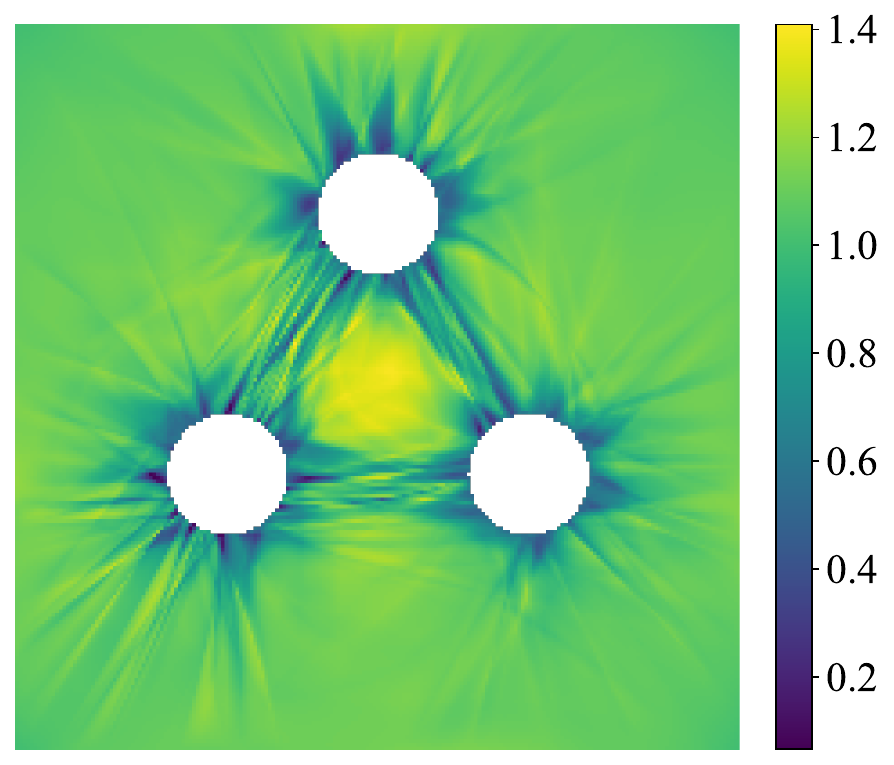}
		\caption{$P=1600$}
		\label{fig:three_long_P1600}
	\end{subfigure}
	\hfill
	% --- Subfigure 8 ---
	\begin{subfigure}[b]{0.24\linewidth}
		\centering
		\includegraphics[width=\linewidth]{Figs/Three_epsilon0_period/F_fields_long_P3200.pdf}
		\caption{$P=3200$}
		\label{fig:three_long_P3200}
	\end{subfigure}
	
\caption{Effect of the resampling period $P$ in the three-cell non-regularized regime ($\varepsilon_0=0$). The Jacobian determinant $J=\det F$ (with $F=\nabla y_\theta$) is shown for the short-distance regime ($d=2.5\,r_c$; top row) and the long-distance regime ($d=5\,r_c$; bottom row).}
	\label{fig:ed-three-eps0-period}
\end{figure}

% ============================================================
\begin{figure}[t]
	\centering
	% =========================================
	% First Row: Short Gap (d=2.5rc)
	% =========================================
	% --- Subfigure 1 ---
	\begin{subfigure}[b]{0.24\linewidth}
		\centering
		\includegraphics[width=\linewidth]{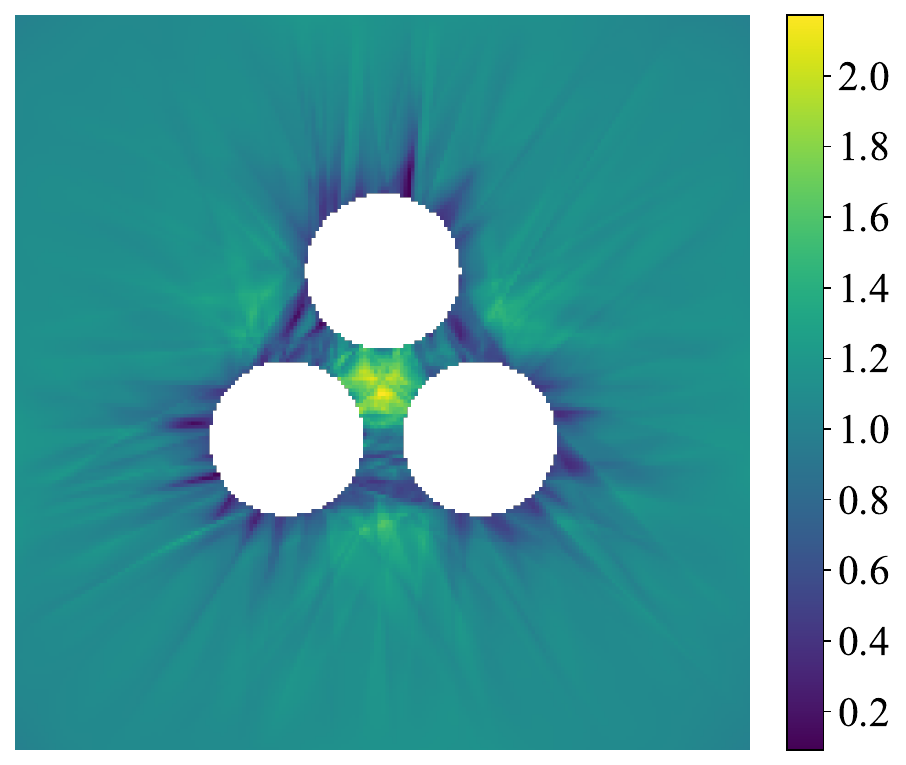}
		\caption{$\rho_{uq}=0.01$}
		\label{fig:three_short_rho0.01}
	\end{subfigure}
	\hfill
	% --- Subfigure 2 ---
	\begin{subfigure}[b]{0.24\linewidth}
		\centering
		\includegraphics[width=\linewidth]{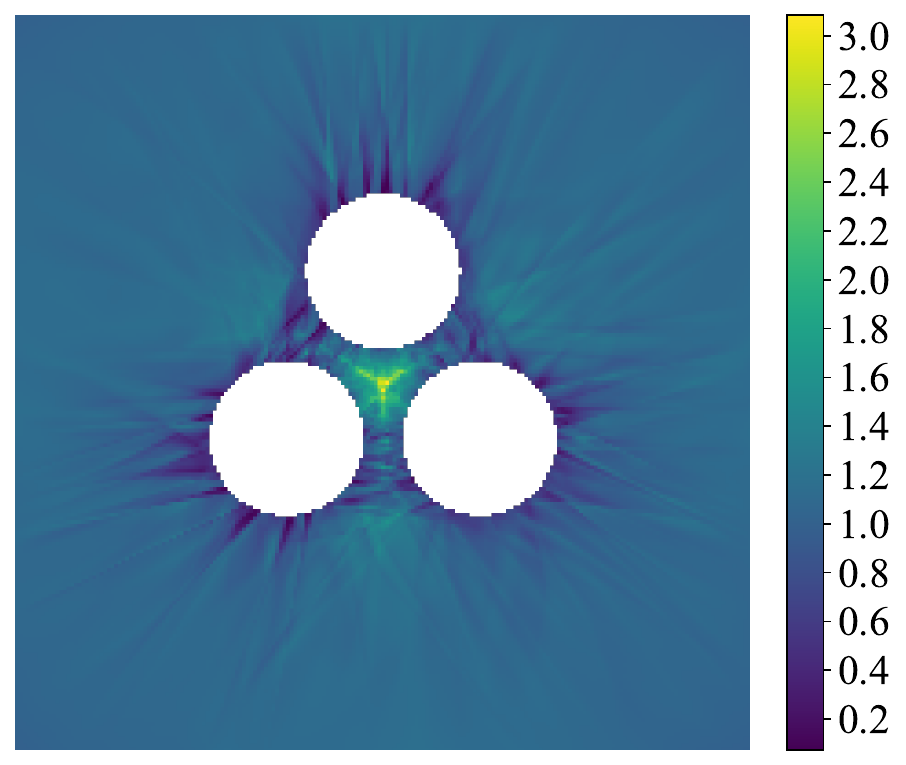}
		\caption{$\rho_{uq}=0.02$}
		\label{fig:three_short_rho0.02}
	\end{subfigure}
	\hfill
	% --- Subfigure 3 ---
	\begin{subfigure}[b]{0.24\linewidth}
		\centering
		\includegraphics[width=\linewidth]{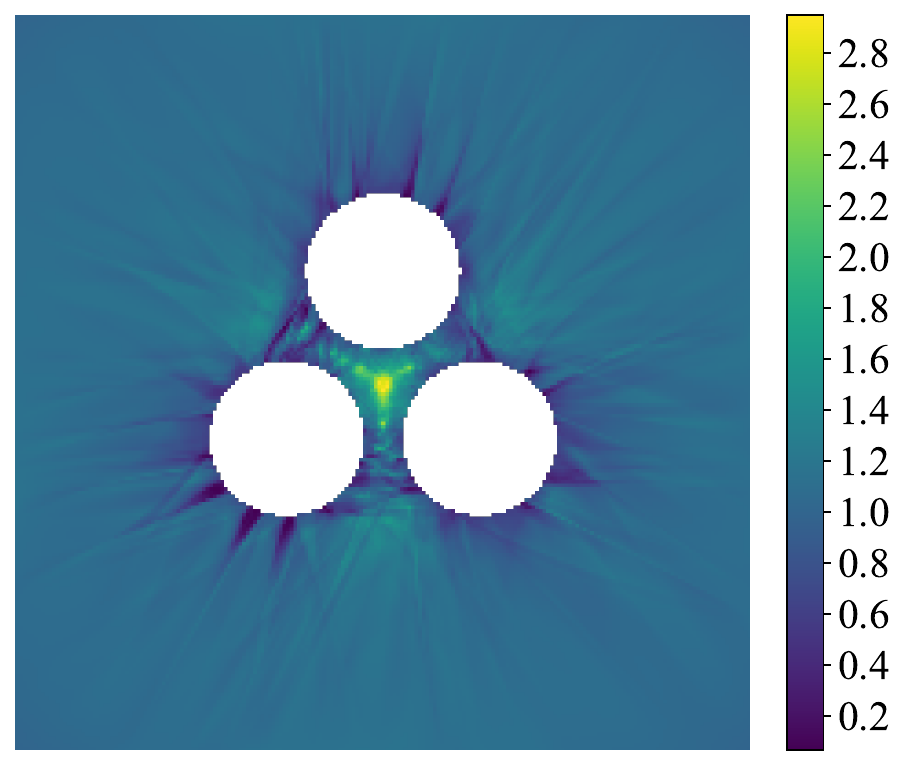}
		\caption{$\rho_{uq}=0.04$}
		\label{fig:three_short_rho0.04}
	\end{subfigure}
	\hfill
	% --- Subfigure 4 ---
	\begin{subfigure}[b]{0.24\linewidth}
		\centering
		\includegraphics[width=\linewidth]{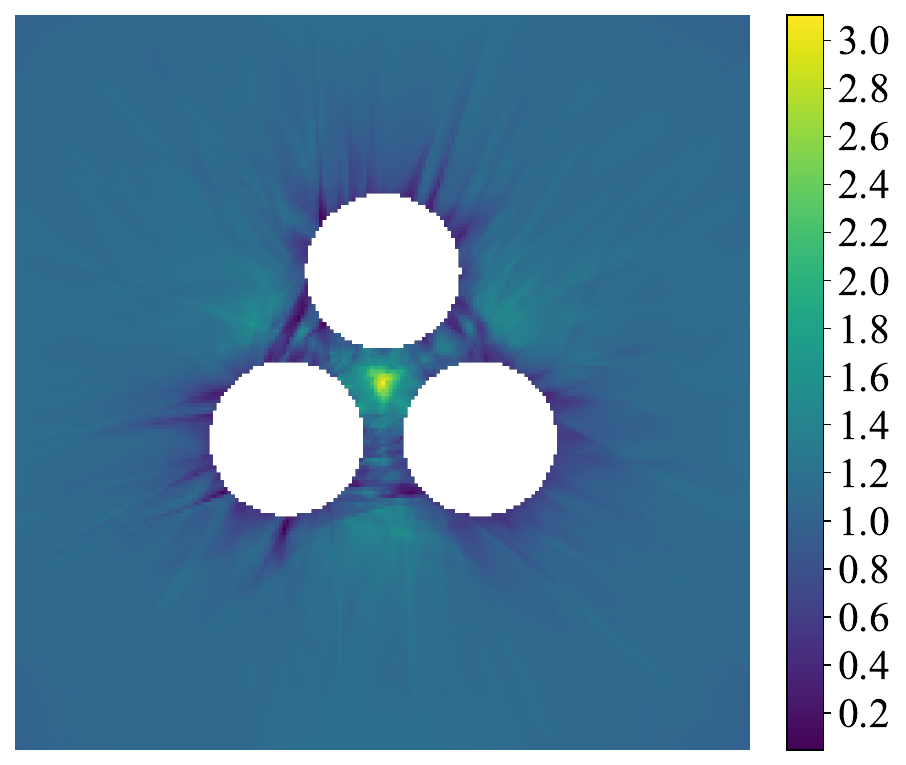}
		\caption{$\rho_{uq}=0.08$}
		\label{fig:three_short_rho0.08}
	\end{subfigure}
	
	% --- Vertical space between rows ---
	\vspace{1em}

	% =========================================
	% Second Row: Long Gap (d=5rc)
	% =========================================
	% --- Subfigure 5 ---
	\begin{subfigure}[b]{0.24\linewidth}
		\centering
		\includegraphics[width=\linewidth]{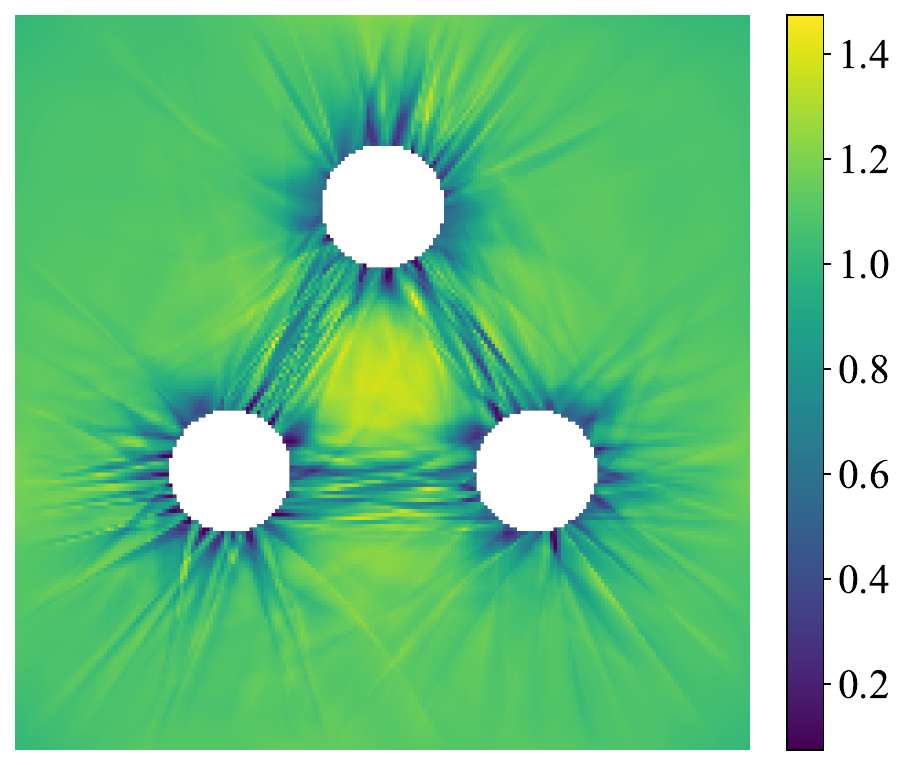}
		\caption{$\rho_{uq}=0.01$}
		\label{fig:three_long_rho0.01}
	\end{subfigure}
	\hfill
	% --- Subfigure 6 ---
	\begin{subfigure}[b]{0.24\linewidth}
		\centering
		\includegraphics[width=\linewidth]{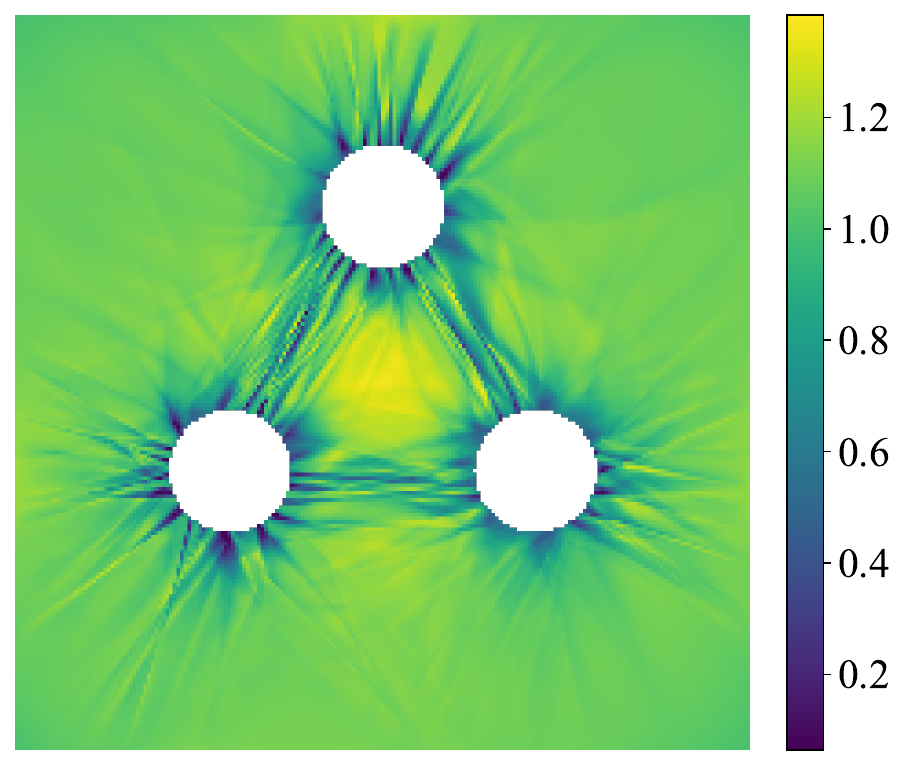}
		\caption{$\rho_{uq}=0.02$}
		\label{fig:three_long_rho0.02}
	\end{subfigure}
	\hfill
	% --- Subfigure 7 ---
	\begin{subfigure}[b]{0.24\linewidth}
		\centering
		\includegraphics[width=\linewidth]{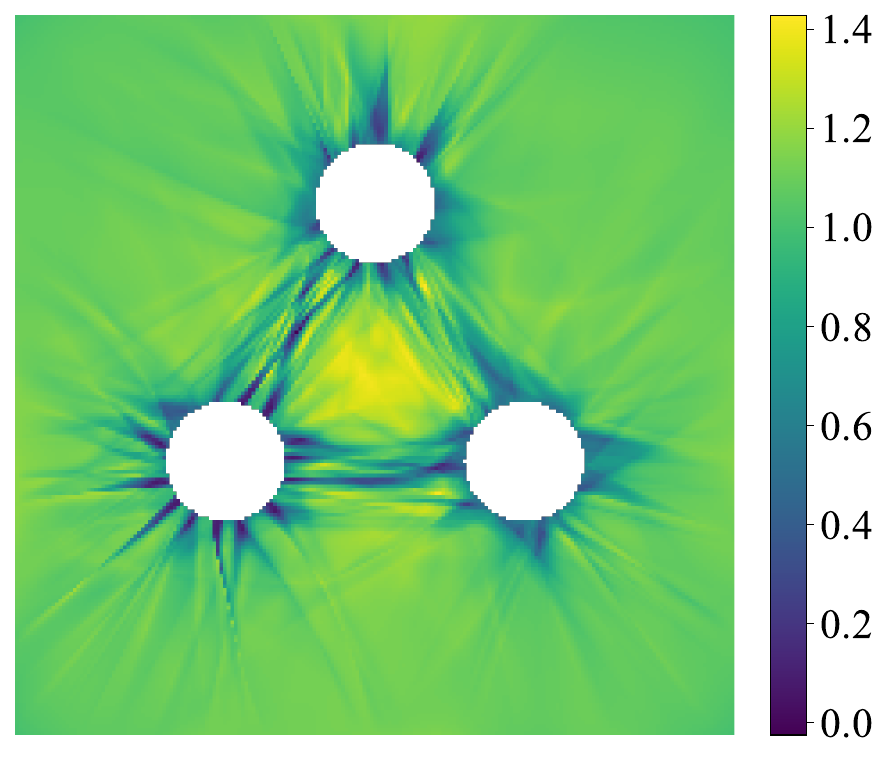}
		\caption{$\rho_{uq}=0.04$}
		\label{fig:three_long_rho0.04}
	\end{subfigure}
	\hfill
	% --- Subfigure 8 ---
	\begin{subfigure}[b]{0.24\linewidth}
		\centering
		\includegraphics[width=\linewidth]{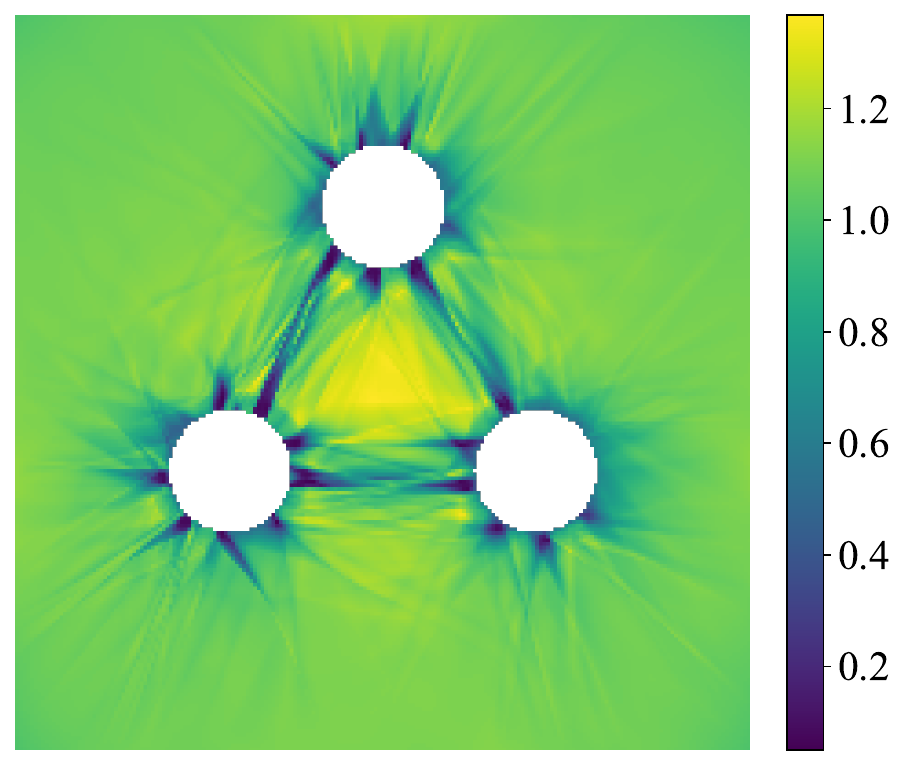}
		\caption{$\rho_{uq}=0.08$}
		\label{fig:three_long_rho0.08}
	\end{subfigure}
	
\caption{Effect of the UQ probe variance $\rho_{\mathrm{uq}}$ in the three-cell non-regularized regime ($\varepsilon_0=0$). The Jacobian determinant $J=\det F$ (with $F=\nabla y_\theta$) is shown for the short-distance regime ($d=2.5\,r_c$; top row) and the long-distance regime ($d=5\,r_c$; bottom row).}
	\label{fig:ed-three-eps0-uqrho}
\end{figure}

\end{document}